\documentclass[lettersize,journal]{IEEEtran}
\usepackage{amsmath,amsfonts}
\usepackage{array}
\usepackage[caption=false,font=normalsize,labelfont=sf,textfont=sf]{subfig}
\usepackage{textcomp}
\usepackage{stfloats}
\usepackage{url}
\usepackage{verbatim}
\usepackage{graphicx}
\usepackage{cite}
\usepackage{algorithm}
\usepackage{algpseudocode}
\newcommand\MYhyperrefoptions{bookmarks=true,bookmarksnumbered=true,
pdfpagemode={UseOutlines},plainpages=false,pdfpagelabels=true,
colorlinks=true,linkcolor={red},citecolor={green},urlcolor={magenta},
pdftitle={Soft-Label Caching and Sharpening for Communication-Efficient Federated Distillation},
pdfsubject={Typesetting},
pdfauthor={Kitsuya Azuma},
pdfkeywords={Federated learning, knowledge distillation, non-IID data, communication efficiency}}
\ifCLASSINFOpdf
\usepackage[\MYhyperrefoptions,pdftex]{hyperref}
\else
\usepackage[\MYhyperrefoptions,breaklinks=true,dvips]{hyperref}
\usepackage{breakurl}
\fi
\usepackage{bm}
\usepackage{booktabs}
\usepackage{multirow}
\usepackage{enumitem}
\usepackage{amsthm}
\usepackage{titlesec}
\usepackage{makecell}
\usepackage{xcolor}
\algrenewcommand{\alglinenumber}[1]{\color{gray}#1:}
\newcommand{\AlgComment}[1]{%
  \hfill\textcolor{gray}{\textit{$\triangleright$ #1}}%
}

\makeatletter
\renewcommand\subsubsection{\@startsection{subsubsection}{3}{\z@}{0.8ex \@plus 0.2ex \@minus 0.1ex}{0.4ex \@plus 0.1ex}{\normalfont\normalsize\bfseries}}
\makeatother

\theoremstyle{definition}
\newtheorem{theorem}{Theorem}

\titlespacing*{\paragraph}{0pt}{0.5em}{0em}

\begin{document}

\title{Soft-Label Caching and Sharpening for \\Communication-Efficient Federated Distillation}

\author{Kitsuya Azuma,~\IEEEmembership{Student Member,~IEEE,}
Takayuki Nishio,~\IEEEmembership{Senior Member,~IEEE,}\\
Yuichi Kitagawa,
Wakako Nakano,
and Takahito Tanimura
\thanks{This article has been accepted for publication in IEEE Transactions on Mobile Computing. This is an open access article under the CC BY license. DOI: \href{https://doi.org/10.1109/TMC.2026.3652819}{10.1109/TMC.2026.3652819}}
}


\urlstyle{tt}

\maketitle

\begin{abstract}
\urlstyle{tt}
Federated Learning (FL) enables collaborative model training across decentralized clients, enhancing privacy by keeping data local. Yet conventional FL, relying on frequent parameter-sharing, suffers from high communication overhead and limited model heterogeneity. Distillation-based FL approaches address these issues by sharing predictions (soft-labels, i.e., normalized probability distributions) instead, but they often involve redundant transmissions across communication rounds, reducing efficiency. We propose SCARLET, a novel framework integrating synchronized soft-label caching and an enhanced Entropy Reduction Aggregation (Enhanced ERA) mechanism. SCARLET minimizes redundant communication by reusing cached soft-labels, achieving up to 50\% reduction in communication costs compared to existing methods while maintaining competitive accuracy. Enhanced ERA resolves the fundamental instability of conventional temperature-based aggregation, ensuring robust control and high performance in diverse client scenarios. Experimental evaluations demonstrate that SCARLET consistently outperforms state-of-the-art distillation-based FL methods in terms of accuracy and communication efficiency. The implementation of SCARLET is publicly available at \url{https://github.com/kitsuyaazuma/SCARLET}.

\end{abstract}


\section{Introduction}

\IEEEPARstart{F}{ederated} Learning (FL) represents a paradigm shift in distributed machine learning, enabling model training across decentralized clients while preserving data privacy by keeping raw data local. Traditional centralized machine learning typically requires aggregating data at a central server, raising concerns over privacy, security, and compliance. FL mitigates these issues by facilitating collaborative training across clients, such as mobile or edge devices, where only model updates or predictions are communicated with a central server.

Conventional FL methods, such as Federated Averaging (FedAvg) \cite{fedavg}, face several fundamental challenges that hinder their scalability and practicality in real-world applications. 
These challenges include high communication costs from frequent transmission of model parameters, heterogeneous device capabilities and network conditions, convergence difficulties due to non-IID data, and privacy concerns related to the exposure of sensitive information \cite{challenges}. Addressing these issues is crucial for enhancing the efficiency and robustness of FL systems.

\begin{figure}[t]
\begin{center}
  \subfloat[Distillation-based FL]{%
    \includegraphics[width=0.95\columnwidth]{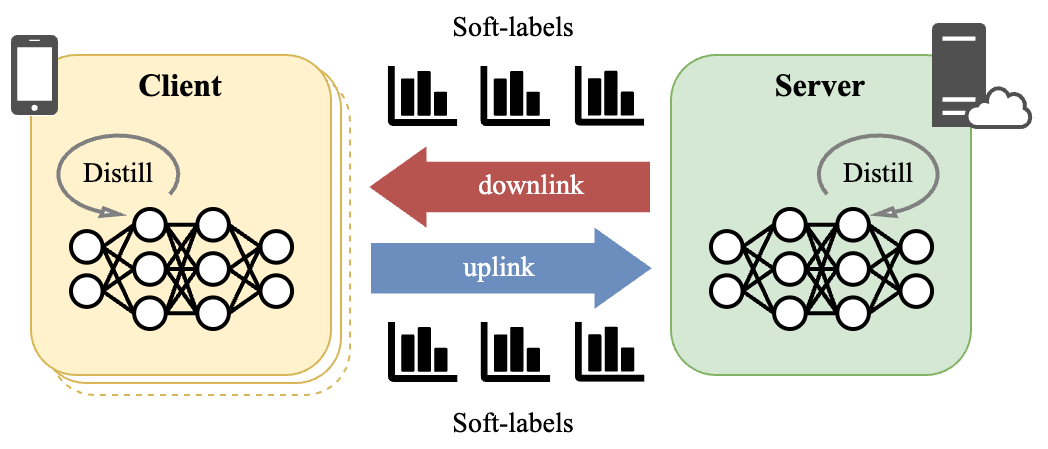}%
    \label{fig:concept:distillation_based_fl}%
  }

  \subfloat[SCARLET]{%
    \includegraphics[width=0.95\columnwidth]{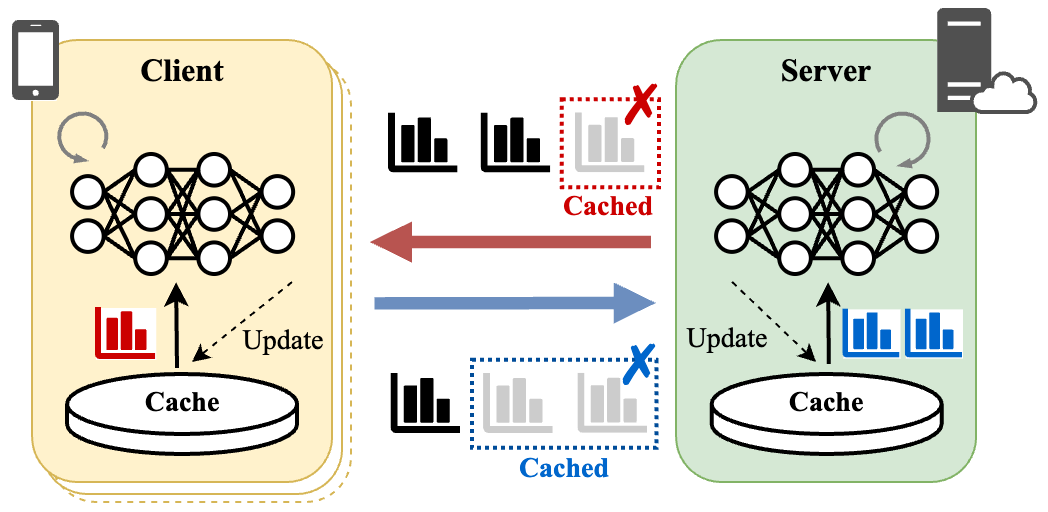}%
    \label{fig:concept:scarlet}%
  }
\end{center}
\caption{Conceptual overview of distillation-based FL approaches. Conventional distillation-based FL (\ref{fig:concept:distillation_based_fl}) transmits all soft-labels in each communication round. SCARLET (\ref{fig:concept:scarlet}) requests soft-labels only for data not present in the server cache and synchronizes cache updates between the server and clients, significantly reducing communication costs.}
\label{figure:concept}
\end{figure}

To mitigate these limitations, recent studies have explored knowledge distillation-based FL approaches, which offer several advantages over conventional parameter-sharing methods. Instead of transmitting full model parameters, these approaches share model predictions (soft-labels\footnote{In this paper, we refer to the normalized outputs of the softmax function as \textit{soft-labels}, and the unnormalized, pre-activation model outputs as \textit{logits}. A more detailed discussion of this notation is provided in Section \ref{subsection:background_and_notations}.}), effectively reducing communication overhead \cite{dsfl}, \cite{cfd}, \cite{fedgen}. They also provide greater flexibility in model architecture, addressing the systems heterogeneity challenge by allowing clients to operate with heterogeneous models of varying computational capacities \cite{ktpfl}, \cite{hbs}. Moreover, since knowledge distillation involves sharing only aggregated or processed outputs rather than raw model updates, it enhances privacy protection by minimizing the risk of exposing sensitive training data through shared gradients \cite{fedad}, \cite{fedkdprivacy}.

However, existing distillation-based FL methods exhibit significant inefficiencies. Specifically, these methods redundantly transmit predictions for identical public dataset samples across multiple communication rounds, despite minimal changes in these predictions. This redundancy results in unnecessary communication overhead. Furthermore, conventional entropy-based aggregation methods (e.g., ERA in \cite{dsfl}) suffer from a fundamental instability, as their temperature parameter is highly sensitive to input entropy variations. This instability makes robust tuning impractical and degrades performance in diverse non-IID scenarios.

In this work, we introduce \textbf{SCARLET} (Semi-supervised federated distillation with global CAching and Reduced soft-Label EnTropy), a novel distillation-based FL approach designed explicitly to resolve these inefficiencies. SCARLET employs synchronized prediction caching between the server and clients to significantly reduce redundant transmissions by reusing cached predictions across rounds, as illustrated in Fig. \ref{figure:concept}. Additionally, we propose Enhanced Entropy Reduction Aggregation (Enhanced ERA), which dynamically adjusts the aggregation sharpness, offering robust performance and simplified parameter tuning across various levels of client data heterogeneity in non-IID environments.

Our experiments demonstrate that SCARLET consistently outperforms other distillation-based FL methods, achieving target accuracy with reduced communication and faster convergence. This contribution establishes SCARLET as a promising approach for efficient and scalable FL across heterogeneous client environments.

The key contributions of this paper are summarized as follows:
\begin{itemize} 
    \item \textbf{Novel Concept of Soft-Label Caching:}
    We propose a novel soft-label caching mechanism, which significantly reduces communication overhead by caching predictions for reuse across rounds. To the best of our knowledge, this is the first caching mechanism explicitly designed to eliminate redundant transmission of soft-labels in distillation-based FL, setting a foundation for communication-efficient FL methods.
    
    \item \textbf{Development of Enhanced ERA:} We propose Enhanced ERA, an aggregation mechanism that replaces the mathematically unstable, temperature-based tuning of conventional ERA with an intuitive and provably robust, ratio-based control over aggregation sharpness, ensuring stable performance across diverse data distributions from strong non-IID to IID.
    
    \item \textbf{Extensive Comparative Evaluations:}  
    We comprehensively evaluate SCARLET against state-of-the-art distillation-based FL methods, demonstrating consistent superiority in accuracy, convergence speed, and communication efficiency across diverse non-IID scenarios.
    
    \item \textbf{Open-Source Implementations for Reproducibility:}  
    We publicly release the source code for SCARLET along with our re-implementations of multiple state-of-the-art baseline methods. This facilitates transparent and reproducible evaluations and supports future research in communication-efficient FL.
\end{itemize}

The remainder of this paper is organized as follows: Section \ref{related_work} reviews related work on parameter-sharing and distillation-based FL methods. Section \ref{proposed_method} describes the SCARLET framework, including the prediction caching mechanism and adaptive aggregation strategies. Section \ref{section:experiments} presents our experimental evaluation. Section \ref{limitations} discusses the limitations of the proposed method, and Section \ref{conclusion} concludes with a summary of findings and future directions.

\section{Related Work}
\label{related_work}
\subsection{Federated Learning (FL)}
Federated Learning (FL) is a framework for training machine learning models across decentralized clients while preserving data privacy. Since its introduction by \cite{fedavg} with the Federated Averaging (FedAvg) algorithm, FL research has focused on challenges like communication efficiency, non-IID data distributions, and scalability.

Handling heterogeneous data, or non-IID  data, where client distributions differ significantly, is a key challenge in FL. Traditional FL methods such as FedAvg are not optimized for these non-IID settings, resulting in slower convergence and reduced accuracy. Various solutions have been proposed, including personalized FL methods like pFedMe \cite{pfedme}, which allows client-specific adaptation, and regularization-based approaches such as FedProx \cite{fedprox}, which adds a penalty term to local updates to mitigate client drift. SCAFFOLD \cite{scaffold} addresses local model divergence using control variates, and FedNTD \cite{fedntd} employs selective distillation to retain knowledge from “not-true” classes, maintaining a global perspective without extra communication.

Despite these innovations, most methods still rely on parameter sharing, which incurs communication costs. Efficient knowledge transfer remains crucial, especially in non-IID environments.

\subsection{Communication-Efficient FL}
Communication efficiency in FL minimizes data transmission between clients and the server, addressing bandwidth and latency concerns. Techniques include quantization, sparsification, and knowledge distillation.

Quantization reduces the bit precision of model updates, thus lowering data size. FedPAQ \cite{fedpaq} applies quantization to compressed client updates for efficiency. FedCOMGATE \cite{fedcomgate} combines quantization with gradient tracking, aligning local and global updates to support non-IID data while reducing communication.

Sparsification transmits only a subset of gradient elements, focusing on key information. Adaptive Gradient Sparsification \cite{ags} dynamically adjusts sparsity based on client needs, using a bidirectional top-k approach that retains key gradients to balance communication efficiency and model accuracy.

Knowledge distillation, which shares distilled knowledge (such as output logits) instead of full parameters, allows model heterogeneity by transmitting only essential information. This approach accommodates model heterogeneity among clients by transferring essential output information rather than the model structure itself. Within distillation methods, approaches like FedDF \cite{feddf} aggregate client outputs on the server, allowing clients to download an updated model without the need to exchange full parameters. Similarly, FedKD \cite{fedkd} reduces communication costs by employing mutual distillation between a global model and a local model while dynamically compressing exchanged gradients, achieving significant communication savings without sacrificing performance. However, our research focuses on a different class of distillation methods that exclusively exchange model outputs bidirectionally between clients and the server, further enhancing communication efficiency.

\subsection{Distillation-Based FL}
Distillation-based FL reduces the communication overhead from parameter exchange by transmitting only model outputs (logits or soft-labels). Unlike traditional knowledge distillation \cite{kd}, which compresses models, distillation-based FL enables collaborative learning among heterogeneous clients, managing model heterogeneity while reducing communication.

Some studies use public data in distillation. FedMD \cite{fedmd} enables clients with different models to collaborate by sharing logits derived from a labeled public dataset, allowing clients to refine models without sharing raw data. FD \cite{fd} reduces communication by aggregating label-wise average soft-labels from clients’ private data, while DS-FL \cite{dsfl} improves on FD by using unlabeled public data and introducing Entropy Reduction Aggregation (ERA) to handle high-entropy aggregated soft-labels in non-IID settings.

Recent methods further optimize communication, personalization, and non-IID handling. CFD \cite{cfd} employs soft-label quantization and delta coding to reduce data transmission, and COMET \cite{comet} clusters clients with similar data, sharing only soft-labels   
for improved personalization. Selective-FD \cite{selectivefd} filters ambiguous knowledge through selective sharing, using client-side selectors and a server-side selector to handle out-of-distribution samples and high-entropy predictions.

However, existing distillation-based FL methods suffer from two fundamental limitations. First, they involve redundant communication, as they do not evaluate how significantly soft-labels are updated in each round. Second, their aggregation mechanisms (e.g., conventional ERA) are often mathematically unstable, leading to poor performance in non-IID settings. These limitations motivate our proposed method, SCARLET, which explicitly addresses such inefficiencies by introducing both a soft-label caching mechanism and the stable Enhanced ERA mechanism.

\subsection{FL with Caching Mechanisms}
Several studies have proposed FL methods incorporating caching mechanisms. CacheFL \cite{cachefl} employs a caching mechanism that locally stores global models to reduce training time, enabling resource-limited clients to quickly access cached models. FedCache \cite{fedcache} utilizes sample-specific hash mappings and server-side caching to efficiently deliver personalized knowledge without relying on public datasets. FedCache 2.0 \cite{fedcache2} replaces logit exchanges with dataset distillation, treating the distilled datasets as a centralized knowledge cache on the server, from which clients download information aligned with their data.

However, these existing caching approaches fundamentally differ from the soft-label caching mechanism proposed in our method, SCARLET. Specifically, SCARLET’s caching mechanism explicitly aims to reduce the communication overhead associated with redundant soft-label transmission in each communication round of distillation-based FL, addressing distinct challenges and objectives compared to previous methods.

\section{Proposed Method}
\label{proposed_method}
\subsection{Background and Notations}
\label{subsection:background_and_notations}
Consider a cross-device FL setting where $K$ clients are participating while communicating with a central server. We assume an $N$-class classification task. Each client $k \in [K]$ has its classification model $\mathbf{w}_k$ and labeled private dataset $\mathcal{B}_k$. Each data sample $b \in \mathcal{B}_k$ is a pair $(\mathbf{x},y)$ where $\mathbf{x} \in \mathbb{R}^d$ is the input and $y \in [N]$ is the class label.

In FedAvg \cite{fedavg}, a standard FL method, each client $k \in [K]$ first updates its model $\mathbf{w}_k$ with its private dataset $\mathcal{B}_k$ using stochastic gradient descent (SGD). Then, each client uploads its updated model parameters or gradients to the server. The server aggregates these uploaded parameters to obtain the updated global model \(\mathbf{w}_g\) as follows:
\begin{equation}
    \mathbf{w}_g = \sum_{k=1}^{K} \frac{\lvert \mathcal{B}_k \rvert}{ \sum_{j=1}^{K} \lvert \mathcal{B}_j \rvert} \mathbf{w}_k,
\end{equation}
where \(\lvert \cdot \rvert\) denotes the cardinality (size) of a set.
Subsequently, the server broadcasts the global model $\mathbf{w}_g$ to all clients. This procedure is repeated for a finite number of communication rounds.

In DS-FL \cite{dsfl}, a distillation-based FL method that serves as the baseline for SCARLET, each client and the server exchange soft-labels instead of model parameters. We assume that each client $k \in [K]$ has not only its labeled private dataset $\mathcal{B}_k$ but also the public unlabeled dataset $\mathcal{P}$. After each client updates its model $\mathbf{w}_k$ as in FedAvg, it generates local soft-labels $z_{k}^{t} = \sigma(\mathbf{x}; \mathbf{w}_k)$ for $\mathbf{x} \in \mathcal{P}^t$, where $\mathcal{P}^t$ is unlabeled public dataset used in communication round $t$. The local soft-labels $z_{k}^{t}$ are uploaded from each client to the server, and the server aggregates them to create the global soft-labels $\hat{z}^t$ with ERA as follows:
\begin{equation}
    \label{equation:era}
    \hat{z}^t = \text{Softmax} \left( \left. \frac{1}{K} \sum_{k=1}^{K} z_{k}^{t} \right| T \right),
\end{equation}
where $\text{Softmax}(\cdot \mid T)$ denotes the softmax function with temperature $T$. As the temperature decreases, the entropy of the softmax function's output also decreases. Subsequently, the server broadcasts the global soft-label $\hat{z}^t$ to all clients. In the next communication round, the clients first update their local model with the broadcasted soft-labels $\hat{z}^{t-1}$ and the public unlabeled dataset $\mathcal{P}^{t-1}$ as follows:

\begin{equation}
    \mathbf{w}_{k}^{t} \gets \mathbf{w}_{k}^{t}  - \eta_\mathrm{dist} \nabla \phi_\mathrm{dist} (\mathcal{P}^{t-1}, \mathbf{w}_{k}^{t}, \hat{z}^{t-1}),
\end{equation}
where $\eta_\mathrm{dist}$ is the learning rate and $\phi_\mathrm{dist} (\cdot, \mathbf{w}_{k}^{t}, \cdot)$ is the loss function in the distillation procedure. The loss function $\phi_\mathrm{dist}$ is exemplified by the Kullback–Leibler divergence. Along with the clients’ local models, the global model $\mathbf{w}_{g}^{t}$ on the server is also updated with the global soft-labels $\hat{z}^{t}$ and public unlabeled dataset $\mathcal{P}^{t}$.

To avoid ambiguity, we clarify that in this paper, we refer to the outputs shared in distillation-based FL methods as \textit{soft-labels} rather than \textit{logits}. Although the term \textit{logits} is often used in some prior works, such as \cite{dsfl}, \cite{fd}, to refer to the output of the softmax function, it more commonly denotes the unnormalized model output before softmax function. Consequently, we will consistently use the term \textit{soft-labels} to describe the normalized outputs used in the distillation process, aligning with standard terminology in classification tasks.

For clarity, we list the notation used for the paper in Table \ref{table:notations}. Superscripts indicate the round index, while subscripts denote the client index.

\begin{table}[t]
\centering
\caption{List of Key Notations}
\label{table:notations}
\begin{tabular}{cl}
\toprule
\textbf{Notation} & \textbf{Description} \\
\midrule
$T$ & Total number of communication rounds \\
$E$ & Number of local iterations per round \\
$K$ & Total number of clients \\
$\mathcal{B}_k$ & Labeled private dataset of client $k \in [K]$ \\
$\mathcal{P}$ & Unlabeled public dataset \\
$\mathcal{P}^t$ & Unlabeled public dataset used in round $t \in [T]$ \\
$\mathcal{I}^t$ & Indices of $\mathcal{P}^t$ within $\mathcal{P}$ \\
$N$ & Number of classes for the classification task \\
$\mathbf{w}_k$ & Classification models of each client $k \in [K]$ \\
$\mathbf{w}_g$ & Classification model of the server \\
$\sigma(\mathbf{x}; \mathbf{w})$ & Soft-labels generated by model $\mathbf{w}$ for input data $\mathbf{x}$ \\
\bottomrule
\end{tabular}
\end{table}

\begin{algorithm}[!t]
\caption{SCARLET (Full Participation Scenario)}
\label{algorithm:scarlet}
\begin{algorithmic}[1]

\State \textbf{Server Initialization:}
\State Initialize global model $\mathbf{w}_g$ and all client models $\{\mathbf{w}_k\}_{k=1}^K$
\State Initialize global cache $\mathcal{C}_g$ and all local caches $\{\mathcal{C}_k\}_{k=1}^K$
\State Set cache duration $D$
\State Distribute public dataset $\mathcal{P}$ to all clients $k \in [K]$
\State
\For {$t=1,\cdots,T$ \textbf{communication rounds}}\textbf{:}
    \State \textbf{Server:}
    \State Randomly select subset $\mathcal{P}^t \subset \mathcal{P}$
    \State Compute $\mathcal{I}_{\text{req}}^{t} \gets \{i \in \mathcal{I}^t \mid \mathcal{C}_{g}(i) \text{ does not exist} \}$ \newline \hspace*{2em} \AlgComment{Identify samples not cached to request} \label{alg:scarlet:compute_req}
    \State Send $\mathcal{I}_{\text{req}}^{t}$ to all clients $k \in [K]$
    \State
    \If{$t > 1$}
        \State Send $\gamma^{t-1}, \hat{z}_{\text{req}}^{t-1}, \mathcal{I}^{t-1}$ to all clients $k \in [K]$ \label{alg:scarlet:send_signals}
    \EndIf

    \State
    \State \textbf{Client $k \in [K]$ (in parallel):}
    \State $\mathbf{w}_{k}^{(t,0)} \gets \mathbf{w}_{k}^{(t-1,E)}$
    \If{$t > 1$}
        \State $\hat{z}^{t-1} \gets $ \newline \hspace*{4em} \Call{UpdateLocalCache}{$\mathcal{C}_k$, $\gamma^{t-1}$, $\hat{z}_{\text{req}}^{t-1}$, $\mathcal{I}^{t-1}$} \newline \hspace*{6em}\AlgComment{See Algorithm~\ref{algorithm:soft-label_caching_module} for definition}
        \State $\mathbf{\tilde{w}}_{k}^{(t,0)}\gets \mathbf{w}_{k}^{(t,0)}$
        \For {$i=1,\cdots,E_\text{dist}$}\textbf{:} \AlgComment{Distillation step} \label{alg:scarlet:distillation_step}
            \State $\mathbf{\tilde{w}}_{k}^{(t,i)} \gets \newline \hspace*{6em} \mathbf{\tilde{w}}_{k}^{(t,i-1)}  - \eta_\text{dist} \nabla \phi_\text{dist} (\mathcal{P}^{t-1}, \mathbf{\tilde{w}}_{k}^{(t,i-1)}, \hat{z}^{t-1}) $
        \EndFor
        \State $\mathbf{w}_{k}^{(t,0)}\gets \mathbf{\tilde{w}}_{k}^{(t,E_{\text{dist}})}$
    \EndIf
    \For {$j=1,\cdots,E$}\textbf{:} \AlgComment{Local training step}
        \State $\mathbf{w}_{k}^{(t,j)} \gets \mathbf{w}_{k}^{(t,j-1)} - \eta \nabla \phi (\mathcal{B}_k, \mathbf{w}_{k}^{(t,j-1)})$
    \EndFor
    \State $\mathcal{P}_{k}^{t} \gets \{\mathcal{P}[i] \mid i \in \mathcal{I}_{\text{req}}^{t}\}$
    \State Send soft-labels $z_{k}^{t} \gets \sigma(\mathcal{P}_{k}^{t}, \mathbf{w}_{k}^{(t,E)})$ to server \label{alg:scarlet:client_send}
    \State
    \State \textbf{Server:}
    \State Set $\mathbf{w}_{g}^{(t,0)} \gets \mathbf{w}_{g}^{(t-1,E)}$
    \State Aggregate $\{z_{k}^{t}\}_{k \in [K]}$ with Enhanced ERA to get $\hat{z}_{\text{req}}^{t}$ \newline \hspace*{6em} \AlgComment{See (Eq.~\ref{equation:enhanced_era}}) \label{alg:scarlet:aggregate}
    \State $\hat{z}_{\text{cached}}^{t} \gets \{ \mathcal{C}_g(i) \mid i \in \mathcal{I}^t \setminus \mathcal{I}_{\text{req}}^{t} \} \newline \hspace*{4em} $ \AlgComment{Retrieve soft-labels from global cache}
    \State $\hat{z}^t \gets \hat{z}_{\text{req}}^{t} \cup \hat{z}_{\text{cached}}^{t}$ \AlgComment{Assemble full soft-label set}
    \For {$i=1,\cdots,E_{\text{dist}}$}\textbf{:} \AlgComment{Global distillation step}
        \State $\mathbf{w}_{g}^{(t,i)} \gets \mathbf{w}_{g}^{(t,i-1)} - \eta_{\text{dist}} \nabla \phi_{\text{dist}} (\mathcal{P}^{t}, \mathbf{w}_{g}^{(t,i-1)}, \hat{z}^t)$
    \EndFor
    \State $\gamma^t \gets$ \Call{UpdateGlobalCache}{$\mathcal{C}_g$, $\hat{z}^t$, $\mathcal{I}^t$, $t$, $D$} \newline \hspace*{6em}\AlgComment{See Algorithm~\ref{algorithm:soft-label_caching_module} for definition} \label{alg:scarlet:update_global_cache}
\EndFor

\end{algorithmic}
\end{algorithm}

\subsection{SCARLET: System Design and Workflow}
SCARLET (Semi-supervised federated distillation with global CAching and Reduced soft-Label EnTropy) is a novel distillation-based FL framework to address inefficiencies observed in existing distillation-based FL methods. SCARLET integrates synchronized soft-label caching and an adaptive aggregation mechanism called Enhanced ERA, significantly enhancing communication efficiency and adaptability to heterogeneous client environments.

Algorithm~\ref{algorithm:scarlet} describes the overall workflow of SCARLET. At the beginning of each communication round $t$, the server selects a subset of the public dataset, $\mathcal{P}^t$, to be used for distillation.

While the SCARLET framework can accommodate both full and partial client participation, we first assume the fundamental case of full client participation (i.e., all $K$ clients participate in every round) for analytical clarity. This full participation setting allows a focused analysis on the introduction of our two primary contributions: the soft-label caching mechanism (Section~\ref{subsubsection:cache_module} and Algorithm~\ref{algorithm:soft-label_caching_module}) and the Enhanced ERA mechanism (Section~\ref{subsection:enhanced_era}).

The necessary considerations for the more practical and complex scenario of partial client participation (i.e., where clients participate intermittently), which extends this fundamental case, will be detailed separately in Section \ref{subsection:partial_participation}. The subsequent experimental evaluation in Section~\ref{subsection:experimental_results} will then validate SCARLET's effectiveness under both of these settings: the full participation scenario (Sections \ref{subsubsection:comparative_evaluaion} to \ref{subsubsection:ablation_study_of_enhanced_era}) and the partial participation scenario (Section \ref{subsubsection:evaluation_under_partial_client_participation}).

While the workflow follows that of the baseline DS-FL \cite{dsfl}, SCARLET introduces two major differences.

First, after local training, each client selectively transmits soft-labels only for public dataset samples that are not currently cached on the server. This selective transmission is based on the request list $\mathcal{I}_{\text{req}}^{t}$ computed by the server (Algorithm~\ref{algorithm:scarlet}, line~\ref{alg:scarlet:compute_req}). This is managed by a global cache $\mathcal{C}_g$ on the server, which stores previously aggregated soft-labels. This caching mechanism, detailed in Section~\ref{subsubsection:cache_module}, synchronizes soft-label storage between the server and clients via cache signals $\gamma^t$ (generated in line~\ref{alg:scarlet:update_global_cache} and sent in line~\ref{alg:scarlet:send_signals}), significantly reducing communication overhead compared to DS-FL, which retransmits soft-labels in every round.

Second, SCARLET introduces a novel aggregation process at the server after collecting soft-labels (Algorithm~\ref{algorithm:scarlet}, line~\ref{alg:scarlet:aggregate}). We propose Enhanced ERA, an adaptive aggregation strategy that enables robust control over the entropy of aggregated predictions. In contrast to the conventional ERA used in DS-FL, which often requires careful tuning of the temperature parameter, Enhanced ERA adopts a new parameterization that allows flexible and stable adjustment of label sharpness. This design enhances robustness and stability under varying degrees of data heterogeneity across clients. Enhanced ERA is described in detail in Section~\ref{subsection:enhanced_era}.

Together, the soft-label caching mechanism and Enhanced ERA distinctly position SCARLET as a highly communication-efficient and robust alternative to existing distillation-based FL methods. The subsequent sections provide detailed descriptions and analysis of these components.

\subsection{Soft-Label Caching in SCARLET}
\label{subsubsection:cache_module}

\begin{figure*}[t]
\begin{center}
  \subfloat[Public data (not cached)]{%
    \includegraphics[height=6.8cm]{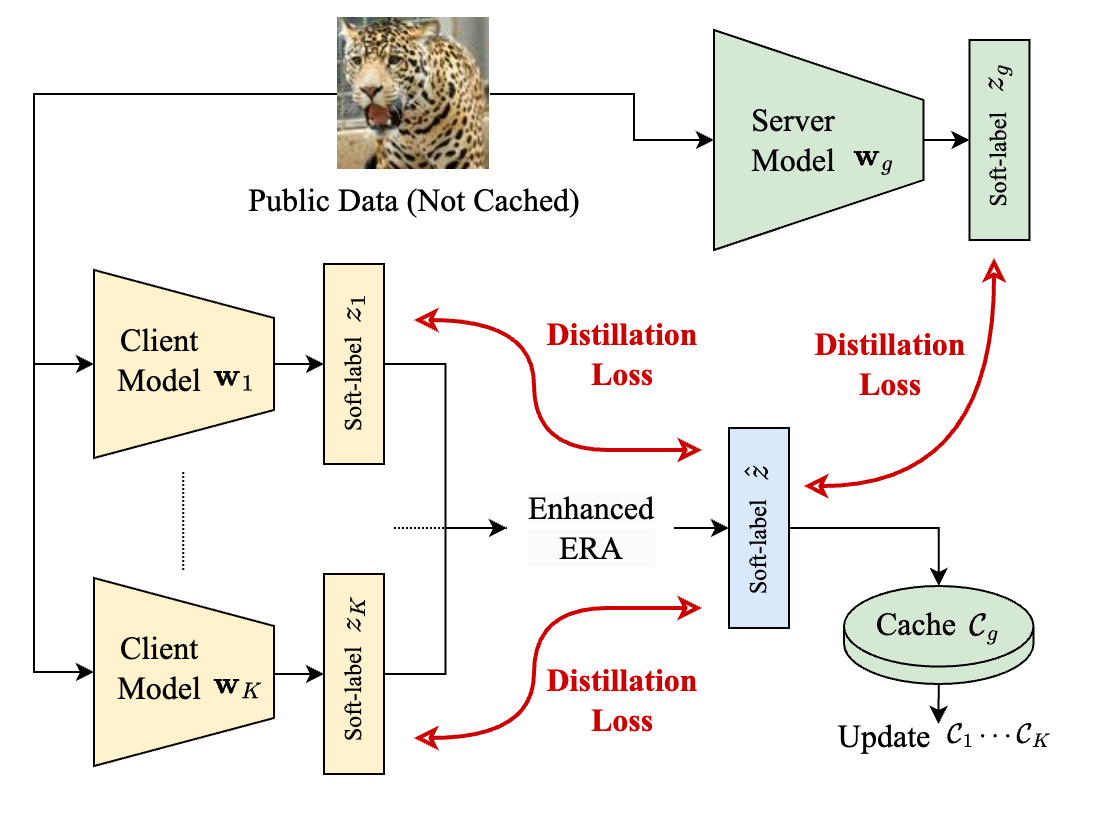}%
    \label{fig:detail:not_cached}%
  }\hfill
  \subfloat[Public data (cached)]{%
    \includegraphics[height=6.8cm]{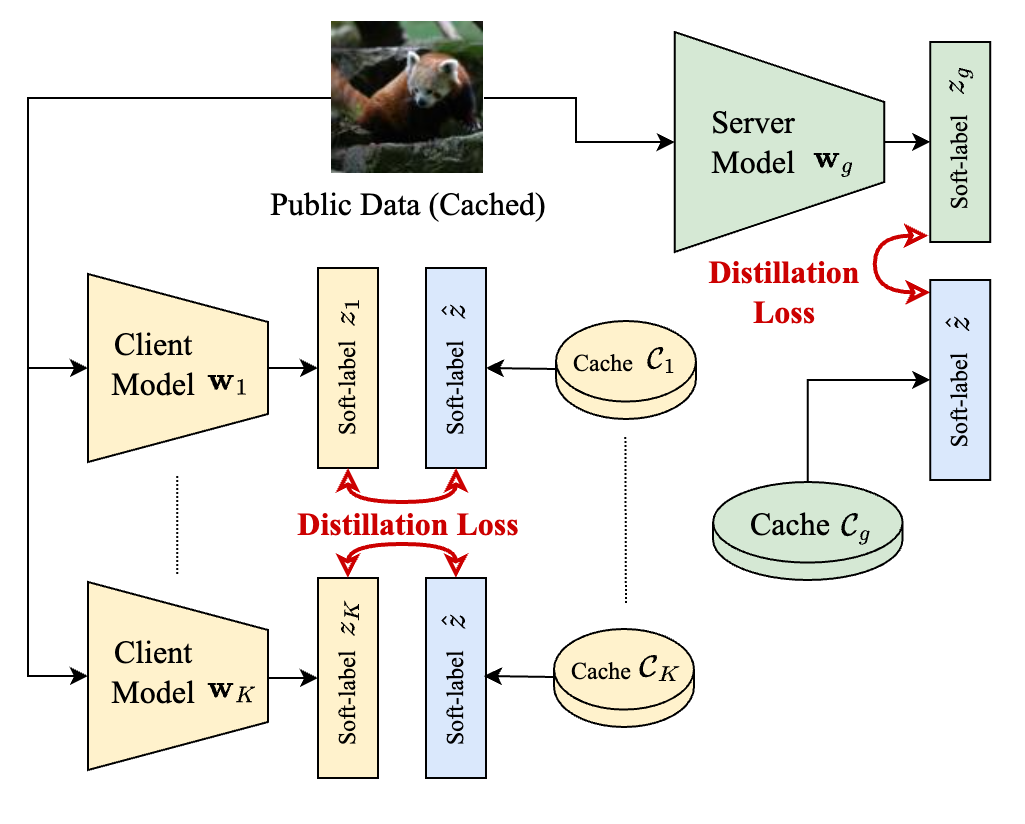}%
    \label{fig:detail:cached}%
  }
\end{center}
\caption{Illustration of the soft-label caching mechanism in SCARLET. When public data is not cached (\ref{fig:detail:not_cached}), clients transmit their soft-labels to the server, where they are aggregated using Enhanced ERA and stored in the server cache before synchronizing with client caches. When public data is cached (\ref{fig:detail:cached}), both the server and clients retrieve soft-labels from their respective caches for knowledge distillation, eliminating the need for additional soft-label communication.}
\label{figure:detail}
\end{figure*}

\begin{algorithm}[t]
\caption{Soft-Label Caching}
\label{algorithm:soft-label_caching_module}
\begin{algorithmic}[1]

\Function{UpdateGlobalCache}{$\mathcal{C}_g$, $\hat{z}^t$, $\mathcal{I}^t$, $t$, $D$}
    \State \textbf{Input:} Cache $\mathcal{C}_g$, soft-labels $\hat{z}^t$, indices $\mathcal{I}^t$,\newline \hspace*{5em}current round $t$, cache duration $D$
    \State \textbf{Output:} Cache signals $\gamma^t$
    \State Initialize $\gamma^t \gets [\;]$
    \ForAll {$i, z \in (\mathcal{I}^t, \hat{z}_t)$}
        \If {$\mathcal{C}_{g}(i) \text{ does not exist}$}
            \State $\mathcal{C}_g (i) \gets (z, t)$
            \State Append \texttt{NEWLY\_CACHED} to $\gamma^t$
        \Else
            \State $(z_c, t_c) \gets \mathcal{C}_g(i)$
            \If {$t - t_c \leq D$}
                \State Append \texttt{CACHED} to $\gamma^t$
            \Else
                \State Delete $\mathcal{C}_g(i)$
                \State Append \texttt{EXPIRED} to $\gamma^t$
            \EndIf
        \EndIf
    \EndFor
    \State \Return $\gamma^t$
\EndFunction
\State
\Function{UpdateLocalCache}{$\mathcal{C}_k$, $\gamma^{t-1}$, $\hat{z}_{\text{req}}^{t-1}$, $\mathcal{I}^{t-1}$}
    \State \textbf{Input:} Cache $\mathcal{C}_k$, cache signals $\gamma^{t-1}$,\newline \hspace*{5em}soft-labels $\hat{z}_{\text{req}}^{t-1}$, indices $\mathcal{I}^{t-1}$
    \State \textbf{Output:} Soft-labels $\hat{z}^{t-1}$
    \State Initialize $\hat{z}^{t-1} \gets [\;]$
    \State Initialize queue $Q \gets \hat z_{\mathrm{req}}^{t-1}$  \AlgComment{treat as FIFO queue}
    
    \ForAll {$i, z \in (\mathcal{I}^{t-1}, \hat{z}^t)$}
        \If {$\gamma^{t-1}[i]$ is \texttt{NEWLY\_CACHED}}
            \State $\mathcal{C}_k (i) \gets Q.\text{pop}()$
            \State Append $\mathcal{C}_k(i)$ to $\hat{z}^{t-1}$
        \ElsIf {$\gamma^{t-1}[i]$ is \texttt{CACHED}}
            \State Append $\mathcal{C}_k(i)$ to $\hat{z}^{t-1}$
        \ElsIf {$\gamma^{t-1}[i]$ is \texttt{EXPIRED}}
            \State Delete $\mathcal{C}_k(i)$
            \State Append $Q.\text{pop}()$ to $\hat{z}^{t-1}$
        \EndIf
    \EndFor

    \State \Return $\hat{z}^{t-1}$
\EndFunction

\end{algorithmic}
\end{algorithm}

The \textbf{soft-label caching mechanism} is introduced in SCARLET to address the communication inefficiency caused by redundant transmission of soft-labels in distillation-based FL methods. Specifically, SCARLET reduces communication overhead by caching and reusing soft-labels across communication rounds.

In general, distillation-based FL methods require clients to transmit soft-labels for the same public dataset samples across different rounds, even when the variations in these labels between rounds are minimal. This redundancy arises because local models tend to overfit private datasets during training, resulting in only marginal changes in soft-labels for public data samples. SCARLET addresses this inefficiency by caching and reusing soft-labels for a specified duration, thus reducing redundant transmissions.

The caching process operates on both the server and clients, as illustrated in Fig.~\ref{figure:detail}, ensuring efficient reuse of soft-labels during distillation. The detailed procedure is described in  Algorithm~\ref{algorithm:soft-label_caching_module}.
At the server side, a global cache $\mathcal{C}_g$ maintains aggregated soft-labels along with the corresponding indices of public dataset samples and their timestamps. At the beginning of each communication round, the server identifies which samples in the selected public subset $\mathcal{P}^t$ are already cached---i.e., those with valid and non-expired entries in $\mathcal{C}_g$ based on a predefined cache duration $D$. Only the samples that are not cached (either not previously stored or whose entries have expired) are selected for soft-label updates. This selective update mechanism significantly reduces the amount of data transmitted during each round.

Similarly, clients maintain local caches $\mathcal{C}_k$, which store the soft-labels received from the server. During the distillation process, clients attempt to retrieve soft-labels for public dataset samples from their local caches. For samples that are not cached, the client uses newly received soft-labels from the server and updates its cache accordingly. This dual caching strategy across server and clients minimizes communication overhead while preserving the effectiveness of the distillation process.

The caching mechanism introduces minimal additional costs compared to the communication overhead it reduces. Each round, the server sends cache synchronization signals, denoted as $\gamma^t$ in Algorithm~\ref{algorithm:soft-label_caching_module}, to clients. These signals indicate the cache status of public dataset samples and can be represented as integers, resulting in an additional communication cost of \( O(\lvert \mathcal{P}^t \rvert) \), where \( \lvert \mathcal{P}^t \rvert \) is the number of public dataset samples selected per round. For memory, both the server and clients require \( O(\lvert \mathcal{P} \rvert \times N) \) to cache soft-labels, where \( \lvert \mathcal{P} \rvert \) is the total size of the public dataset and \( N \) is the number of classes. This memory usage is small compared to model sizes in FL and can be further optimized by storing caches on disk if necessary.

\begin{figure}[t]
\centering
\includegraphics[width=\linewidth]{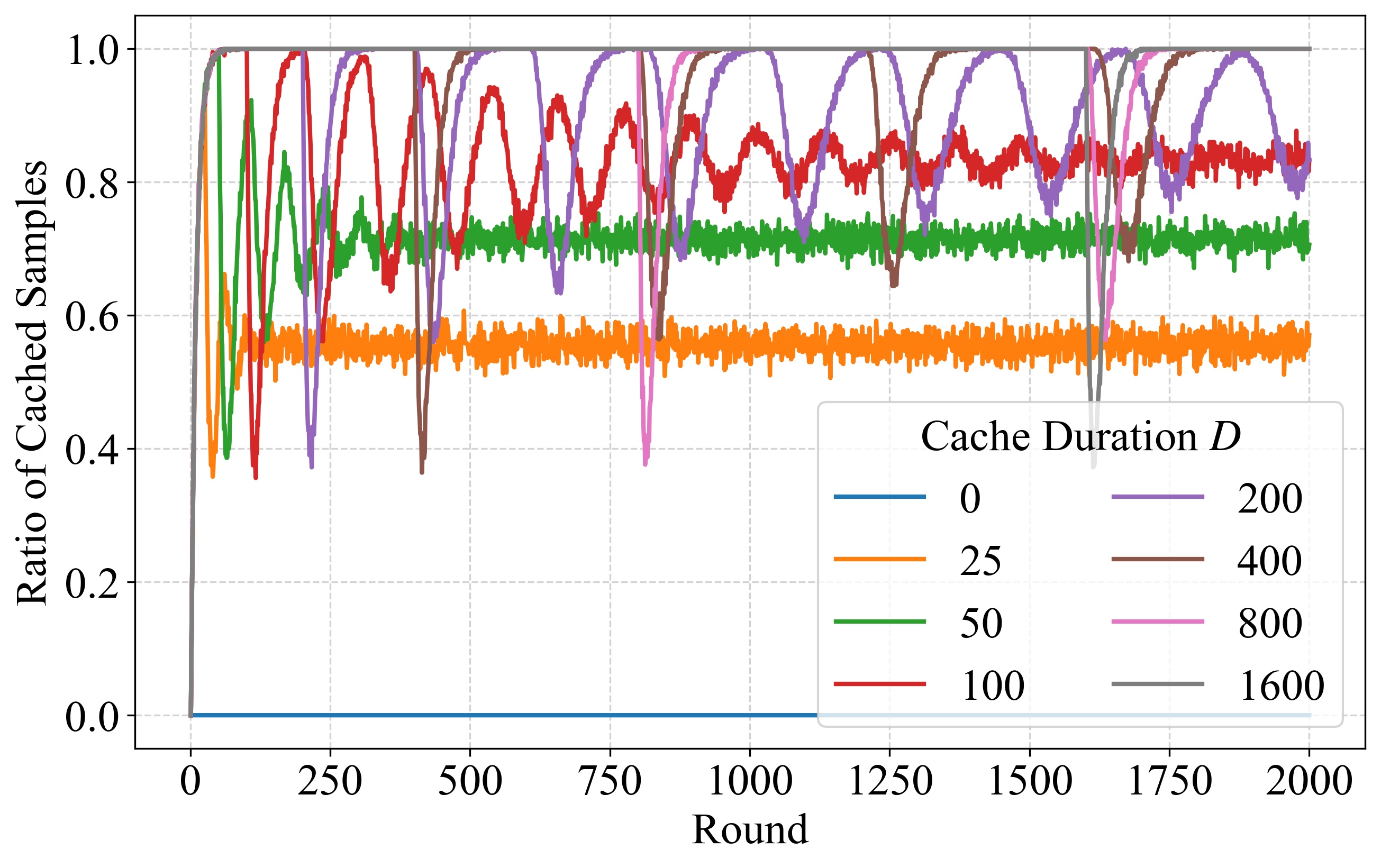}
\caption{Simulation of the ratio of cached samples per round as a function of cache duration $D$. Cached samples are defined as those in $\mathcal{P}^t$ that are found in the local cache and have not expired. The communication cost decreases as the ratio of cached samples increases, since only uncached samples require transmission.}
\label{fig:cache_simulation}
\end{figure}

A key advantage of this mechanism is that the impact of $D$ can be estimated using a lightweight, standalone simulation without executing the full FL process. This simulation, detailed in Appendix~\ref{appendix:cache_simulation}, models only the random sampling process---not the FL training itself---to rapidly predict the cache hit rate over time for a given $D$. Fig.~\ref{fig:cache_simulation} shows the result of such a simulation, modeling a scenario where 10\% of the total public dataset $\mathcal{P}$ is randomly selected as $\mathcal{P}^{t}$ each round. The figure confirms that $D$ dynamically controls the ratio of cached samples. Due to the random sampling, this ratio fluctuates before decaying and stabilizing over the rounds.

This simulation reveals that $D$ directly governs the cache's refresh frequency. In distillation-based FL, the transmission of soft-labels for these public samples constitutes the majority of communication cost. Therefore, even a conservative duration, such as $D < 100$, results in a high proportion of public data samples being cached, offering significant communication savings. However, for $D \ge 200$, the simulation shows periods where the cache ratio approaches 1.0. During these rounds, all samples used for distillation are retrieved from the cache. This means the system is training on identical, outdated soft-labels, and communication is reduced to only the cache signals. This risks both accuracy degradation from stale knowledge and wasted communication rounds.

This analysis of the simulation results clarifies the fundamental trade-off: As $D$ increases, a greater number of selected samples remain valid in the cache across rounds, leading to a higher cache hit rate and a corresponding reduction in the need to retransmit soft-labels. However, as the simulation highlights, excessively large values of $D$ reduce the frequency of cache updates, increasing the likelihood of reusing outdated soft-labels---particularly those that significantly deviate from the true labels----which may negatively impact training. These results demonstrate that the caching module effectively reduces communication overhead while providing a tunable trade-off between the reuse of cached soft-labels and their freshness, controlled via the cache duration parameter $D$.

This lightweight simulation thus provides a valuable tool for identifying a reasonable $D$ that balances cost reduction with cache freshness. The practical guideline for selecting $D$ using this simulation procedure, as well as its effectiveness, is discussed in the final paragraph of Section~\ref{subsubsection:ablation_study_of_cache_duration}.

\subsection{Handling Partial Client Participation}
\label{subsection:partial_participation}

When clients participate intermittently, the integrity of the local cache ($\mathcal{C}_k$) becomes a critical challenge. A client rejoining a session after being offline for several rounds will hold a stale cache, which is not synchronized with the server's global cache ($\mathcal{C}_g$). This client cannot correctly reconstruct the full soft-label set $\hat{z}^{t-1}$ from the standard update package ($\gamma^{t-1}, \hat{z}_{req}^{t-1}$) as described in Algorithm~\ref{algorithm:soft-label_caching_module}.

To address this, we introduce a ``catch-up package'' mechanism. When a client $k$ is selected for round $t$, the server checks if the client participated in the previous round $t-1$. If the client is synchronized, having participated in $t-1$, it receives the standard, lightweight update package as described in Algorithm~\ref{algorithm:scarlet} (line~\ref{alg:scarlet:send_signals}), containing only the cache signals $\gamma^{t-1}$ and the soft-labels for newly computed samples $\hat{z}_{req}^{t-1}$.

Conversely, if the client is stale, meaning it did not participate in round $t-1$, its local cache is not synchronized with the latest updates. Therefore, in addition to the standard lightweight package for round $t-1$, the server sends a catch-up package. This supplementary package contains the aggregated soft-labels required for the client to update its cache for all the rounds it was inactive. Before performing the distillation step, the client first uses this catch-up package to update its local cache $\mathcal{C}_k$ for all indices in $\mathcal{I}^{t-1}$, ensuring it is fully synchronized. It then proceeds with the distillation step (Algorithm~\ref{algorithm:scarlet}, line~\ref{alg:scarlet:distillation_step}) using the now-complete cache information.

This mechanism guarantees that all participating clients possess the correct teacher soft-labels for distillation. On the downlink, this incurs an additional communication cost specifically for stale clients, corresponding to the differential updates accumulated since their last participation to synchronize with the server's latest cache, thereby avoiding the need to synchronize all clients in every round. In contrast, the uplink contribution remains consistently communication-efficient: clients strictly compute and transmit soft-labels $z_k^t$ only for the samples in the server's request list $\mathcal{I}_{req}^t$ (Algorithm~\ref{algorithm:scarlet}, line~\ref{alg:scarlet:client_send}).

\subsection{Enhanced ERA in SCARLET}
\label{subsection:enhanced_era}

\begin{figure}[t]
  \centering

  \subfloat[High-entropy soft-label]{%
    \includegraphics[width=0.48\columnwidth]{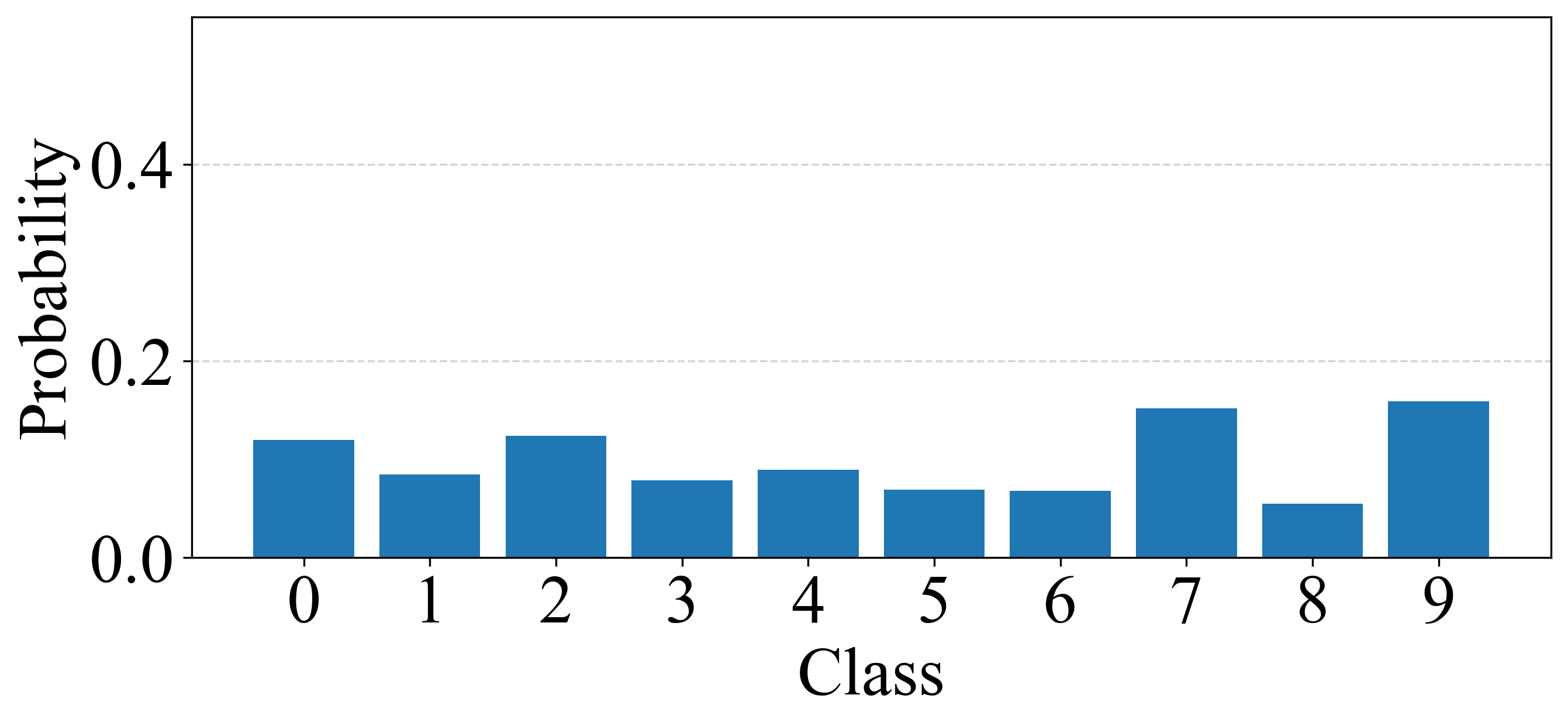}%
    \label{fig:comparison_era_enhanced_era:high-entropy}%
  }\hfill
  \subfloat[Low-entropy soft-label]{%
    \includegraphics[width=0.48\columnwidth]{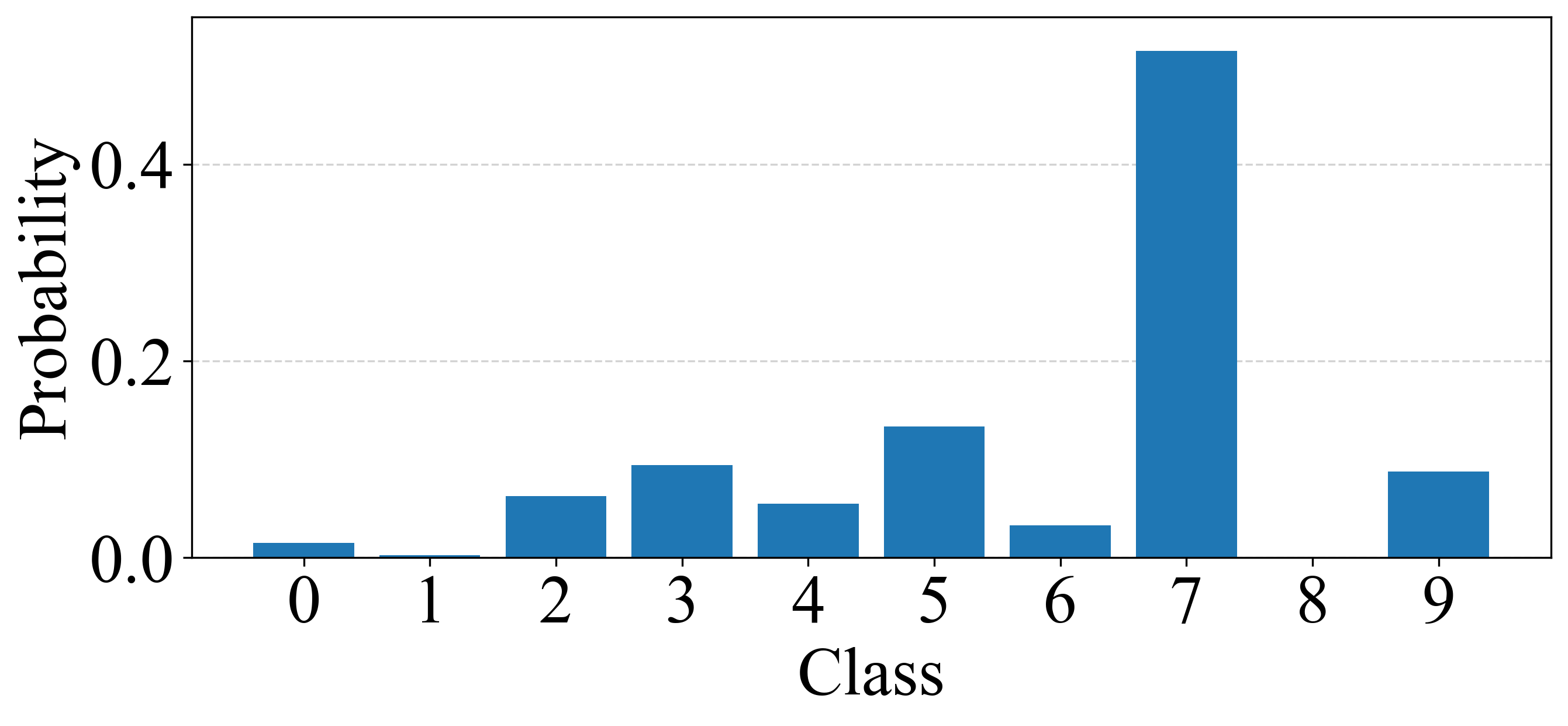}%
    \label{fig:comparison_era_enhanced_era:low-entropy}%
  }

  \subfloat[Entropy vs. $T$ using ERA for (\ref{fig:comparison_era_enhanced_era:high-entropy}) and (\ref{fig:comparison_era_enhanced_era:low-entropy}); shown left and right, respectively]{%
    \begin{minipage}[b]{\columnwidth}
      \centering
      \includegraphics[width=0.48\columnwidth]{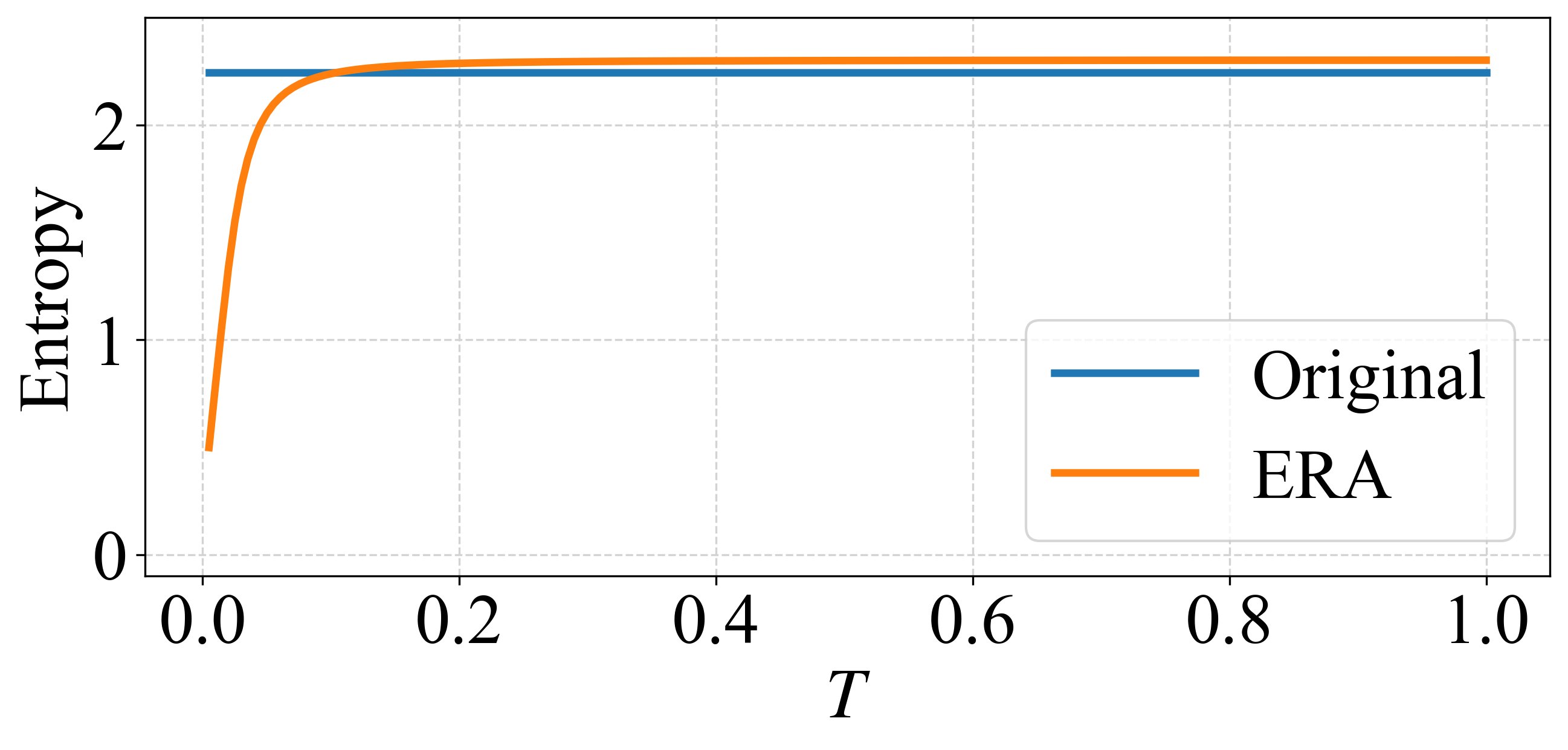}\hfill
      \includegraphics[width=0.48\columnwidth]{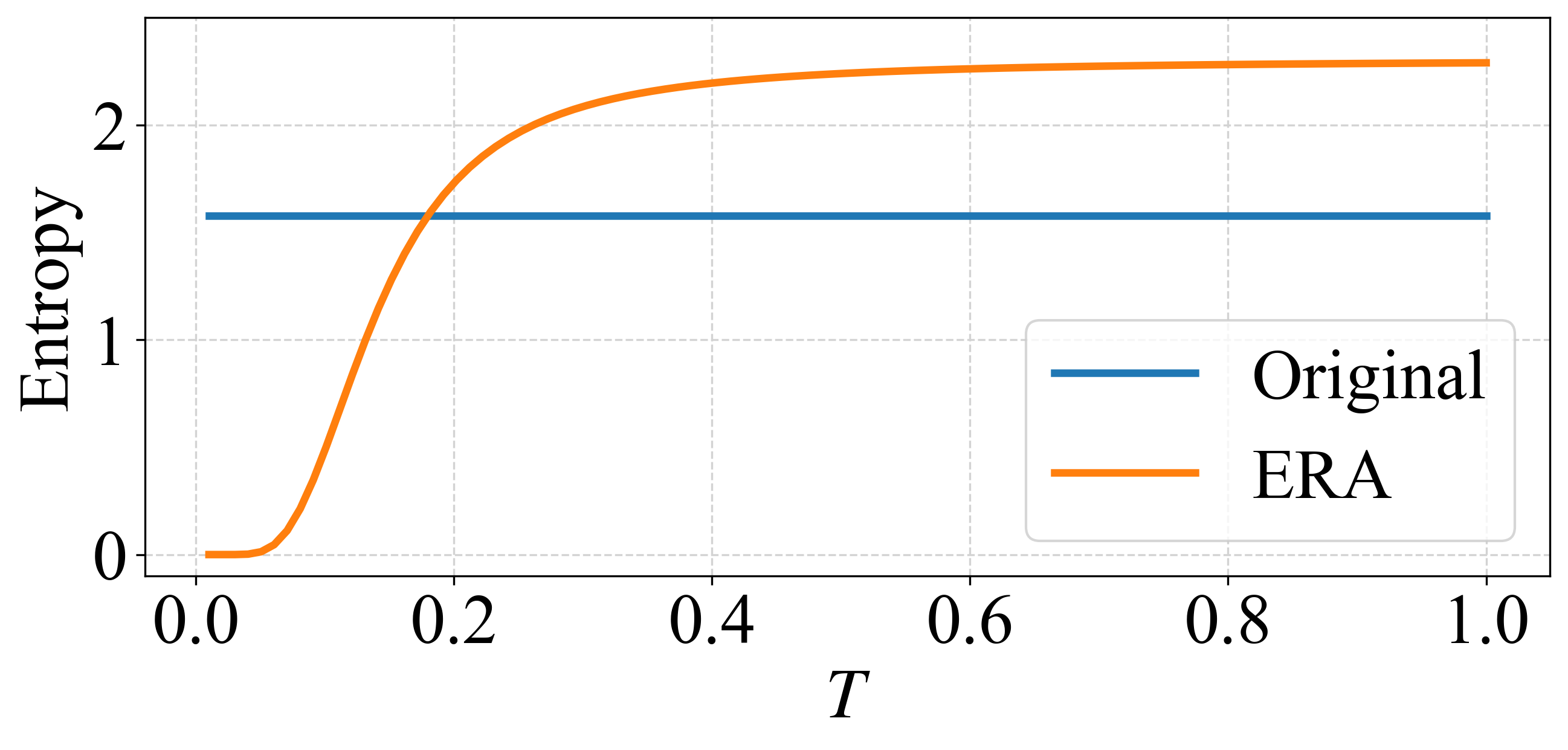}
      \label{fig:comparison_era_enhanced_era:era}
    \end{minipage}
  }

  \subfloat[Entropy vs.\ $\beta$ using Enhanced ERA for (\ref{fig:comparison_era_enhanced_era:high-entropy}) and (\ref{fig:comparison_era_enhanced_era:low-entropy}); shown left and right, respectively]{%
    \begin{minipage}[b]{0.96\columnwidth}
      \centering
      \includegraphics[width=0.48\columnwidth]{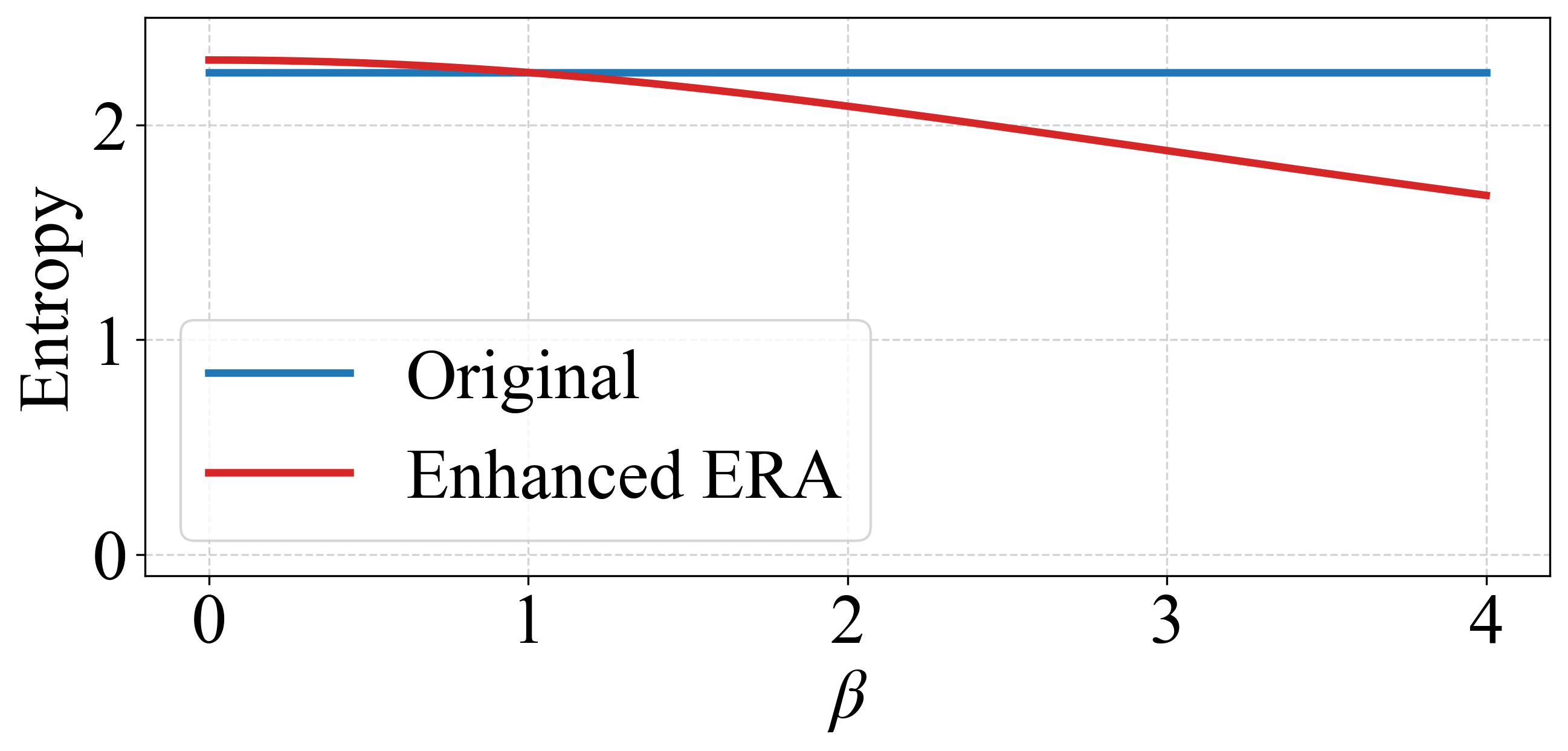}\hfill
      \includegraphics[width=0.48\columnwidth]{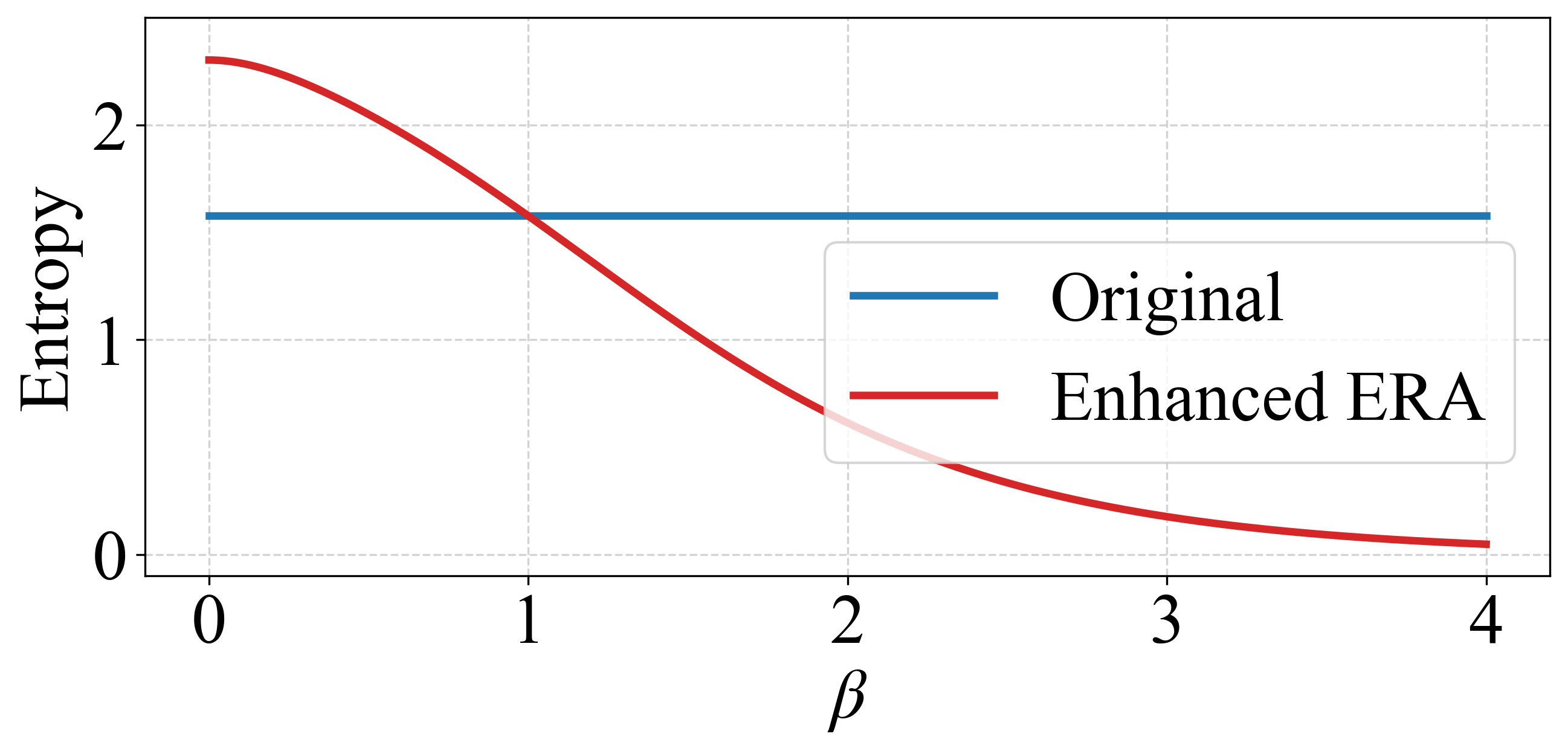}
    \end{minipage}
    \label{fig:comparison_era_enhanced_era:enhanced_era}
  }

  \caption{Comparison of ERA and Enhanced ERA with different parameter settings.
    (\ref{fig:comparison_era_enhanced_era:high-entropy}) and (\ref{fig:comparison_era_enhanced_era:low-entropy}) show the original high-entropy (low-confidence) and low-entropy (high-confidence) soft-labels, respectively. (\ref{fig:comparison_era_enhanced_era:era}) illustrates how ERA rapidly reduces entropy by decreasing temperature~$T$, particularly for (\ref{fig:comparison_era_enhanced_era:high-entropy}) compared to (\ref{fig:comparison_era_enhanced_era:low-entropy}).
    (\ref{fig:comparison_era_enhanced_era:enhanced_era}) demonstrates how Enhanced ERA provides smoother and more flexible entropy control by varying $\beta$, ensuring that $\beta = 1$ recovers the original entropy for both (\ref{fig:comparison_era_enhanced_era:high-entropy}) and (\ref{fig:comparison_era_enhanced_era:low-entropy}).
  } 
  \label{fig:comparison_era_enhanced_era}
\end{figure}

A primary challenge in distillation-based FL arises when aggregating soft-labels from clients with non-IID data. Simply averaging their heterogeneous predictions often results in an ambiguous global soft-label with high entropy, which provides a poor training signal and degrades performance.

To address this, \cite{dsfl} introduced Entropy Reduction Aggregation (ERA), defined in (Eq.~\ref{equation:era}), a method that intentionally sharpens the aggregated soft-labels. However, the original ERA must be carefully calibrated. The foundational concept of knowledge distillation \cite{kd} is to transfer generalization ability by using teacher's ``soft targets'' (referred to as soft-labels in this paper), which contain nuanced information in the relative probabilities of incorrect answers. Excessively reducing entropy turns these soft targets into near one-hot vectors, destroying this valuable information and increases the risk of overfitting.

This balance is particularly relevant in our realistic setting where the public and private datasets differ. For instance, when a model trained on CIFAR-10 \cite{cifar} sees a ``raccoon'' from CIFAR-100, its soft-label reflects a nuanced understanding of shared features (e.g., similarity to ``cat'' and ``dog''). Aggressively sharpening this into a confident ``cat'' prediction would incorrectly teach the model that a raccoon is a canonical cat, distorting its learned feature representation. Preserving some entropy is thus essential.

Moreover, in the early stages of training---when the model is not yet sufficiently trained and the soft-label tends to have high entropy, with all classes assigned relatively high probabilities---strong entropy reduction may remove the information associated with the correct class and instead amplify incorrect classes, thereby leading the model to learn from an incorrect label. Thus, for such high-entropy soft-labels, it is particularly important to avoid reducing the entropy too aggressively.

This practical challenge of finding a single, robust $T$ for mixed-entropy inputs is the core limitation of ERA. The averaged soft-labels exhibit widely varying entropy levels across samples (e.g., Fig.~\ref{fig:comparison_era_enhanced_era:high-entropy} versus Fig.~\ref{fig:comparison_era_enhanced_era:low-entropy}), and, as shown in Fig.~\ref{fig:comparison_era_enhanced_era:era}, ERA’s temperature $T$ is highly sensitive to this initial entropy.
Specifically, for high-entropy soft-labels (Fig.~\ref{fig:comparison_era_enhanced_era:high-entropy}), an extremely low $T$ is necessary to reduce ambiguity and extract a clear signal. However, applying such a low $T$ to low-entropy soft-labels (Fig.~\ref{fig:comparison_era_enhanced_era:low-entropy}) leads to aggressive over-sharpening. This excessive reduction destroys the nuanced information contained in the relative probabilities, effectively converting the soft-label into a one-hot vector. 
Conversely, setting a conservative $T$ to preserve low-entropy labels fails to sharpen high-entropy ones and can even increase their entropy.
Thus, finding a single value of $T$ that balances these conflicting requirements remains widely impractical.

To address this fundamental instability, we propose \textbf{Enhanced ERA}, defined as:
\begin{equation}
    \label{equation:enhanced_era}
    {\hat{z}^t}_i = \frac{\bigl(\overline{z^t}_{i}\bigr)^\beta}{\sum_{j=1}^{N} \bigl(\overline{z^t}_{j}\bigr)^\beta},
\end{equation}
where \(\overline{z^t} = \frac{1}{K} \sum_{k=1}^{K} z_{k}^{t}\) is the averaged local soft-labels, and \(\beta\) controls the aggregation sharpness. As $\beta$ increases beyond 1, the  distribution is monotonically sharpened. This is formally guaranteed by the principle of \textit{majorization}: we provide a proof that the output distribution for $\beta_2 > \beta_1$ \emph{majorizes} the distribution for $\beta_1$ in Appendix~\ref{appendix:majorization_proof}.

Beyond this core stability, Enhanced ERA directly solves the impasse of the conventional ERA. As illustrated in Fig.~\ref{fig:comparison_era_enhanced_era:enhanced_era}, the power-based formulation provides smooth, consistent control. Crucially, when $\beta=1$, it perfectly preserves the original entropy, regardless of whether the input is high-entropy (Fig.~\ref{fig:comparison_era_enhanced_era:high-entropy}) or low-entropy (Fig.~\ref{fig:comparison_era_enhanced_era:low-entropy}). This provides a stable ``identity'' baseline, allowing the server to preserve low-entropy labels (by setting $\beta \approx 1$) while still having a mechanism to gradually and predictably sharpen high-entropy labels (by setting $\beta > 1$).

This intuitive stability, observed in Fig.~\ref{fig:comparison_era_enhanced_era}, is underpinned by a formal mathematical difference between the two methods. The instability of ERA stems from its $1/T$ (inverse) relationship and its dependence on the absolute difference between probabilities. In contrast, the stability of Enhanced ERA stems from its $\beta$ (linear) relationship and its dependence on the ratio of probabilities. A formal comparative analysis, using log-probability ratios and sensitivity derivatives to prove this ``instability vs. stability'' dynamic, is provided in Appendix~\ref{appendix:analysis_sensitivity_stability}.

\begin{figure}[t]
  \centering

  \subfloat[Dirichlet $\alpha=0.05$]{%
    \includegraphics[width=0.48\columnwidth]{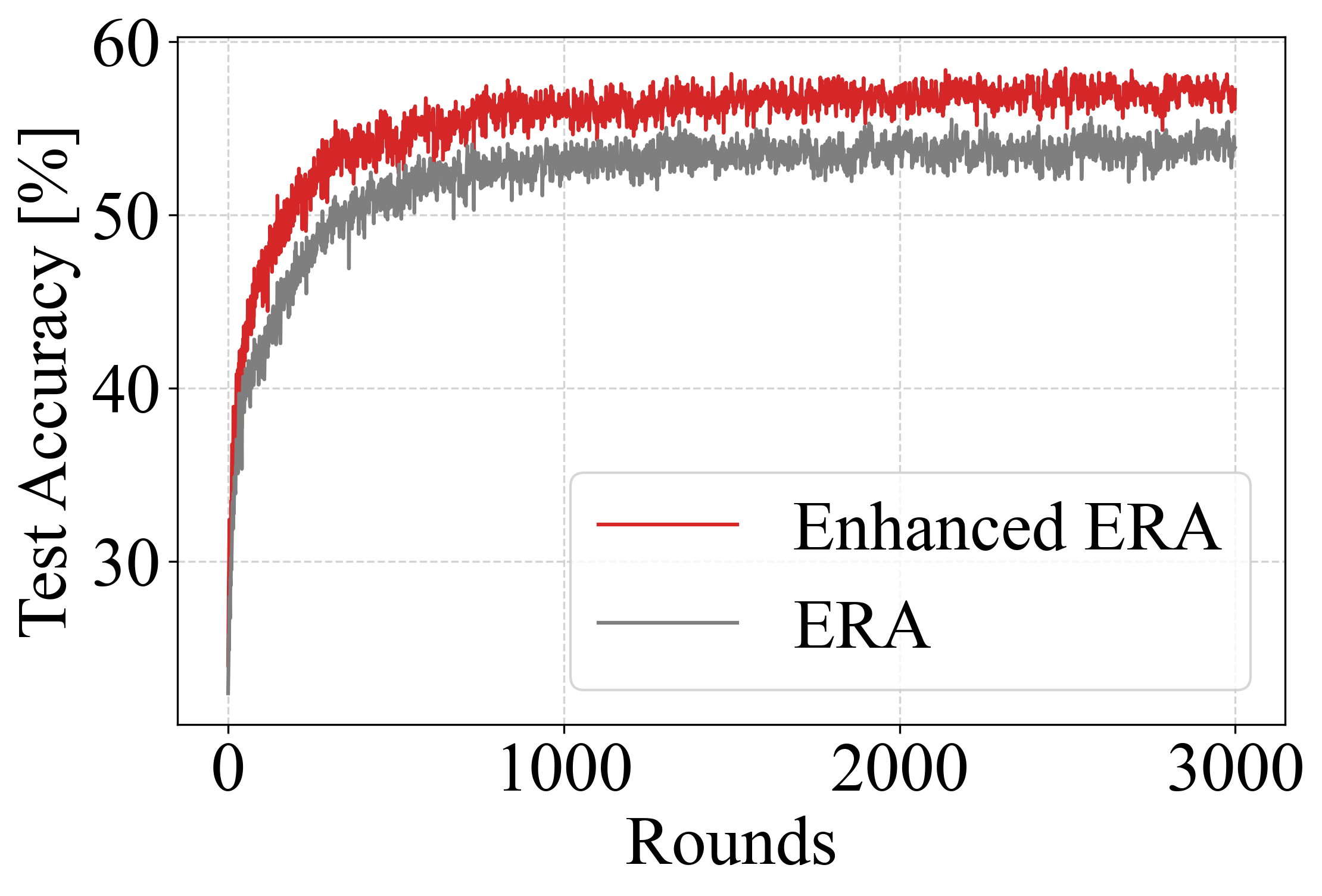}%
    \label{fig:era_vs_enhanced_era:0-05}%
  }\hfill
  \subfloat[Dirichlet $\alpha=0.3$]{%
    \includegraphics[width=0.48\columnwidth]{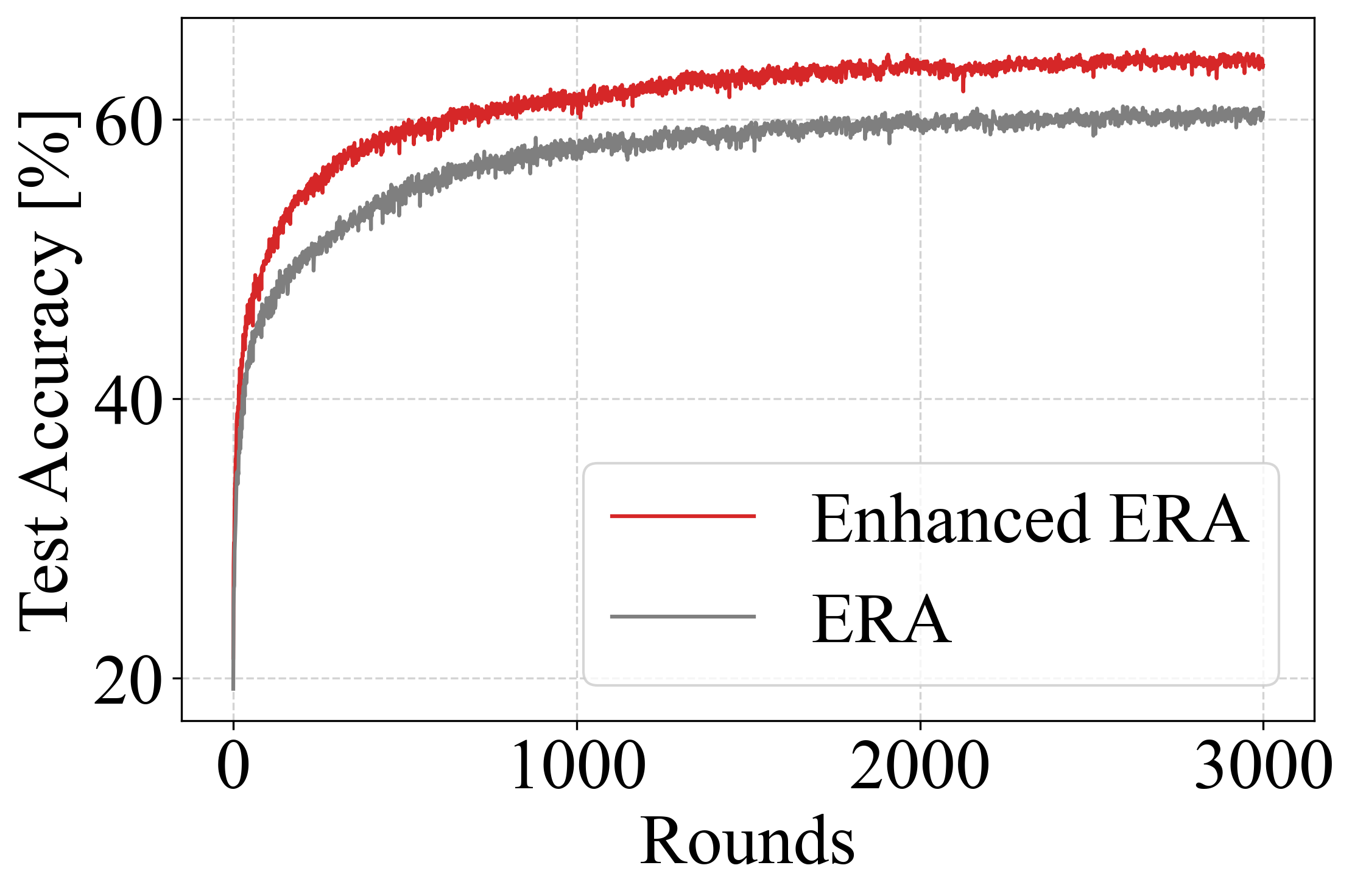}%
    \label{fig:era_vs_enhanced_era:0-3}%
  }

  \caption{Preliminary comparison of aggregation stability. Learning curves show Enhanced ERA (red lines) consistently outperforms the conventional ERA (gray lines) on CIFAR-10. The optimal parameters plotted are Enhanced ERA ($\beta=2.5$) vs. ERA ($T=0.1$) for strong non-IID (\ref{fig:era_vs_enhanced_era:0-05}), and  Enhanced ERA ($\beta=1.0$) vs. ERA ($T=0.2$) for moderate non-IID (\ref{fig:era_vs_enhanced_era:0-3}).}
  \label{fig:era_vs_enhanced_era}
\end{figure}

To validate the practical impact of this theoretical stability, we conducted a preliminary experiment (Fig.~\ref{fig:era_vs_enhanced_era}) comparing the two aggregation mechanisms in isolation. We evaluated the performance of the DS-FL \cite{dsfl} framework using its default ERA mechanism, and compared it directly against an identical framework where we replaced ERA with our proposed Enhanced ERA. Neither configuration used soft-label caching, and all other experimental settings were identical to the main setup described in Section \ref{subsection:experimental_setup}. We performed a fine-grained search for both mechanisms: ERA ($T \in \{0.025, 0.05, \dots, 0.275, 0.3\}$) and Enhanced ERA (tuning $\beta \in \{0.5, 1.0, \dots, 2.5, 3.0\}$). The optimal hyperparameter for each was selected based on the mean server-side test accuracy over the final 10 rounds of 3000 communication rounds.

The results in Fig.~\ref{fig:era_vs_enhanced_era} show Enhanced ERA significantly outperformed in both scenarios. It yielded an accuracy improvement of approximately 3 percentage points under strong non-IID conditions ($\alpha=0.05$) and approximately 4 percentage points under moderate non-IID conditions ($\alpha=0.3$). This consistent gap provides strong practical evidence for our theoretical analysis in Appendix~\ref{appendix:analysis_sensitivity_stability}: while a theoretical optimum $T$ may exist for ERA, its input-sensitive instability (Fig.~\ref{fig:comparison_era_enhanced_era:era}) makes it impractically difficult to tune. Enhanced ERA's robust, stable control (Fig.~\ref{fig:comparison_era_enhanced_era:enhanced_era}), in contrast, delivers superior performance with practical tuning.

Beyond this core stability, Enhanced ERA also provides a critical theoretical advantage regarding robustness to the number of classes $N$. The denominator of ERA (Eq.~\ref{equation:era}) scales with $N$ (i.e., $\sum e^{\overline{z}_{j}/T} \ge N$), forcing $T$ to be tuned to aggressively small values for datasets with large $N$ (e.g., CIFAR-100). In contrast, the denominator of Enhanced ERA (Eq.~\ref{equation:enhanced_era}) is bounded ($\sum_{j=1}^N (\overline{z}_j)^\beta$), making its normalization term independent of $N$ and $\beta$ a robust hyperparameter across datasets of varying class sizes.

In summary, Enhanced ERA provides controlled entropy adjustments, enabling SCARLET to effectively handle different degrees of non-IID data. This paper primarily focuses on introducing the fundamental mechanisms of soft-label caching and Enhanced ERA; the introduction of this stable, robust, and mathematically-grounded aggregation mechanism, in particular, is a key contribution. While this provides a controlled adjustment for non-IID data, the dynamic tuning of $\beta$ itself remains an open research question (discussed in Section~\ref{limitations}).

\section{Experiments}
\label{section:experiments}

\subsection{Experimental Setup}
\label{subsection:experimental_setup}

\subsubsection{Datasets and Partitioning}
\smallbreak
SCARLET was evaluated using three combinations of private and public datasets for image classification tasks under heterogeneous data conditions. The detailed dataset configurations are summarized in Table~\ref{table:datasets}.

\begin{table}[t]
 \caption{Dataset Configurations and Dirichlet Parameters Used in SCARLET Evaluation.}
 \label{table:datasets}
 \centering
  \begin{tabular}{ccccc}
   \toprule
    \multirow{2}{*}{Private Dataset} & \multirow{2}{*}{Public Dataset} & \multicolumn{2}{c}{Size} & \multirow{2}{*}{Dirichlet $\alpha$} \\
    & & Private & Public & \\
    \midrule 
    CIFAR-10 & CIFAR-100 & 50,000 & 10,000 & 0.05, 0.3 \\
    CIFAR-100 & Tiny ImageNet & 50,000 & 10,000 & 0.05 \\
    Tiny ImageNet & Caltech-256 & 100,000 & 10,000 & 0.05 \\
   \bottomrule
  \end{tabular}
\end{table}

To ensure the robustness of SCARLET in practical scenarios, we deliberately avoided methods that split a single dataset into public and private datasets, such as in \cite{dsfl}. Instead, we adopted the approach inspired by \cite{comet}, where separate datasets with no overlap in image content are used for public and private datasets. This choice reflects real-world conditions, as acquiring a public dataset with identical distributions to private data is often impractical. By utilizing distinct datasets, SCARLET aligns more closely with application contexts where such separations are necessary, enhancing the realism and relevance of the evaluation.

The private and public datasets used in our experiments are entirely distinct, with no direct class overlaps. Specifically, CIFAR-10 and CIFAR-100 \cite{cifar}, despite being from the same dataset family, have no overlapping classes. For instance, CIFAR-10 includes categories such as ``automobile'' and ``truck'' while CIFAR-100 contains related but separate categories like ``bus'' and ``pickup truck.'' Similarly, Tiny ImageNet \cite{tinyimagenet} and Caltech-256 \cite{caltech} cover broad object categories without any direct overlap in image content. While some semantic similarities exist between certain categories across datasets, they do not affect the integrity of the experimental setup.

Data heterogeneity across clients was simulated using a Dirichlet distribution \cite{dirichlet} with varying \(\alpha\) values. Larger $\alpha$ values approximate IID distributions, while smaller values result in each client holding data from fewer classes. Fig.~\ref{fig:partition} visualizes the non-IID sample distributions among clients. Smaller $\alpha$ values lead to skewed distributions, where each client primarily holds data from fewer classes. This setup evaluates SCARLET’s performance under varying degrees of data heterogeneity and better simulates challenging FL environments.

\begin{figure}[t]
  \subfloat[CIFAR-10 ($\alpha = 0.05$)]{%
    \includegraphics[width=0.5\columnwidth]{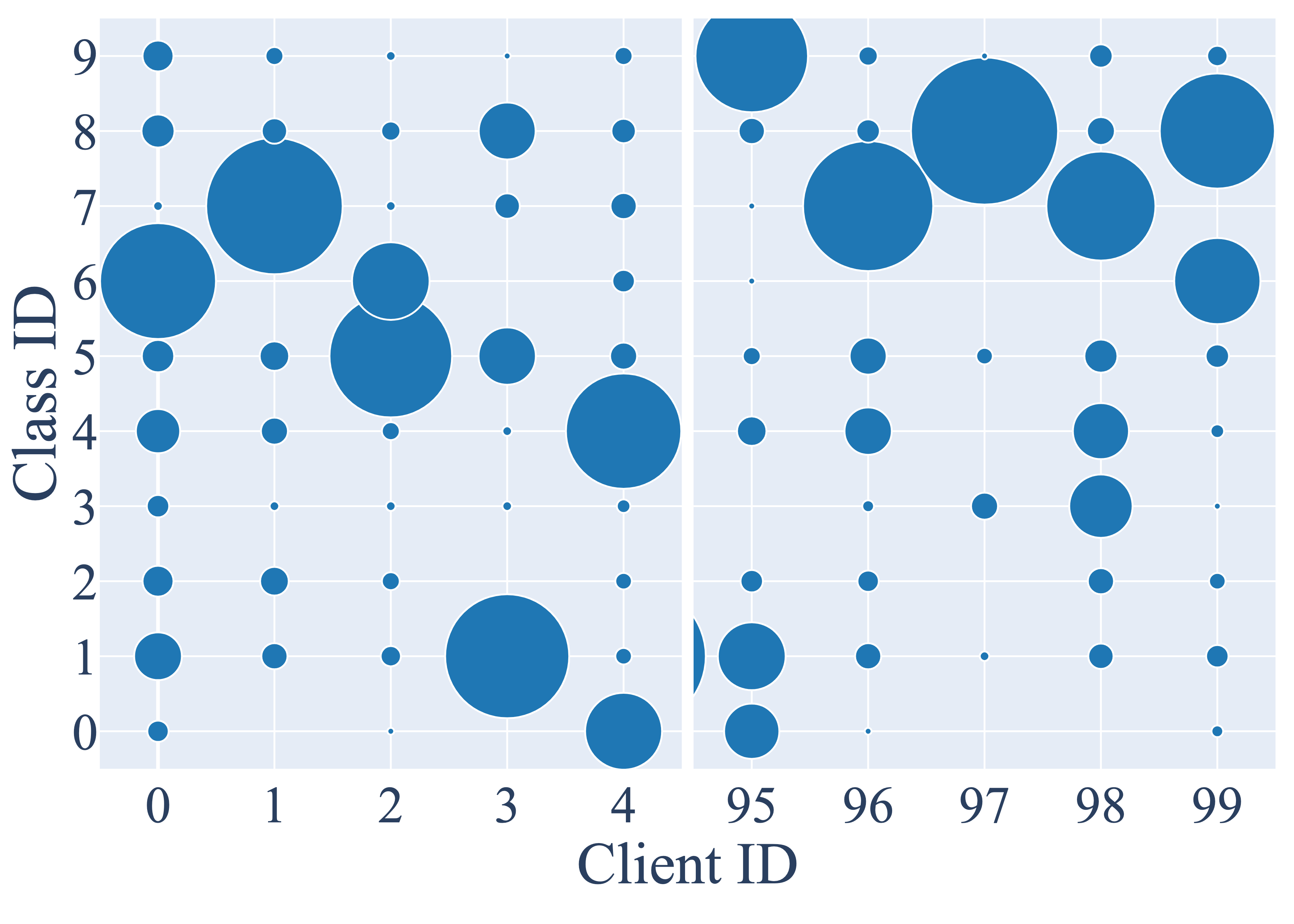}%
    \label{fig:partition:cifar10_a005}%
  }\hfill
  \subfloat[CIFAR-10 ($\alpha = 0.3$)]{%
    \includegraphics[width=0.5\columnwidth]{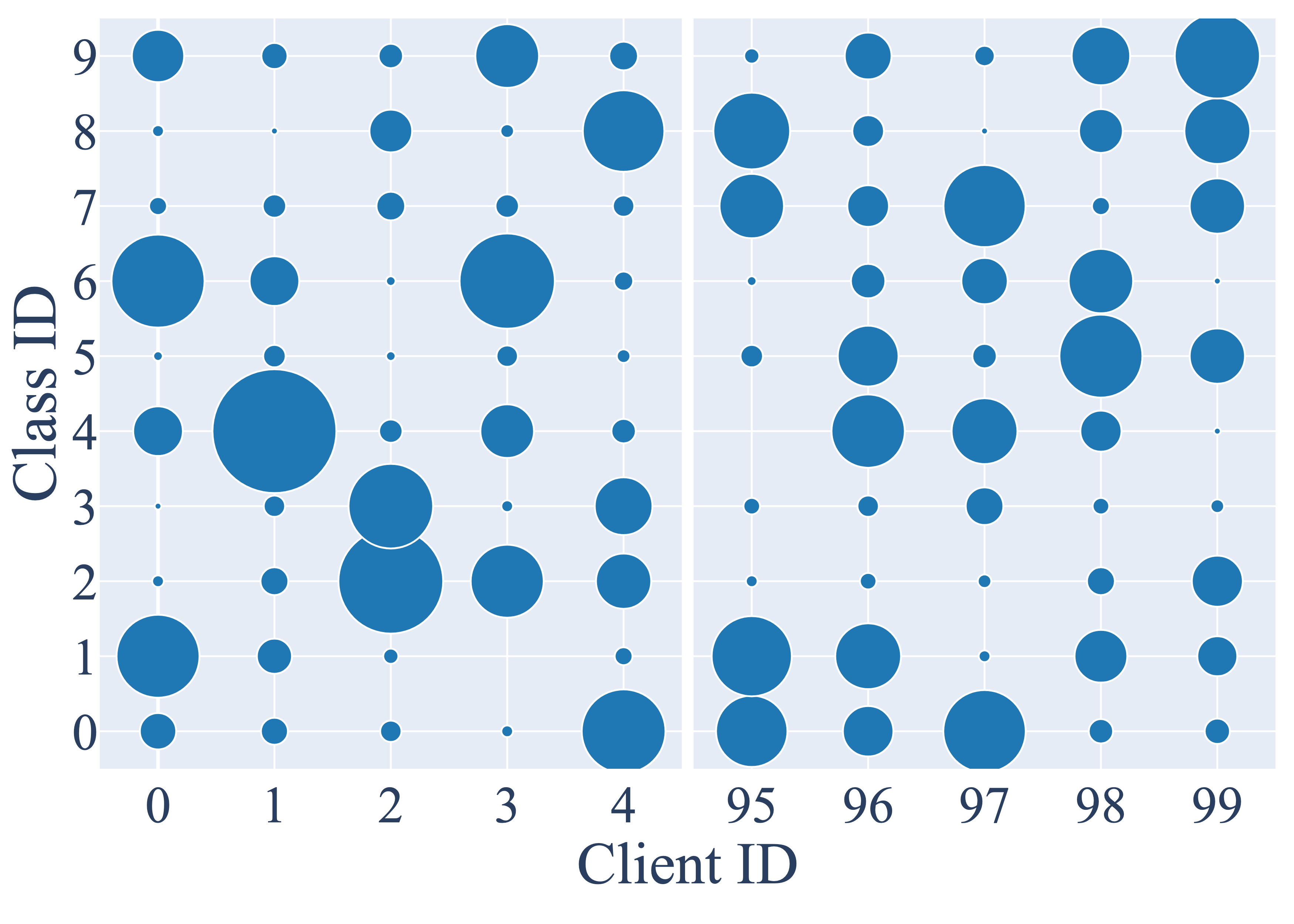}%
    \label{fig:partition:cifar10_a03}%
  }

  \subfloat[CIFAR-100 ($\alpha = 0.05$)]{%
    \includegraphics[width=0.5\columnwidth]{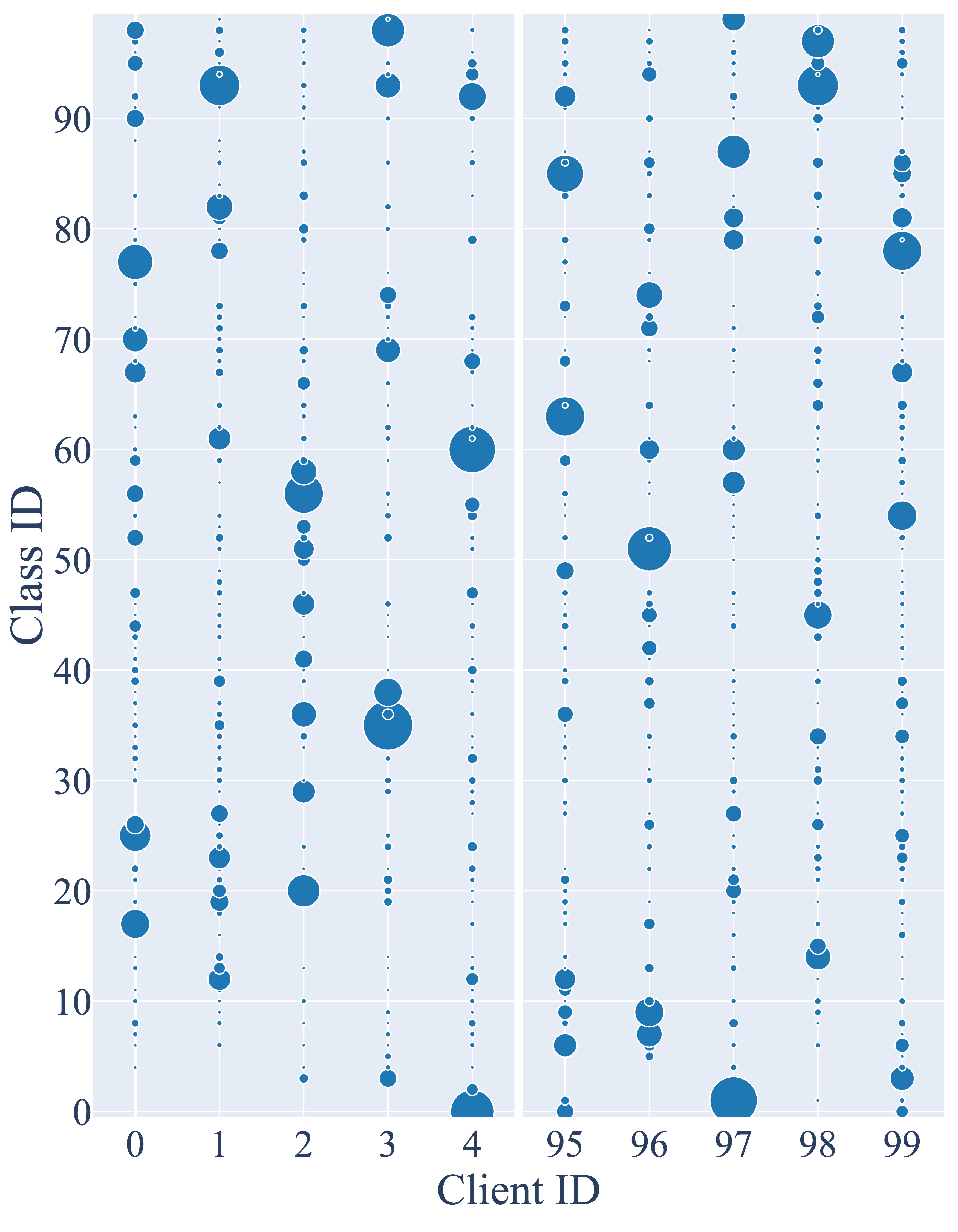}%
    \label{fig:partition:cifar100_a005}%
  }\hfill
  \subfloat[Tiny ImageNet ($\alpha = 0.05$)]{%
    \includegraphics[width=0.5\columnwidth]{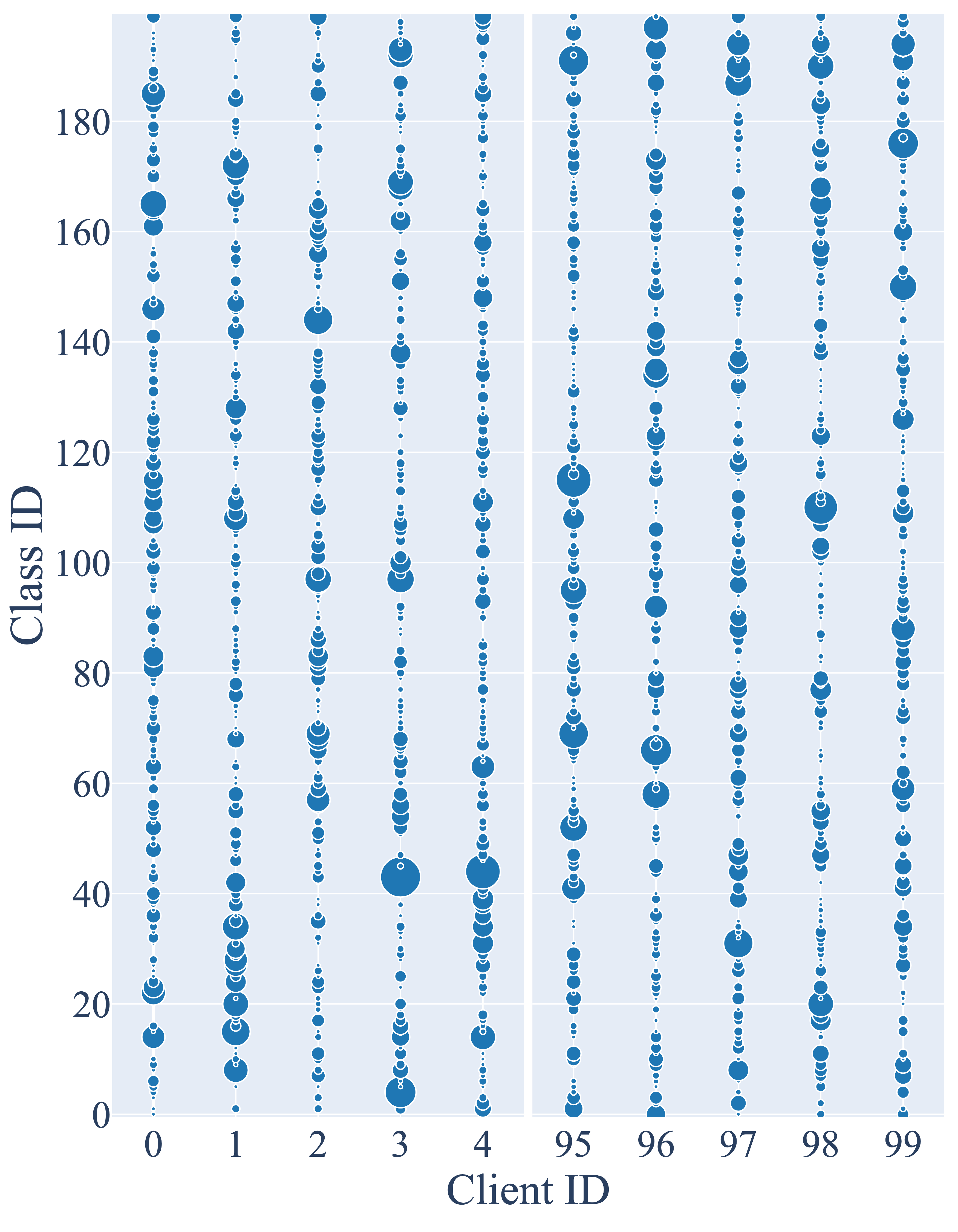}%
    \label{fig:partition:tiny_a005}%
  }

  \caption{Visualization of private dataset distributions across clients under different Dirichlet parameters ($\alpha$). Smaller $\alpha$ values result in more heterogeneous distributions, where clients predominantly hold data from fewer classes. (\ref{fig:partition:cifar10_a005}) and (\ref{fig:partition:cifar10_a03}) show CIFAR-10 with varying $\alpha$, while (\ref{fig:partition:cifar100_a005}) and (\ref{fig:partition:tiny_a005}) illustrate CIFAR-100 and Tiny ImageNet, respectively.}
  \label{fig:partition}
\end{figure}

\subsubsection{Models}
\smallbreak
Table~\ref{table:models} outlines the model architectures and input image sizes used for each private dataset.
For image classification tasks, we tailored ResNet~\cite{resnet} architectures to match each dataset’s characteristics.
CIFAR-10 and CIFAR-100 both use $32\times32$ images; however, the increased complexity of
CIFAR-100, with its 100 classes, necessitates a deeper ResNet-32 model.
For Tiny ImageNet, which has $64\times64$ images, we used a ResNet-18, which is well suited for intermediate-resolution inputs.

\begin{table}[t]
 \caption{Model Architectures and Input Image Sizes for Each Private Dataset.}
 \label{table:models}
 \centering
  \begin{tabular}{ccc}
   \toprule
    Private Dataset & Model & Input Image Size \\
    \midrule
    CIFAR-10 & ResNet-20 & $32\times32$ \\
    CIFAR-100 & ResNet-32 & $32\times32$ \\
    Tiny ImageNet & ResNet-18 & $64\times64$ \\
   \bottomrule
  \end{tabular}
\end{table}

Although SCARLET and other distillation-based FL methods do not require the server and clients to use the same model architecture, we employed a uniform setup across all entities for simplicity. This consistency helps isolate the effects of SCARLET’s communication cost reduction from other factors related to model complexity.

\subsubsection{Comparison Methods and Hyperparameter Settings}
\smallbreak
We evaluated SCARLET's communication efficiency and model accuracy against four state-of-the-art methods: DS-FL \cite{dsfl}, CFD \cite{cfd}, COMET \cite{comet}, and Selective-FD \cite{selectivefd}.

Unless otherwise mentioned, the following hyperparameter settings were common across all methods. The number of clients was set to 100, each performing 5 local epochs per training round. The learning rate for both local client training and client-server distillation was set to 0.1, following \cite{dsfl}. Each training round utilized 1,000 unlabeled public data samples. For simplicity, we employed stochastic gradient descent (SGD) without momentum or weight decay as the optimizer.

The total number of communication rounds was determined to ensure our proposed method, SCARLET, and most of the baselines reached a stable performance plateau where its accuracy saturated. While this experimental setup uses a global test set for analysis, we validate in Appendix~\ref{appendix:validation} that this convergence plateau is observable in practice using accessible proxy metrics (e.g., local validation loss) that do not require such a test set. Accordingly, unless otherwise mentioned, we used 3000 rounds for CIFAR-10 and CIFAR-100, and 1000 rounds for Tiny ImageNet. This duration was deemed sufficient to provide a fair comparison of the convergence speed and accuracy against baseline methods within this operational window.

In contrast, method-specific hyperparameters critical to performance, such as the ERA temperature $T$ in DS-FL, were individually optimized through hyperparameter tuning. Detailed hyperparameter configurations for each comparative method are provided in Appendix~\ref{appendix:hyperparameter}.

\begin{figure}[t]
  \centering

  \subfloat[Private]{%
    \includegraphics[width=0.48\columnwidth]{figures/cifar10_alpha0-3_plotly_try.png}%
    \label{fig:test_partition:private}%
  }\hfill
  \subfloat[Test]{%
    \includegraphics[width=0.48\columnwidth]{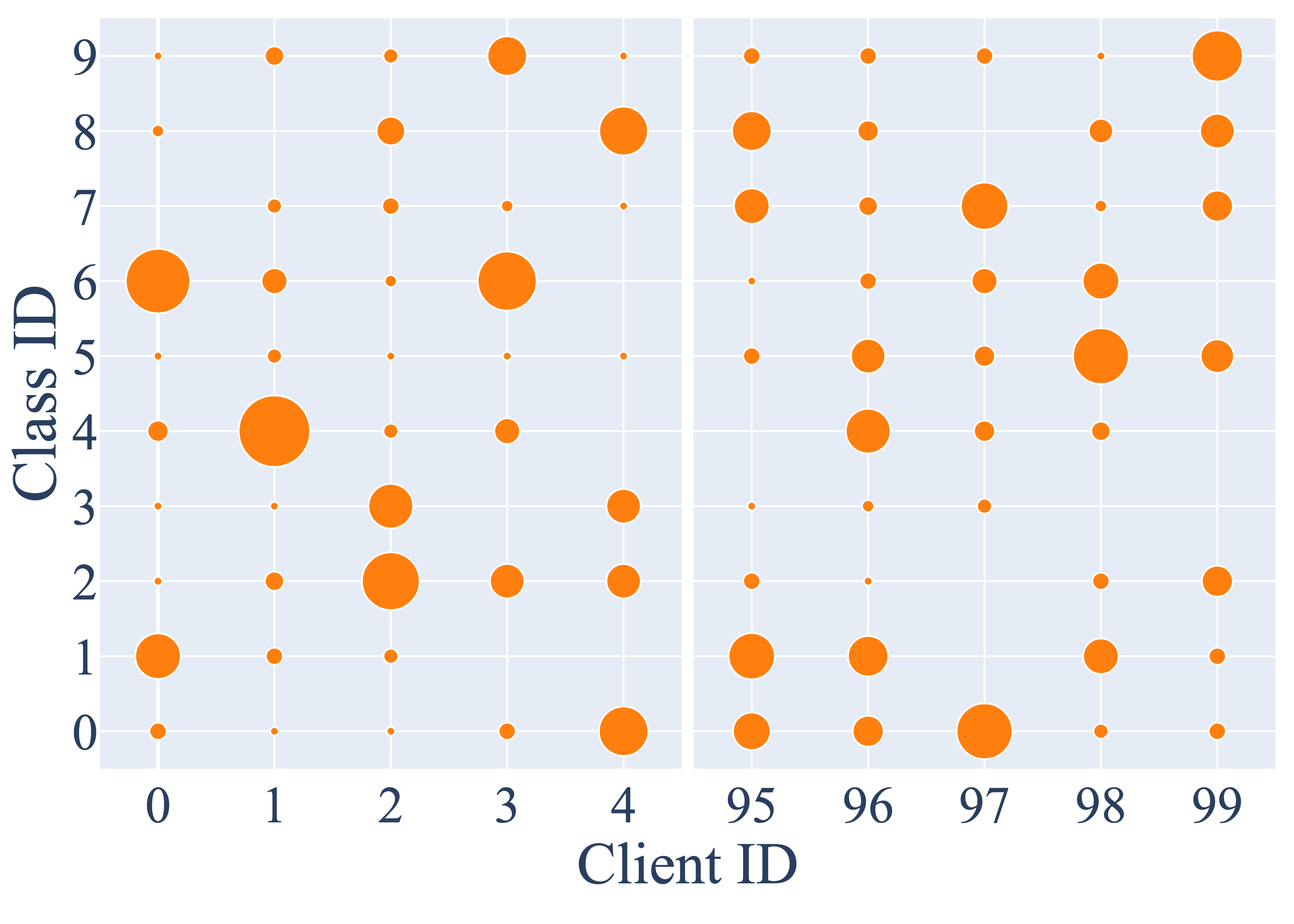}%
    \label{fig:test_partition:test}%
  }

  \caption{Comparison of each client's private and test dataset distributions.
  Both datasets are generated using the same Dirichlet distribution,
  ensuring that each client's test set maintains class proportions similar to its private dataset. This configuration supports personalized and non-IID evaluation on the client side.}
  \label{fig:test_partition}
\end{figure}

\subsubsection{Evaluation Metrics}
\smallbreak

The evaluation of SCARLET is based on two primary metrics: cumulative communication cost and test accuracy. The cumulative communication cost is defined as the total amount of data exchanged between the server and clients, including soft-labels and other parameters, summed over all clients across communication rounds. We exclude the one-time cost for retrieving the public dataset from this calculation, as it is common to all compared methods and does not affect relative performance.
Furthermore, to provide a deeper analysis that reflects real-world network conditions with often asymmetric bandwidth, we also present a detailed breakdown of the uplink (client-to-server) and downlink (server-to-client) communication costs for our main comparison.

Test accuracy is measured separately for the server and clients using distinct datasets. For example, when using the train split of CIFAR-10 (50,000 images) as the private dataset, the whole test split (10,000 images) is used for server-side evaluation. The same test split is further partitioned using the same Dirichlet distribution employed for the private dataset, distributing 100 images to each client and enabling client-side evaluation under non-IID conditions (see Fig.~\ref{fig:test_partition}). This setup ensures that server-side accuracy reflects global performance, while client-side accuracy evaluates personalized performance based on each client’s unique data distribution. Notably, client-side accuracy often surpasses server-side accuracy due to the alignment of client-specific test datasets with their training data distributions.

The main analysis presents the relationship between cumulative communication cost and test accuracy through graphs, highlighting both the efficiency and stability of the training process. Additionally, we provide cumulative communication cost values to reach a predefined target accuracy, but emphasize graphical results to better capture the dynamics of learning stability and convergence across methods.

\subsection{Experimental Results}
\label{subsection:experimental_results}

\begin{table}[t]
 \caption{Performance of Centralized Learning on Three Image Datasets.}
 \label{table:nonfl}
 \centering
  \begin{tabular}{ccc}
   \toprule
    Dataset & Model & Test Accuracy \\
    \midrule 
    CIFAR-10 & ResNet-20 & 85.4\% \\
    CIFAR-100 & ResNet-32 & 55.5\% \\
    Tiny ImageNet & ResNet-18 & 39.5\% \\
   \bottomrule
  \end{tabular}
\end{table}
 
\subsubsection{Centralized Learning (Non-FL) Reference Results}
\label{subsubsection:centralized_learning_reference_results}
\smallbreak
Table~\ref{table:nonfl} shows the performance obtained by centralized learning, where all clients upload their labeled private data to a central server, and a single global model is trained. These results serve as a reference to illustrate achievable performance levels without the constraints of FL. In this setting, since all data is aggregated at a central location, the resulting training data is effectively IID. This often leads to significantly higher accuracy compared to FL conducted under non-IID data conditions.

Note that, unlike the FL experiments which typically employ a higher learning rate, a learning rate of 0.01 is adopted here to ensure stable convergence in the centralized training setting.

\begin{figure*}[t]
  \centering

  \subfloat[Server ($\alpha = 0.05$)]{%
    \includegraphics[width=0.48\textwidth]{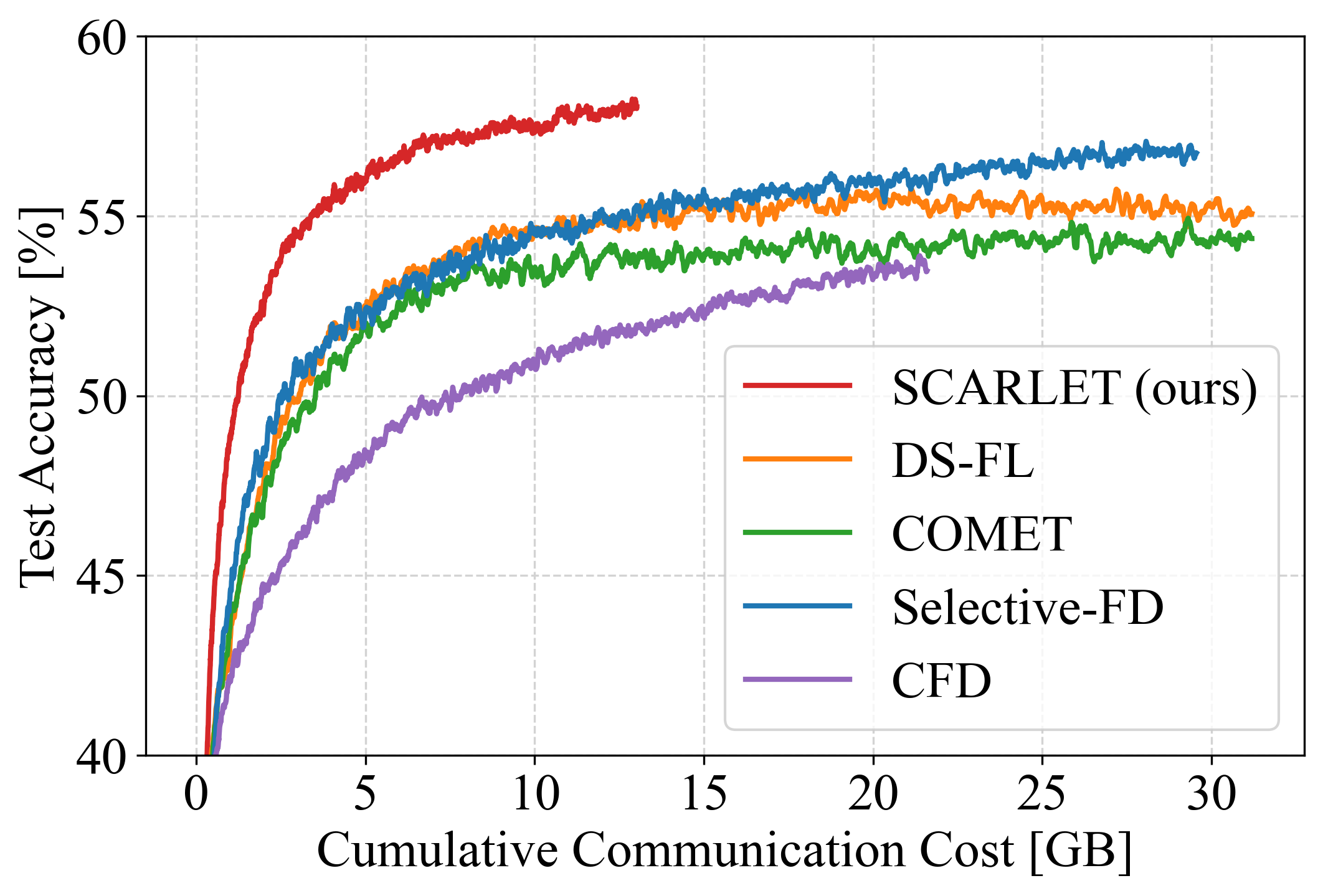}%
    \label{fig:performance_cifar10:server_a005}%
  }\hfill
  \subfloat[Client ($\alpha = 0.05$)]{%
    \includegraphics[width=0.48\textwidth]{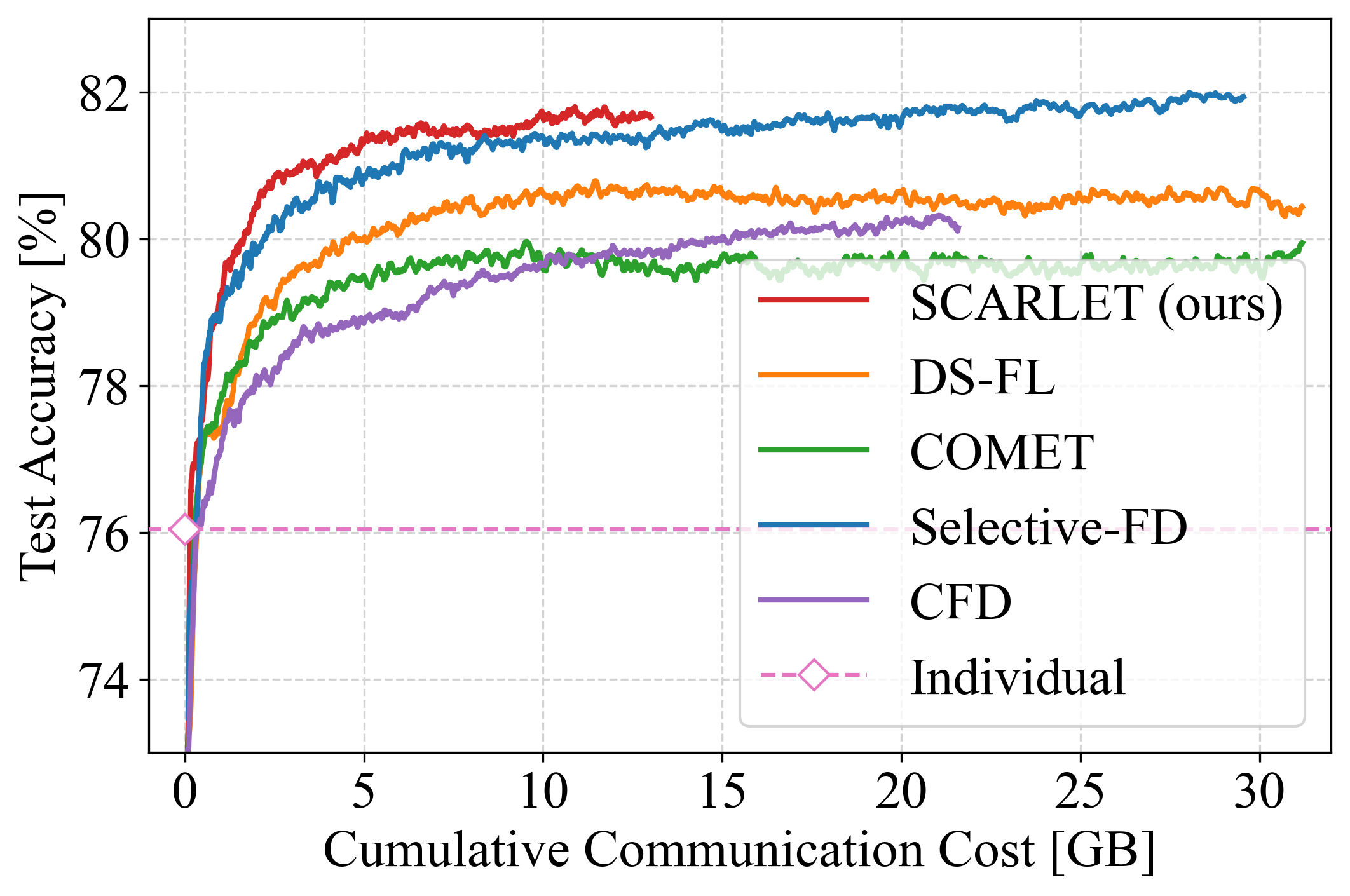}%
    \label{fig:performance_cifar10:client_a005}%
  }

  \subfloat[Server ($\alpha = 0.3$)]{%
    \includegraphics[width=0.48\textwidth]{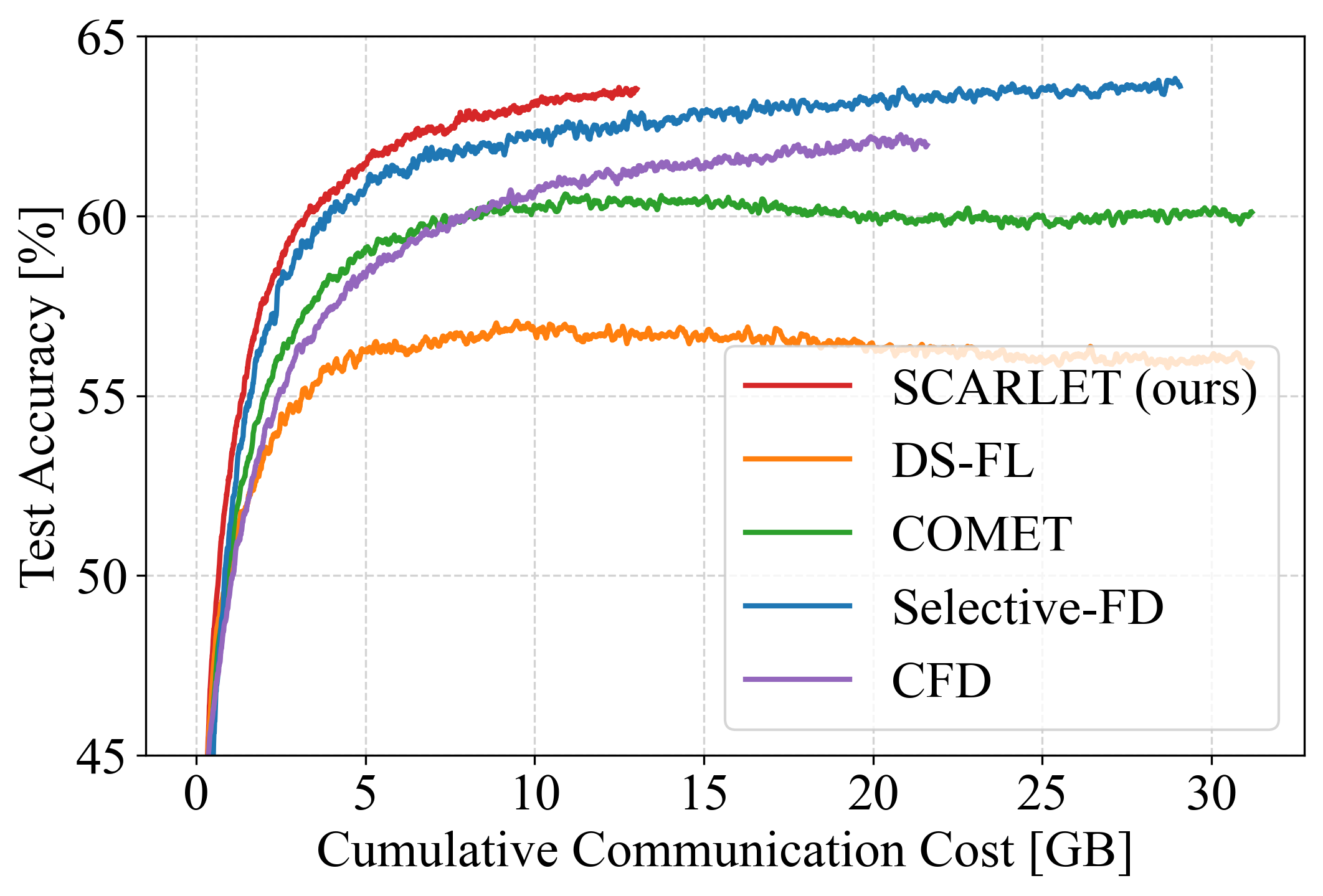}%
    \label{fig:performance_cifar10:server_a03}%
  }\hfill
  \subfloat[Client ($\alpha = 0.3$)]{%
    \includegraphics[width=0.48\textwidth]{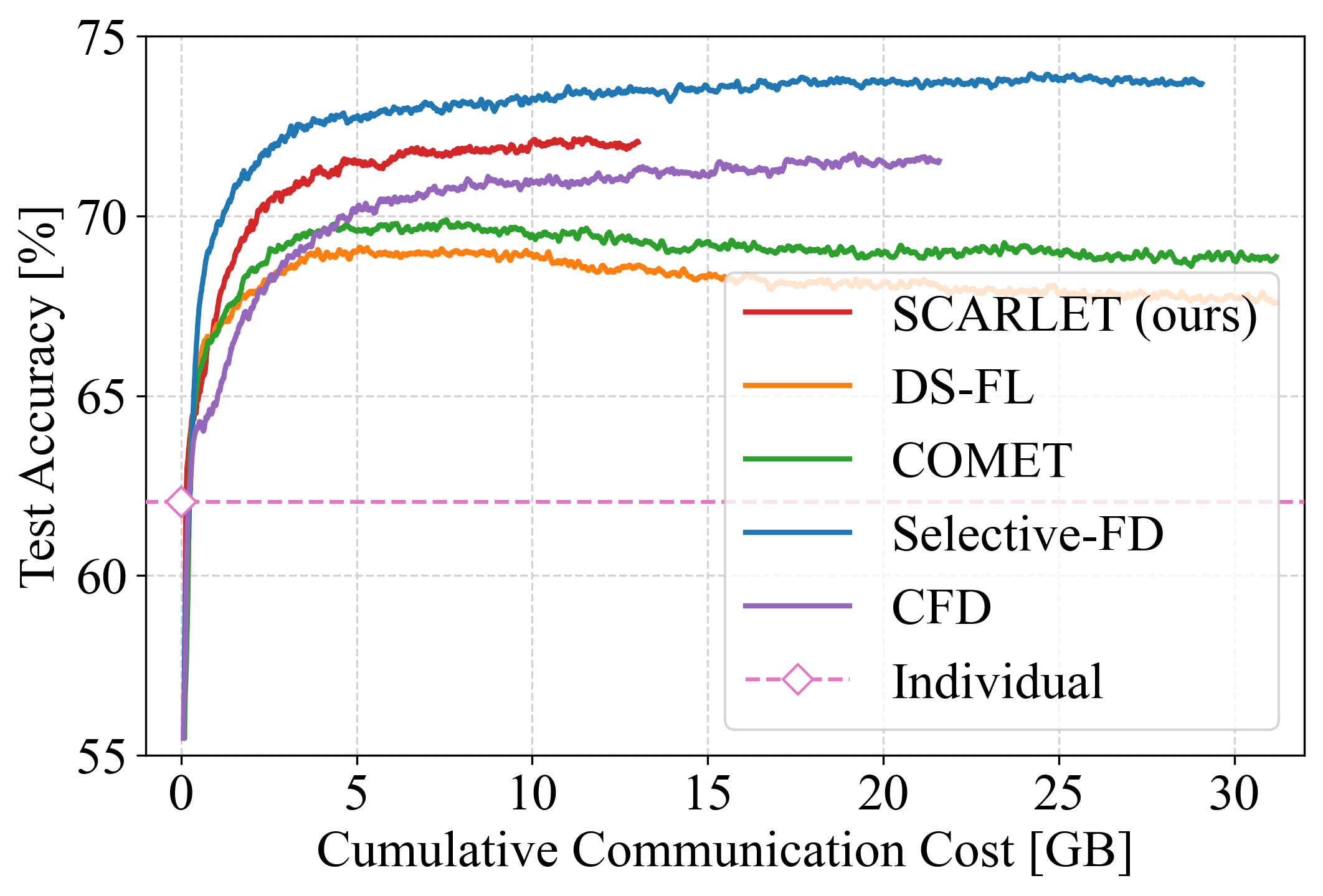}%
    \label{fig:performance_cifar10:client_a03}%
  }

  \caption{Test accuracy vs. communication cost for SCARLET and other state-of-the-art methods on CIFAR-10. SCARLET shows the largest advantage under strong non-IID conditions (\ref{fig:performance_cifar10:server_a005}, \ref{fig:performance_cifar10:client_a005}). This improvement results from the synergy of our two main contributions:
  the soft-label caching mechanism reduces data transmission, while the Enhanced ERA mechanism effectively mitigates the negative impact of data heterogeneity.
  Under moderate non-IID conditions (\ref{fig:performance_cifar10:server_a03}, \ref{fig:performance_cifar10:client_a03}), SCARLET maintains competitive accuracy with similar communication efficiency.}
  \label{fig:performance_cifar10}
\end{figure*}

\smallbreak
\subsubsection{Comparative Evaluation}
\label{subsubsection:comparative_evaluaion}
\smallbreak
We evaluate SCARLET’s communication efficiency and model accuracy by comparing it against four state-of-the-art methods—DS-FL \cite{dsfl}, CFD \cite{cfd}, COMET \cite{comet}, and Selective-FD \cite{selectivefd}—as well as an \textit{Individual} baseline that represents isolated client training without collaboration. All results were averaged across three random seeds, with details of hyperparameter settings for comparative methods provided in Appendix~\ref{appendix:hyperparameter}.

In all experiments, SCARLET employed a soft-label caching mechanism with a cache duration of $D=50$ and utilized the Enhanced ERA strategy to adapt to non-IID data distributions. The sharpness parameter $\beta$ in Enhanced ERA was tuned to reflect the degree of data heterogeneity. For CIFAR-10, $\beta$ was set to $1.5$ under strongly non-IID conditions ($\alpha=0.05$) and $1.0$ under moderate non-IID conditions ($\alpha=0.3$). For CIFAR-100 and Tiny ImageNet, $\beta=1.25$ was used. These parameters were chosen as they represent a strong balance of performance and efficiency, as demonstrated in our detailed ablation studies on cache duration $D$ (Section~\ref{subsubsection:ablation_study_of_cache_duration}) and aggregation sharpness $\beta$ (Section~\ref{subsubsection:ablation_study_of_enhanced_era}).

Fig.~\ref{fig:performance_cifar10} presents the relationship between server-side and client-side test accuracy and cumulative communication cost on CIFAR-10 under strong and moderate non-IID settings. 
SCARLET outperformed other methods under strong non-IID conditions ($\alpha=0.05$), achieving approximately 2 percentage points higher server-side accuracy while reducing communication costs by over 50\%. On the client side, SCARLET's accuracy closely followed the peak performance of Selective-FD (remaining within 1 percentage point of it), significantly outperforming all other baselines. This highlights the effectiveness of SCARLET’s caching mechanism and adaptive aggregation in mitigating challenges posed by heterogeneous data distributions.

\begin{table}[t]
 \caption{Per-Round Uplink and Downlink Communication Costs.}
\label{table:comm_costs}
 \centering
 \setlength{\tabcolsep}{4pt}
  \begin{tabular}{ccc}
   \toprule
   Method & \makecell{Uplink Cost \\ {[MB/round]}} & \makecell{Downlink Cost \\ {[MB/round]}} \\
   \midrule 
   \textbf{SCARLET (ours)} & \textbf{1.37 $\pm$ 0.28 (max: 4.30)} & \textbf{2.97 $\pm$ 0.29 (max: 6.32)} \\
   DS-FL \cite{dsfl} & 4.80 $\pm$ 0.00 (max: 4.80) & 5.60 $\pm$ 0.00 (max: 5.60) \\
   COMET \cite{comet} & 4.80 $\pm$ 0.00 (max: 4.80) & 5.60 $\pm$ 0.00 (max: 5.60) \\
   CFD \cite{cfd} & 1.60 $\pm$ 0.00 (max: 1.60) & 5.60 $\pm$ 0.00 (max: 5.60) \\
   Selective-FD \cite{selectivefd} & 3.88 $\pm$ 0.04 (max: 4.02) & 5.52 $\pm$ 0.02 (max: 5.58) \\
   \bottomrule
  \end{tabular}
\end{table}

\begin{figure}[t]
    \centering
    \includegraphics[width=0.95\columnwidth]{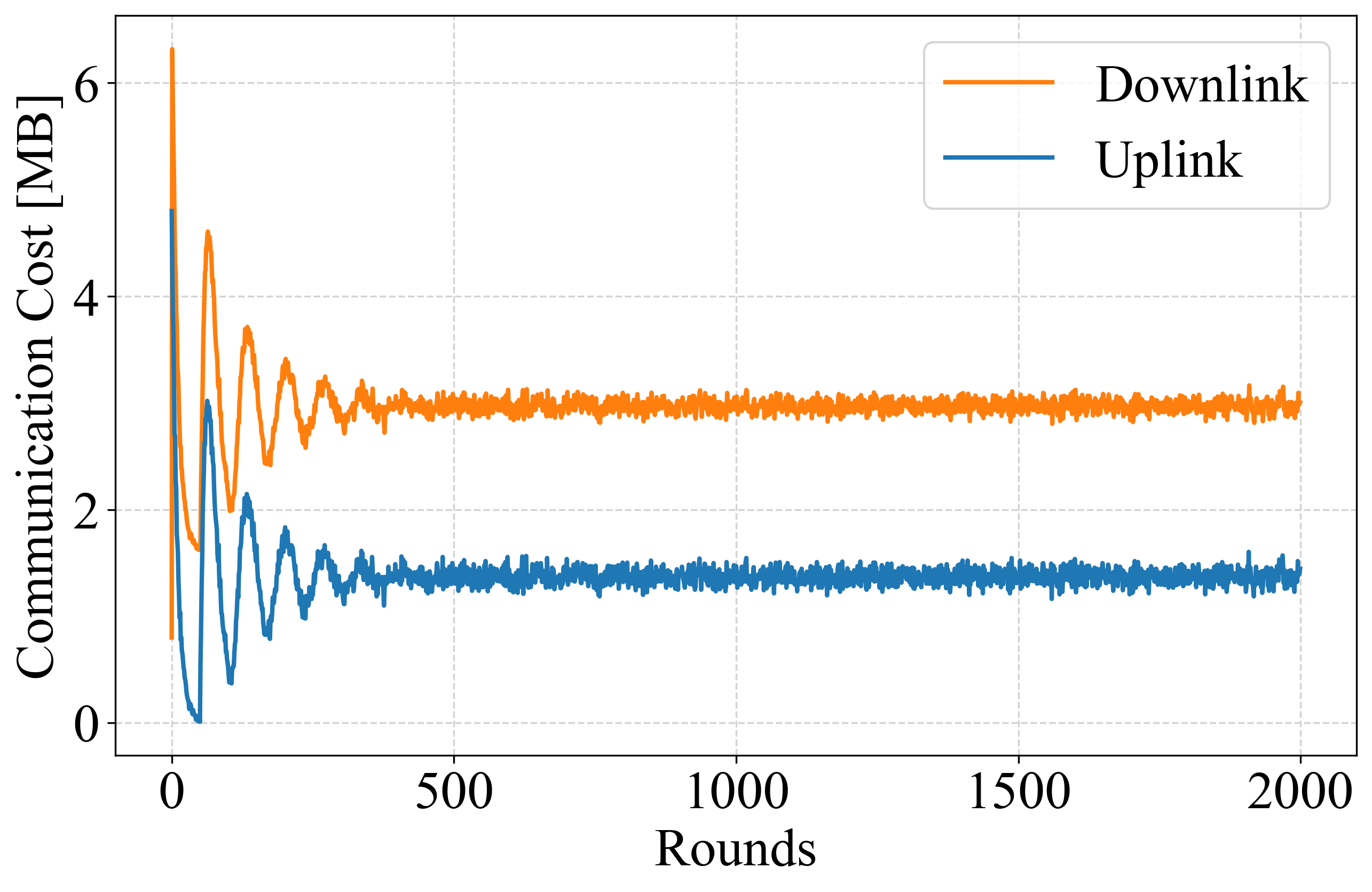}
    \caption{Per-round uplink and downlink communication costs for SCARLET. This figure visualizes the dynamic costs summarized as averages in Table~\ref{table:comm_costs}. Unlike methods with static costs, SCARLET's costs fluctuate dynamically, especially in the initial rounds, due to the soft-label caching mechanism. The costs stabilize as the random selection of public data staggers the timing of cache updates across individual samples, consistently maintaining a lower burden on the uplink.}
    \label{fig:scarlet_cost_dynamics}
\end{figure}

To offer a more granular analysis of this communication efficiency for the CIFAR-10 dataset under a strong non-IID setting ($\alpha=0.05$), Table~\ref{table:comm_costs} details the separated uplink and downlink costs, while Fig.~\ref{fig:scarlet_cost_dynamics} visualizes the per-round dynamics that produce these averages for SCARLET. This breakdown is critical for practical deployments where uplink bandwidth is often a bottleneck. As shown, SCARLET's advantage is particularly pronounced in the uplink, where the soft-label caching mechanism reduces the average cost by approximately 71\% compared to DS-FL and COMET, as clients only transmit uncached data. This targeted reduction of the more critical uplink burden, while still maintaining the lowest downlink cost, underscores SCARLET's robustness for real-world federated networks with asymmetric bandwidth.

Under moderate non-IID conditions ($\alpha=0.3$), SCARLET maintained server-side accuracy comparable to the best-performing method, Selective-FD, while achieving over 50\% lower communication cost. However, Selective-FD demonstrated superior client-side convergence speed and final accuracy, underscoring its strength in personalized learning scenarios. In contrast, DS-FL struggled to achieve high accuracy under these conditions, regardless of how the temperature $T$ in its ERA mechanism was tuned. This limitation can be traced to the presence of aggregated soft-labels with inherently low entropy. When $T$ is reduced excessively, these low-entropy soft-labels approach a one-hot representation, diminishing the effectiveness of the distillation process. Conversely, increasing $T$ too much causes a sudden spike in entropy for certain soft-labels, exacerbating inconsistencies in entropy reduction across public datasets. These issues result in both lower accuracy and performance degradation over time, likely due to overfitting.

SCARLET’s Enhanced ERA mechanism, on the other hand, addresses these challenges by enabling consistent and controlled entropy reduction. By dynamically adjusting the sharpness parameter $\beta$, SCARLET prevents the pitfalls of overly sharp or overly smooth distributions. For less severe non-IID conditions, $\beta$ is set to 1.0, effectively simplifying aggregation to straightforward averaging, which is sufficient under such scenarios. This adaptability allows SCARLET to achieve stable and high accuracy, while preserving its hallmark communication efficiency and maintaining compatibility across a wide range of FL scenarios.

\begin{figure}[t]
  \centering

  \subfloat[CIFAR-100 (Server)]{%
    \includegraphics[width=0.48\columnwidth]{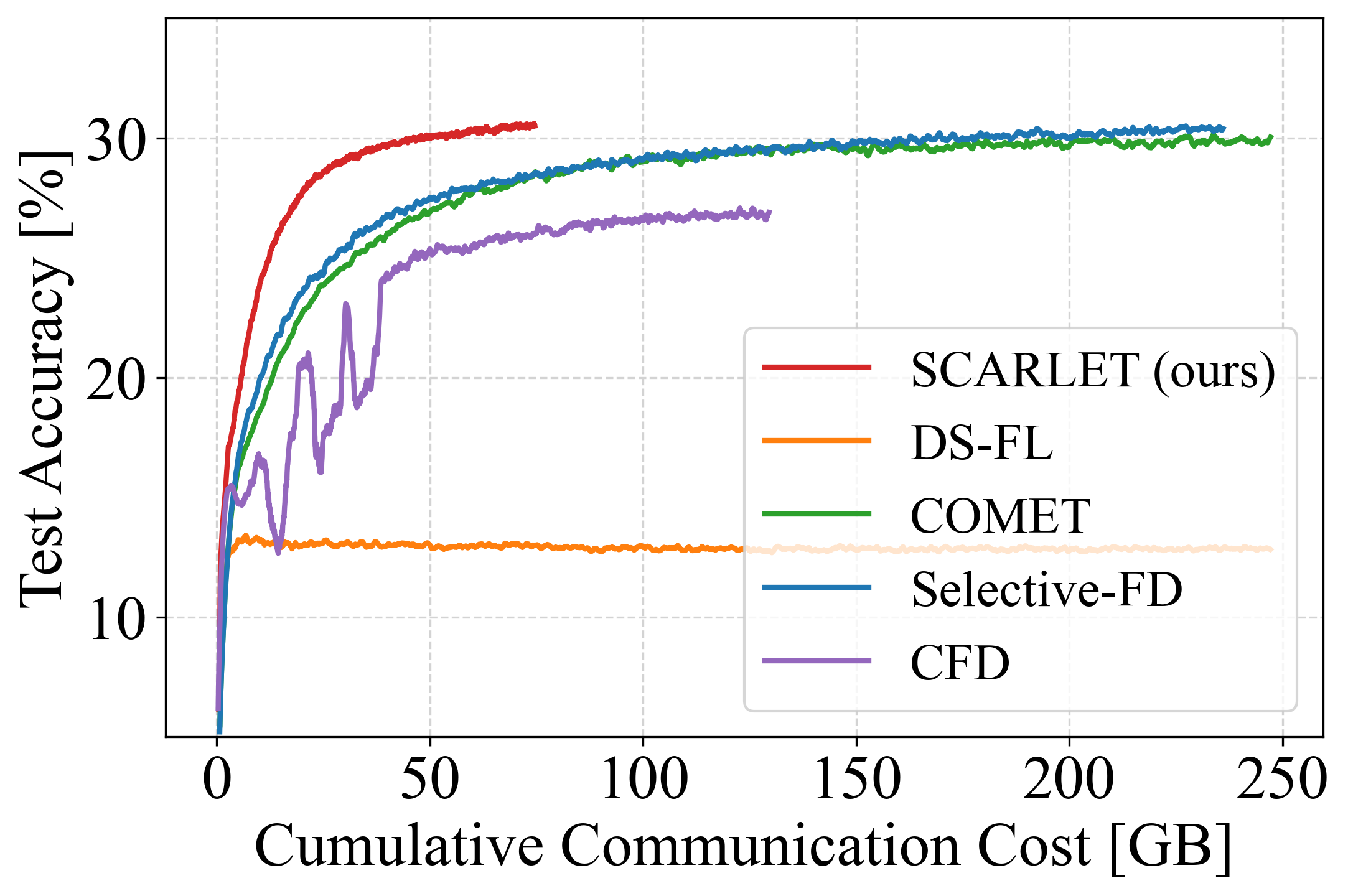}%
    \label{fig:performance:cifar100_server}%
  }\hfill
  \subfloat[CIFAR-100 (Client)]{%
    \includegraphics[width=0.48\columnwidth]{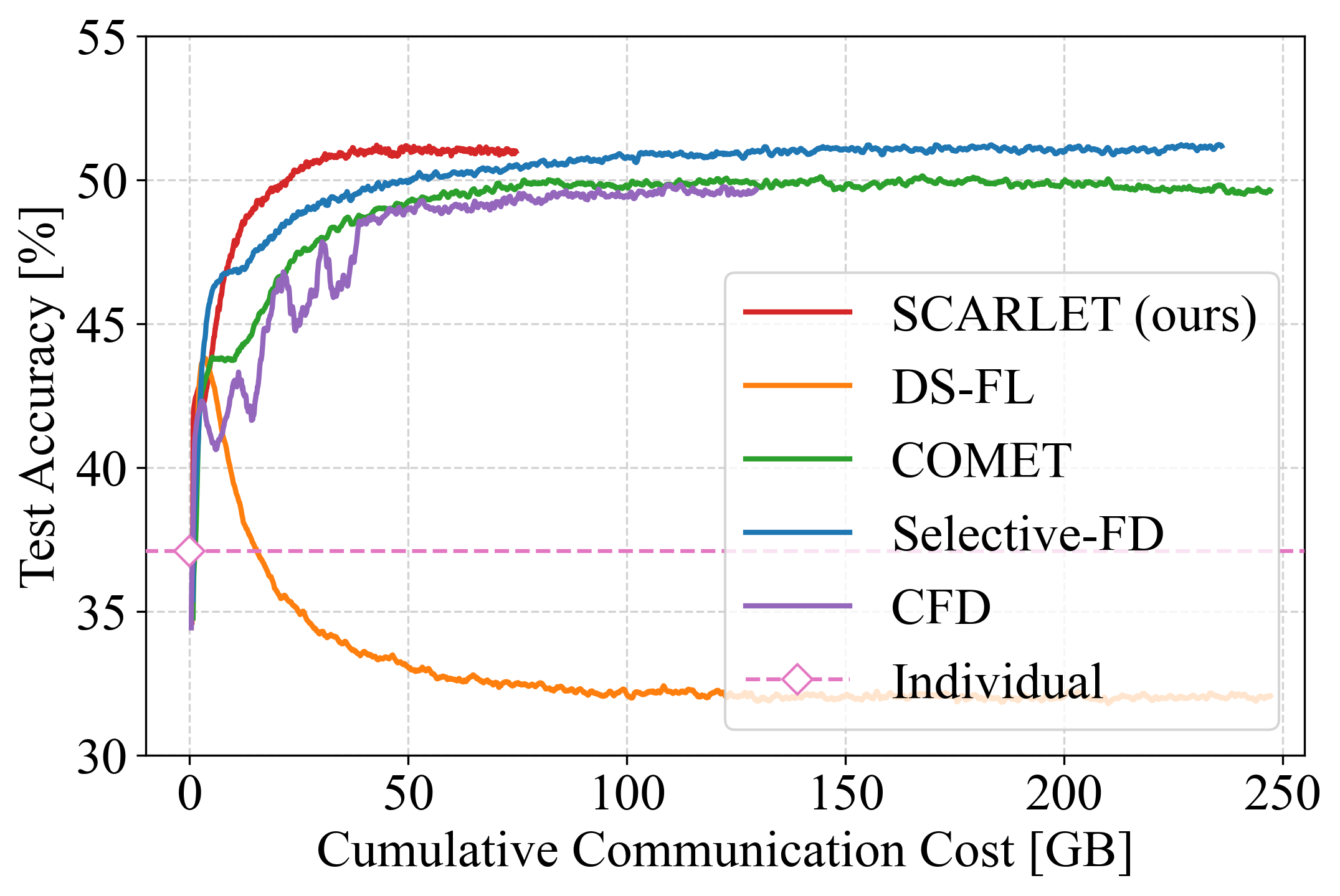}%
    \label{fig:performance:cifar100_client}%
  }

  \subfloat[Tiny ImageNet (Server)]{%
    \includegraphics[width=0.48\columnwidth]{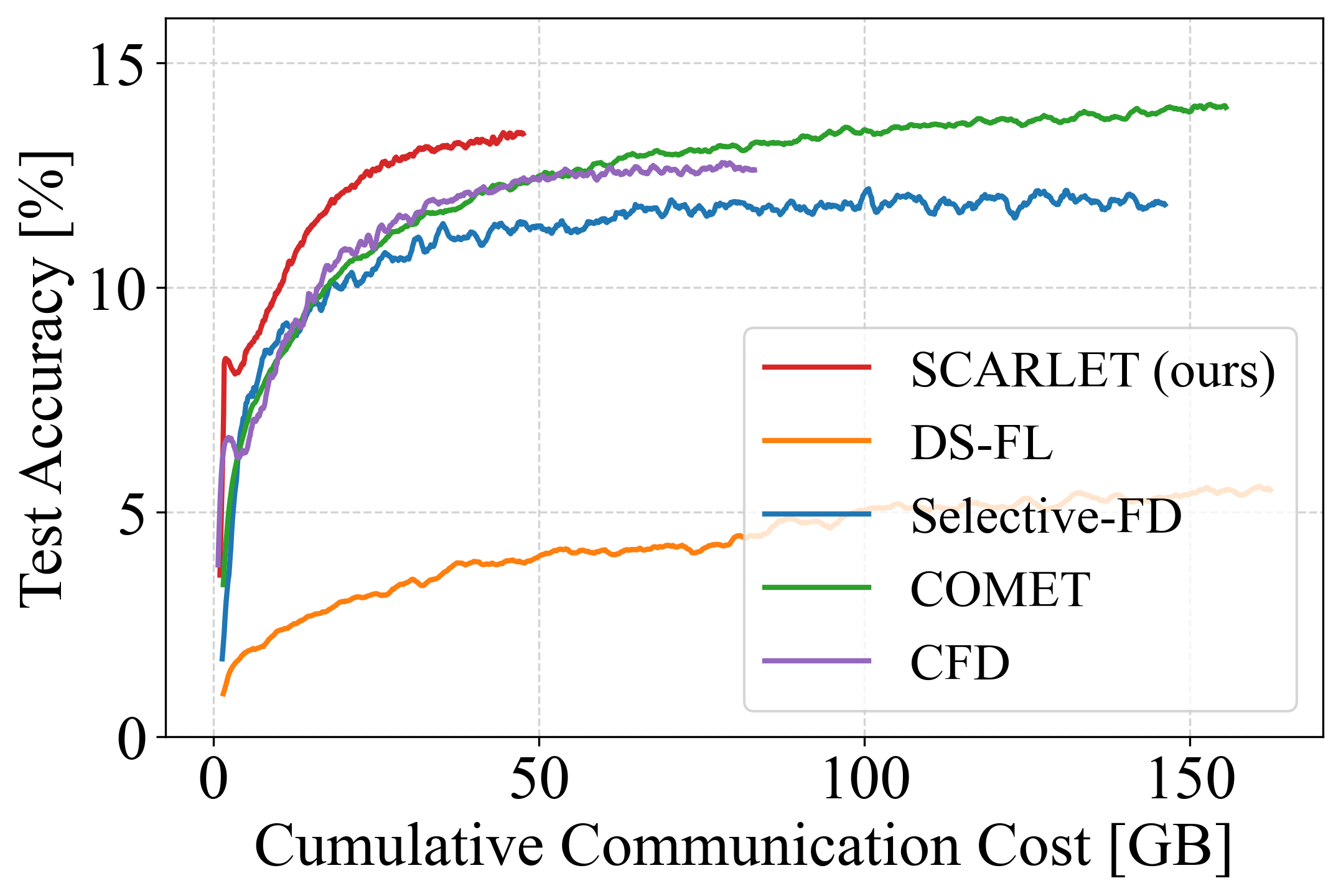}%
    \label{fig:performance:tiny_server}%
  }\hfill
  \subfloat[Tiny ImageNet (Client)]{%
    \includegraphics[width=0.48\columnwidth]{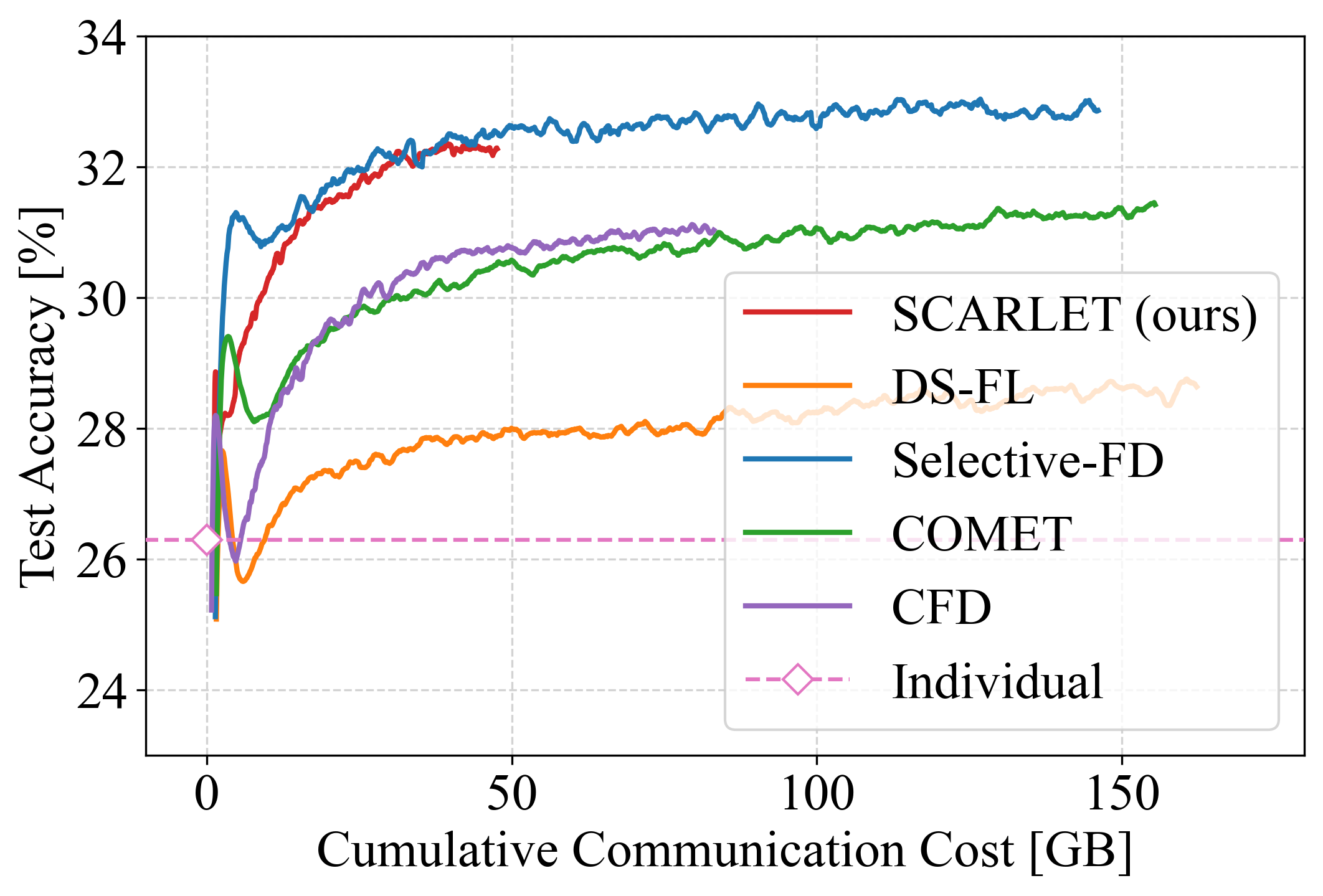}%
    \label{fig:performance:tiny_client}%
  }

  \caption{Test accuracy vs. communication cost on more complex datasets, CIFAR-100 and Tiny ImageNet. The results confirm the scalability and robustness of SCARLET.
  Even in these more challenging settings with a larger number of classes,
  SCARLET consistently maintains the highest communication efficiency,
  achieving competitive accuracy while requiring less communication than other methods.
  This demonstrates that the benefits of soft-label caching and Enhanced ERA generalize well beyond simple datasets.}
  \label{fig:performance_cifar100_tiny}
\end{figure}

Fig.~\ref{fig:performance_cifar100_tiny} presents the results for CIFAR-100 and Tiny ImageNet, demonstrating the scalability and robustness of SCARLET. The soft-label caching mechanism and Enhanced ERA consistently achieved high communication efficiency and competitive accuracy across datasets with more complex label distributions and greater heterogeneity. On the server side, SCARLET achieves test accuracy comparable to the top-performing baselines (Selective-FD and COMET) while reducing cumulative communication costs by approximately 70\%. On the client side, while Selective-FD retains a slight edge in final accuracy, SCARLET consistently ranks as a close second, delivering competitive performance that significantly outperforms the remaining baselines.

In contrast, DS-FL exhibits significantly degraded performance on these datasets. This failure highlights the inherent limitations of the conventional ERA mechanism: specifically, its sensitivity to the larger number of classes ($N$) and the practical infeasibility of finding a single optimal temperature $T$ capable of effectively regularizing the aggregated soft-labels.

In summary, SCARLET demonstrates robust performance in a variety of challenging FL scenarios. Its integration of the soft-label caching mechanism and Enhanced ERA mechanism not only reduces communication costs significantly but also sustains high model accuracy. While Selective-FD retains advantages in specific client-side settings, SCARLET’s adaptability and modular design open pathways for hybrid solutions that leverage the strengths of both approaches. Results on CIFAR-100 and Tiny ImageNet further reinforce SCARLET’s potential as a scalable, communication-efficient framework for diverse FL environments.

\smallbreak
\subsubsection{Effectiveness of the Soft-Label Caching Mechanism in Other SOTA Methods}
\label{subsubsection:effectiveness_of_the_soft-label_caching_mechanism_in_other_sota_methods}
\smallbreak
The soft-label caching mechanism, a core component of SCARLET, is designed to alleviate communication inefficiencies in distillation-based FL. Since such inefficiencies are not specific to SCARLET, the mechanism can also be applied to other SOTA methods to improve their communication efficiency. To evaluate its broader impact, we integrated the caching mechanism into three representative FL methods: CFD, COMET, and Selective-FD. DS-FL was excluded from this evaluation, as it becomes conceptually similar to SCARLET when caching is introduced. For all methods, we used a conservative cache duration of $D = 25$ to balance communication reduction with the risk of overfitting to stale predictions. This evaluation was performed on CIFAR-10 under strong non-IID conditions ($\alpha=0.05$).
The results presented in this section were averaged across three random seeds.

\begin{figure}[t]
  \centering

  \subfloat[CFD (Server)]{%
    \includegraphics[width=0.48\columnwidth]{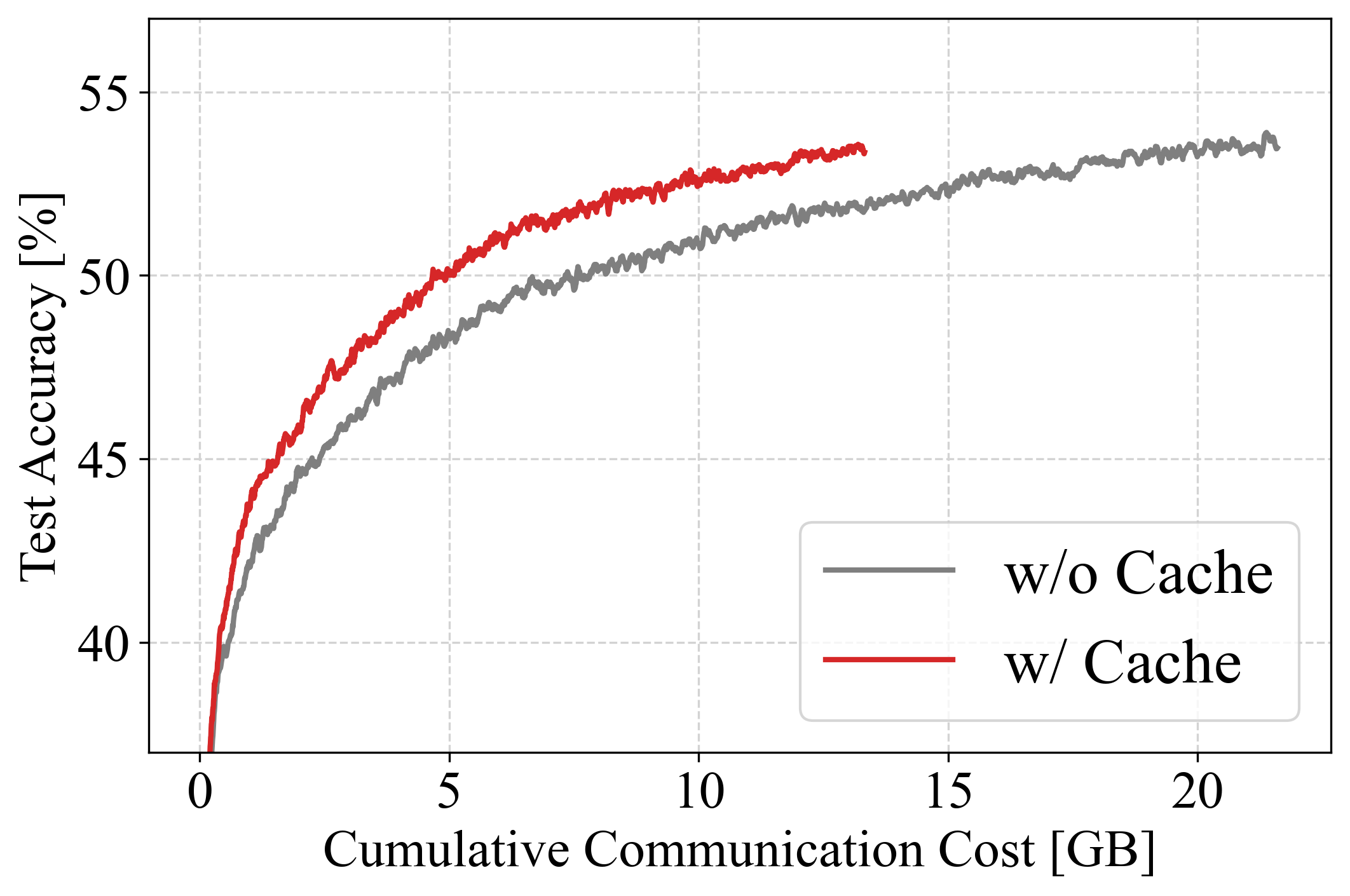}%
    \label{fig:slcm_integration:cfd_server}%
  }\hfill
  \subfloat[CFD (Client)]{%
    \includegraphics[width=0.48\columnwidth]{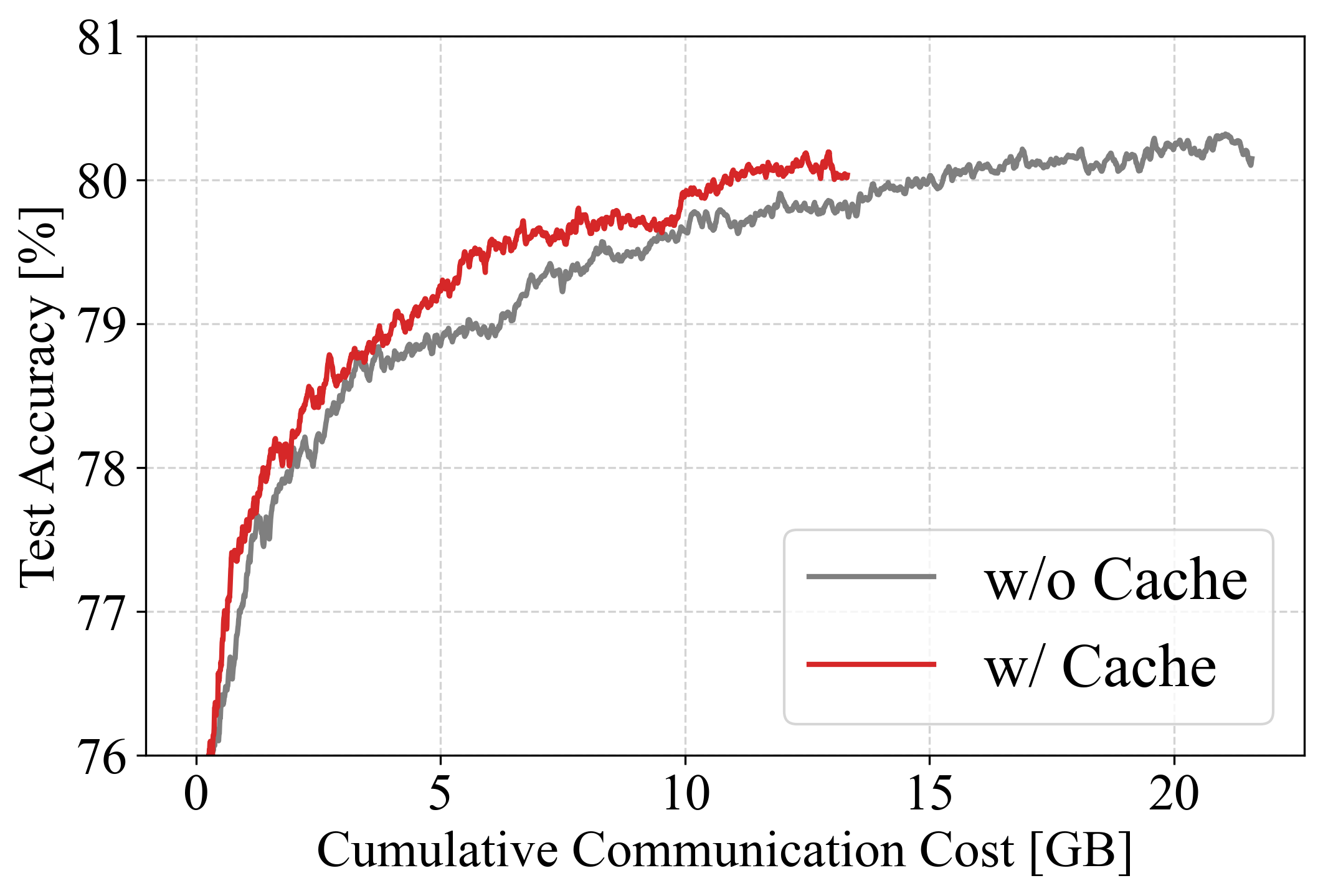}%
    \label{fig:slcm_integration:cfd_client}%
  }

  \subfloat[COMET (Server)]{%
    \includegraphics[width=0.48\columnwidth]{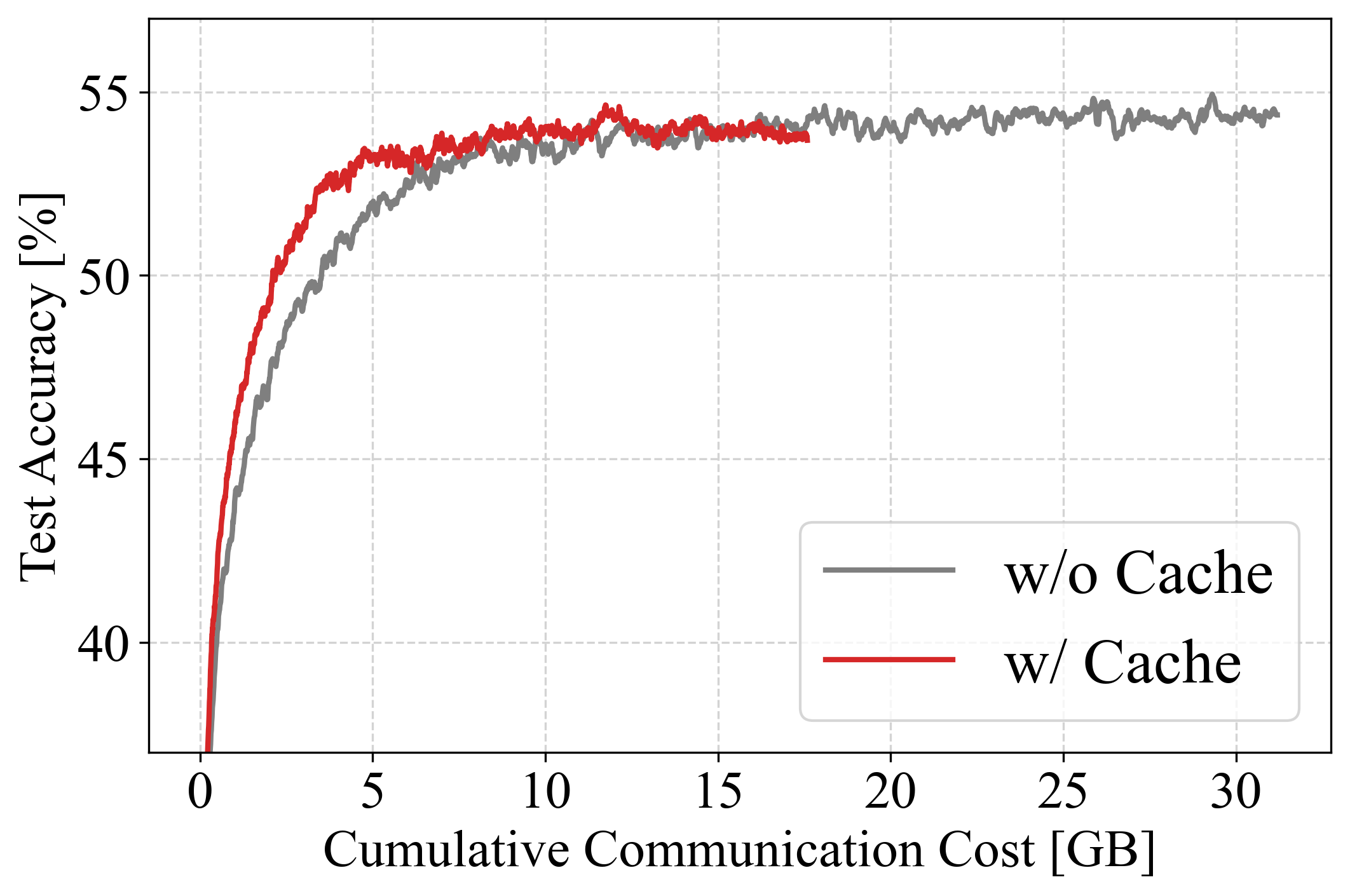}%
    \label{fig:slcm_integration:comet_server}%
  }\hfill
  \subfloat[COMET (Client)]{%
    \includegraphics[width=0.48\columnwidth]{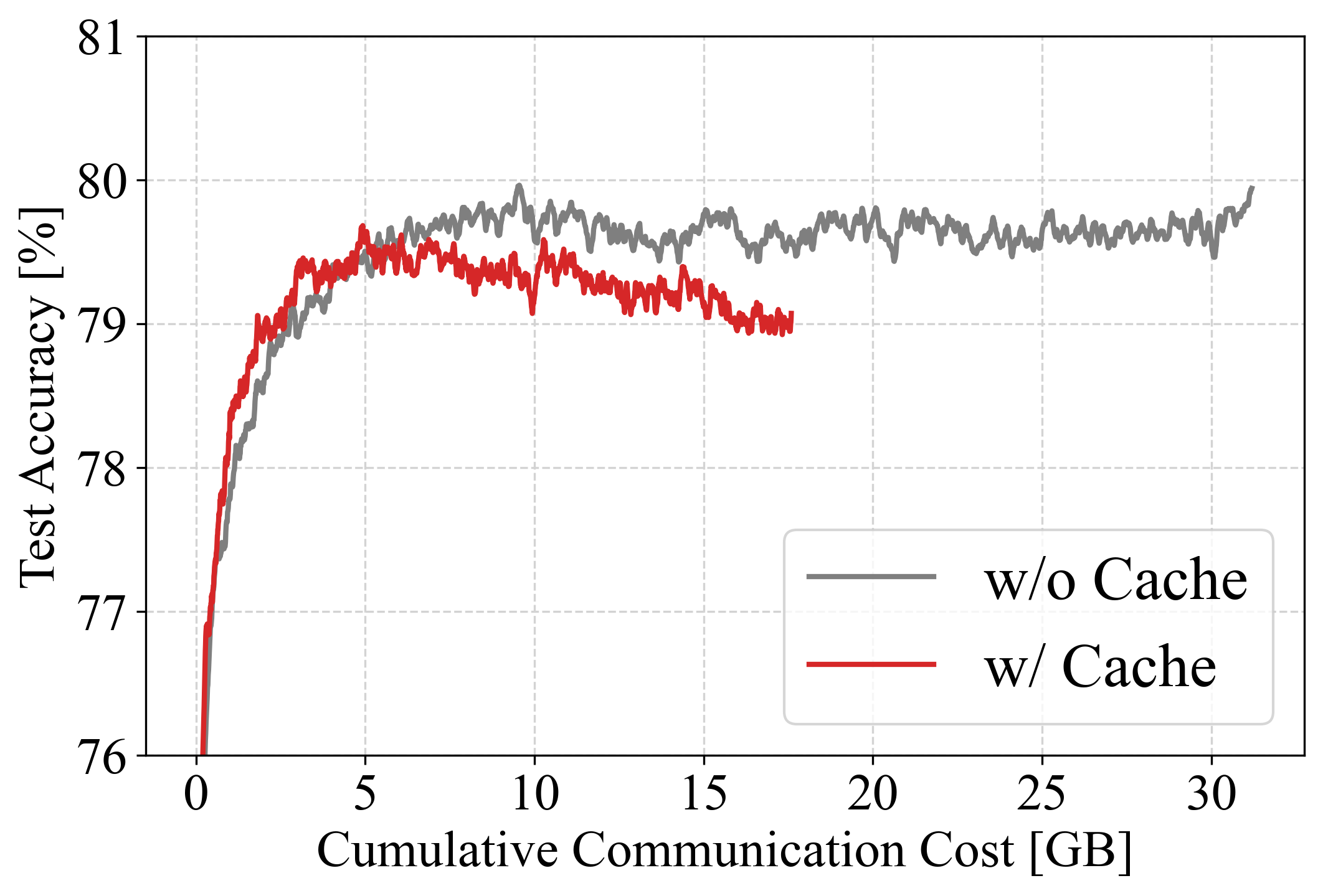}%
    \label{fig:slcm_integration:comet_client}%
  }

  \subfloat[Selective-FD (Server)]{%
    \includegraphics[width=0.48\columnwidth]{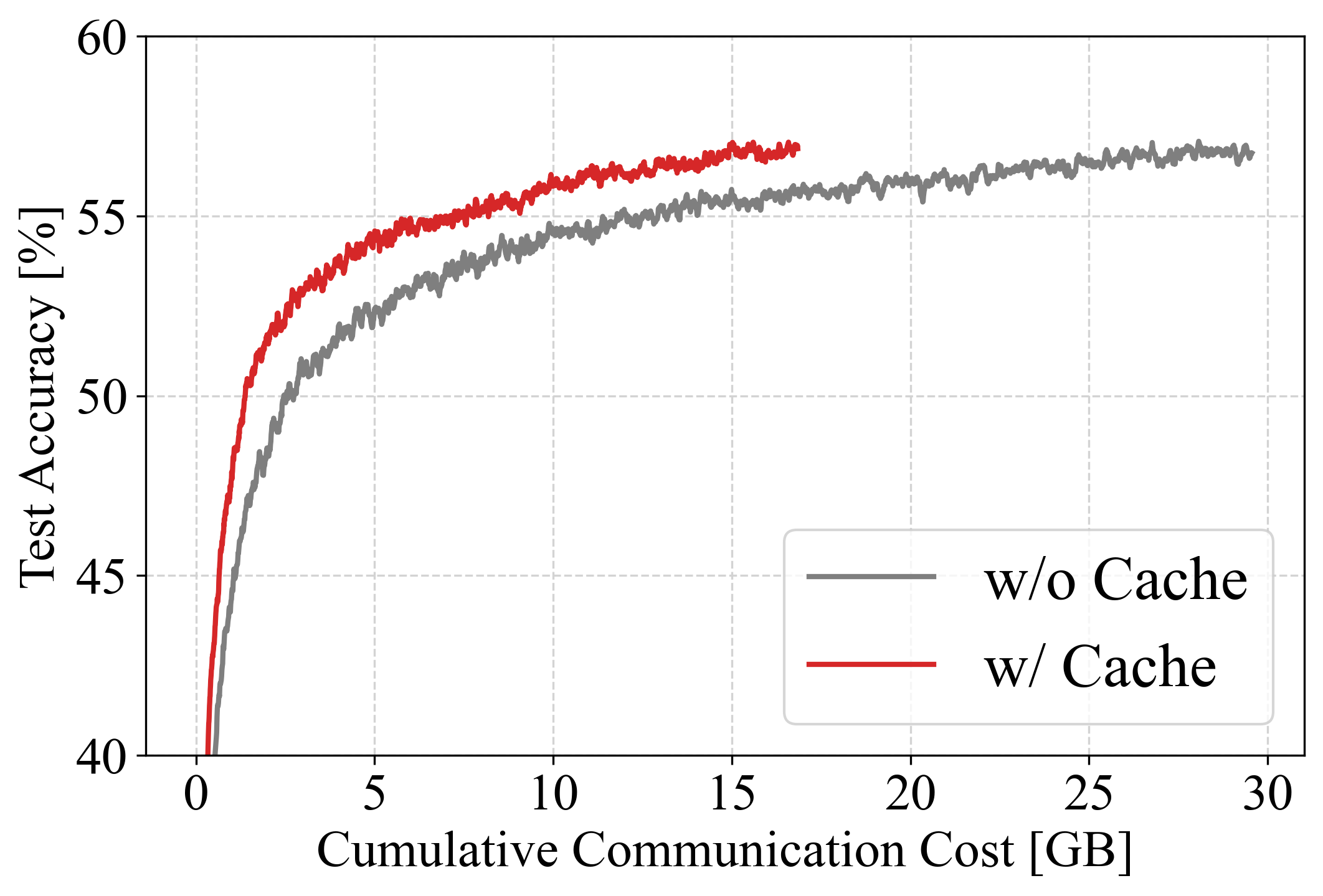}%
    \label{fig:slcm_integration:selectivefd_server}%
  }\hfill
  \subfloat[Selective-FD (Client)]{%
    \includegraphics[width=0.48\columnwidth]{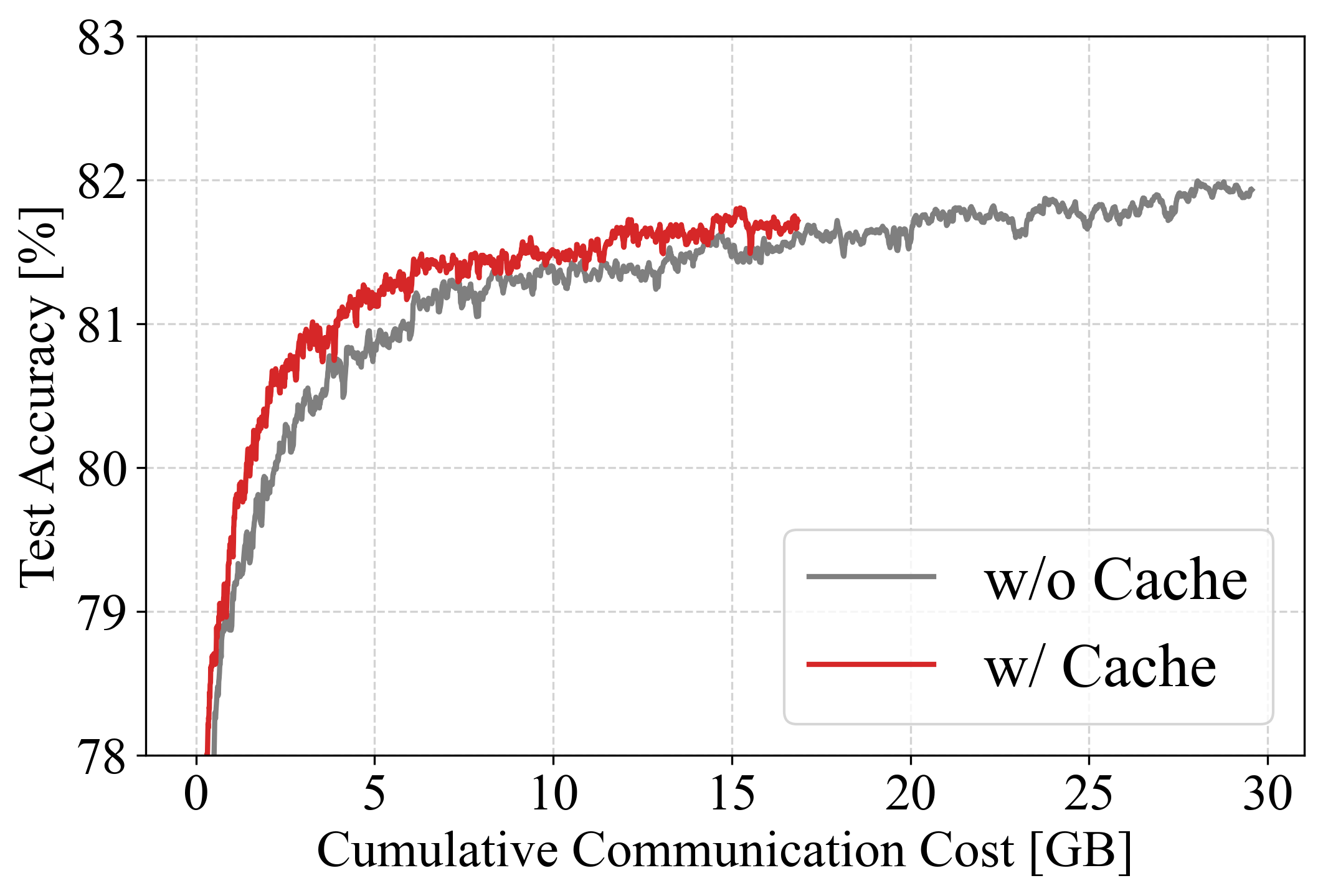}%
    \label{fig:slcm_integration:selectivefd_client}%
  }

  \caption{Impact of integrating SCARLET's soft-label caching mechanism into other state-of-the-art methods on CIFAR-10 ($\alpha=0.05$). The results demonstrate that our caching mechanism is a modular and widely applicable component. When applied to CFD, COMET, and Selective-FD, it reduces their communication costs by approximately 50\% with minimal or no degradation in server-side test accuracy.
  This highlights its potential as a general-purpose enhancement for distillation-based FL frameworks.}
  \label{fig:slcm_integration}
\end{figure}

Fig.~\ref{fig:slcm_integration} presents the server-side and client-side test accuracy versus cumulative communication cost for each method with the soft-label caching mechanism applied. The results show that incorporating the caching mechanism reduced server-side communication costs by approximately 50\% across all three methods, without substantial loss in target model accuracy. Notably, CFD and Selective-FD maintained server-side accuracy levels comparable to their original implementations, indicating that the mechanism imposes minimal performance overhead. On the client side, although communication reductions were less dramatic due to the quicker convergence of personalized models, notable gains were still observed. Selective-FD, in particular, preserved strong personalized accuracy while benefiting from improved communication efficiency. COMET showed a notable client-side accuracy degradation in later training rounds. This is likely attributed to a fundamental conflict between COMET's dynamic mechanism of selecting the optimal clustered teacher every round and the soft-label caching mechanism's static nature of reusing a past teacher for $D$ rounds.

These findings demonstrate the generalizability and effectiveness of the soft-label caching mechanism. Its seamless integration into a variety of distillation-based FL methods---including CFD, COMET, and Selective-FD---highlights its potential as a practical, drop-in component for enhancing communication efficiency. By incorporating this mechanism, existing FL frameworks can achieve significant server-side communication reductions. They can also maintain strong global and personalized model performance when the caching mechanism does not conflict with the method's core design, as seen in CFD and Selective-FD. This contributes to the scalability and practicality of FL systems.

\smallbreak
\subsubsection{Ablation Study of Cache Duration}
\label{subsubsection:ablation_study_of_cache_duration}
\smallbreak
The cache duration $D$, a key parameter in our soft-label caching mechanism, significantly influences both communication cost and accuracy in SCARLET. We performed a detailed analysis using CIFAR-10 under a Dirichlet distribution with $\alpha = 0.05$, as results with other $\alpha$ values showed similar trends and are therefore omitted for brevity. The results for this ablation study were averaged across three random seeds.

\begin{figure}[t]
  \centering

  \subfloat[Server]{%
    \includegraphics[width=0.95\columnwidth]{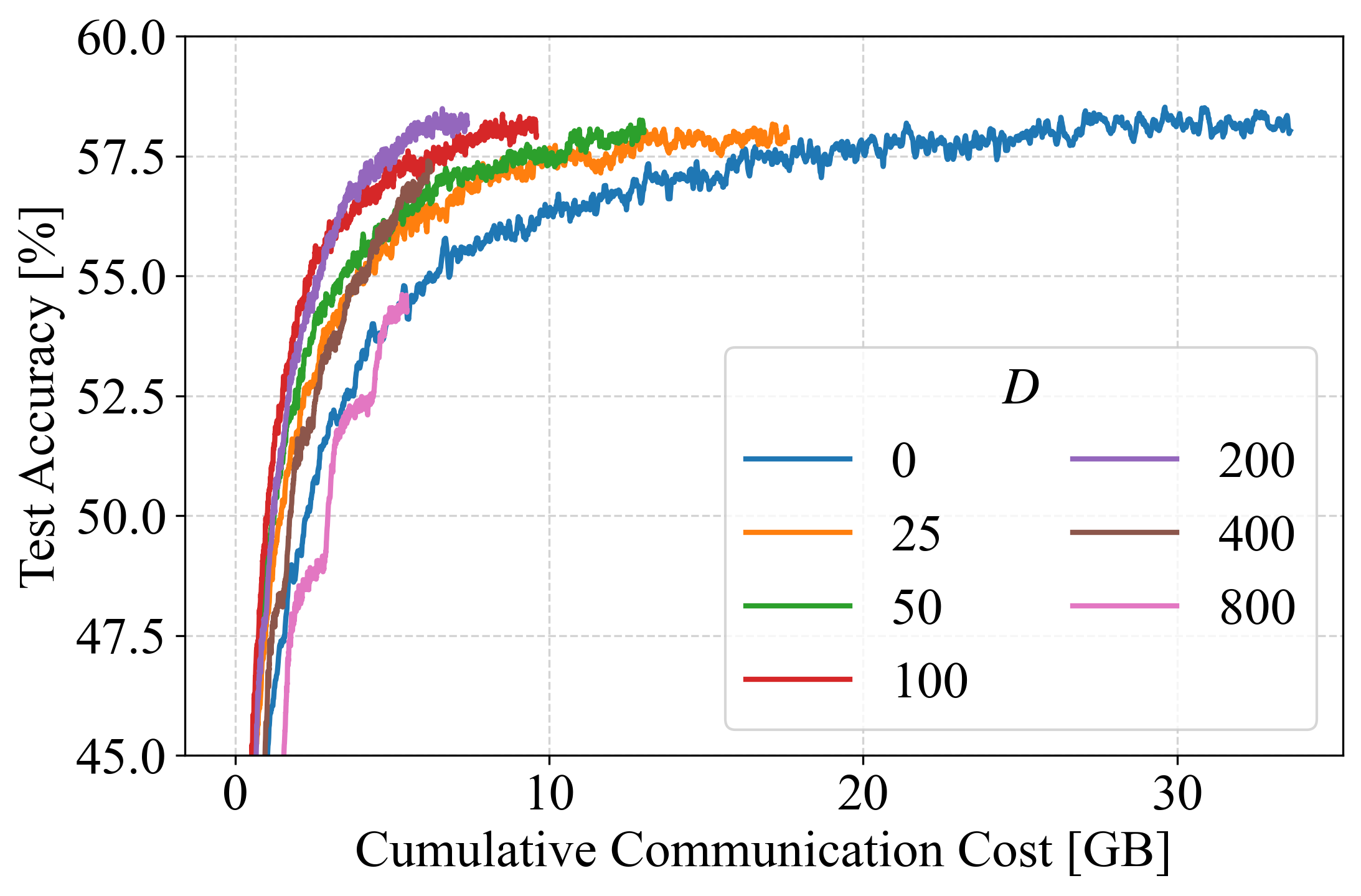}%
    \label{fig:ablation_duration:server}%
  }

  \subfloat[Client]{%
    \includegraphics[width=0.95\columnwidth]{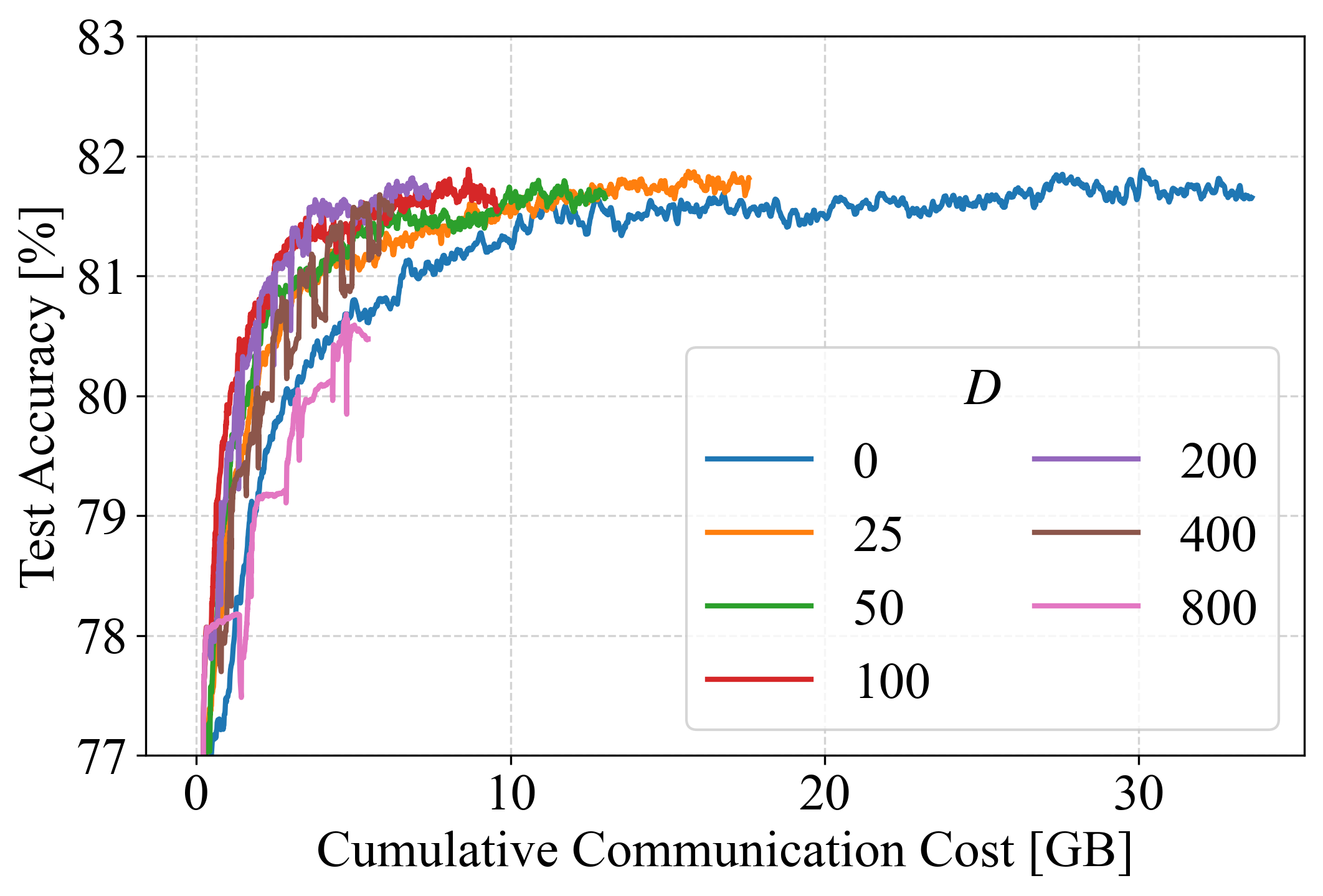}%
    \label{fig:ablation_duration:client}%
  }

  \caption{Impact of cache duration $D$ on the accuracy vs. communication cost trade-off (CIFAR-10, $\alpha=0.05$). This figure highlights that while caching ($D>0$) always saves communication, excessively long durations degrade performance. Increasing $D$ from 0 (blue line) to 50 or 100 (green/red lines) significantly shifts the curves left (saving communication) with minimal accuracy loss. However, very long durations (e.g., $D \geq 400$) cause a clear drop in final accuracy, demonstrating the negative impact of using stale cached labels.}
  \label{fig:ablation_duration}
\end{figure}

The results in Fig.~\ref{fig:ablation_duration} demonstrate the trade-offs associated with cache duration. On the server side (Fig.~\ref{fig:ablation_duration:server}), we observe that a conservative cache duration provides substantial communication savings with negligible impact on performance. Increasing $D$ from 0 (no cache) up to 200 results in almost no change to the final model accuracy (less than a 1 percentage point difference). However, as the cache duration becomes excessively long (e.g., $D=400$ or $D=800$), the model fails to converge to the target accuracy within the 3000-round training period. This performance degradation is caused by the system's heavy reliance on stale cached soft-labels, as predicted by our simulation in Section~\ref{subsubsection:cache_module} (Fig. \ref{fig:cache_simulation}). Notably, the stepwise accuracy improvements observed for $D=800$ at multiples of its cache duration (e.g., rounds 800 and 1600) support this hypothesis, as these moments correspond to the full refresh of the stale cache.

On the client side (Fig.~\ref{fig:ablation_duration:client}), a similar trend is observed. For appropriate cache durations (e.g., $D \le 200$), there is no clear degradation in final accuracy compared to the baseline. However, for excessively long durations (e.g., $D \ge 400$), the training stability markedly decreases, and the model's performance becomes erratic. Given that these long durations also fail to deliver optimal server-side performance, they are not considered effective.

These findings simplify the practical guidance for tuning $D$. The primary trade-off is not a gradual, nuanced decline, but rather a distinct performance cliff when $D$ becomes excessively large (e.g., $D \ge 400$). This behavior aligns with the risk identified in our earlier simulation (Fig.~\ref{fig:cache_simulation}), which predicted that excessively long durations lead to heavy reliance on stale cached soft-labels. The most critical insight for practical deployment is that a conservative duration (e.g., $D=50$ or $D=100$) provides a substantial reduction in communication cost (Fig.~\ref{fig:ablation_duration}) with minimal accuracy trade-offs. This demonstrates that a robust and effective $D$ that avoids significant performance loss is easily identifiable, making cache duration a valuable parameter for tuning communication efficiency in FL systems.

\smallbreak
\subsubsection{Ablation Study of Enhanced ERA}
\label{subsubsection:ablation_study_of_enhanced_era}
\smallbreak
To evaluate the impact of the aggregation sharpness parameter $\beta$ in the Enhanced ERA mechanism, we conducted experiments on CIFAR-10 under both strong ($\alpha = 0.05$) and moderate ($\alpha = 0.3$) non-IID conditions. A fixed cache duration of $D = 50$ was used in all experiments, and the results are presented in Fig.~\ref{fig:ablation_eera}.

\begin{figure}[t]
  \centering

  \subfloat[Server ($\alpha = 0.05$)]{%
    \includegraphics[width=0.48\columnwidth]{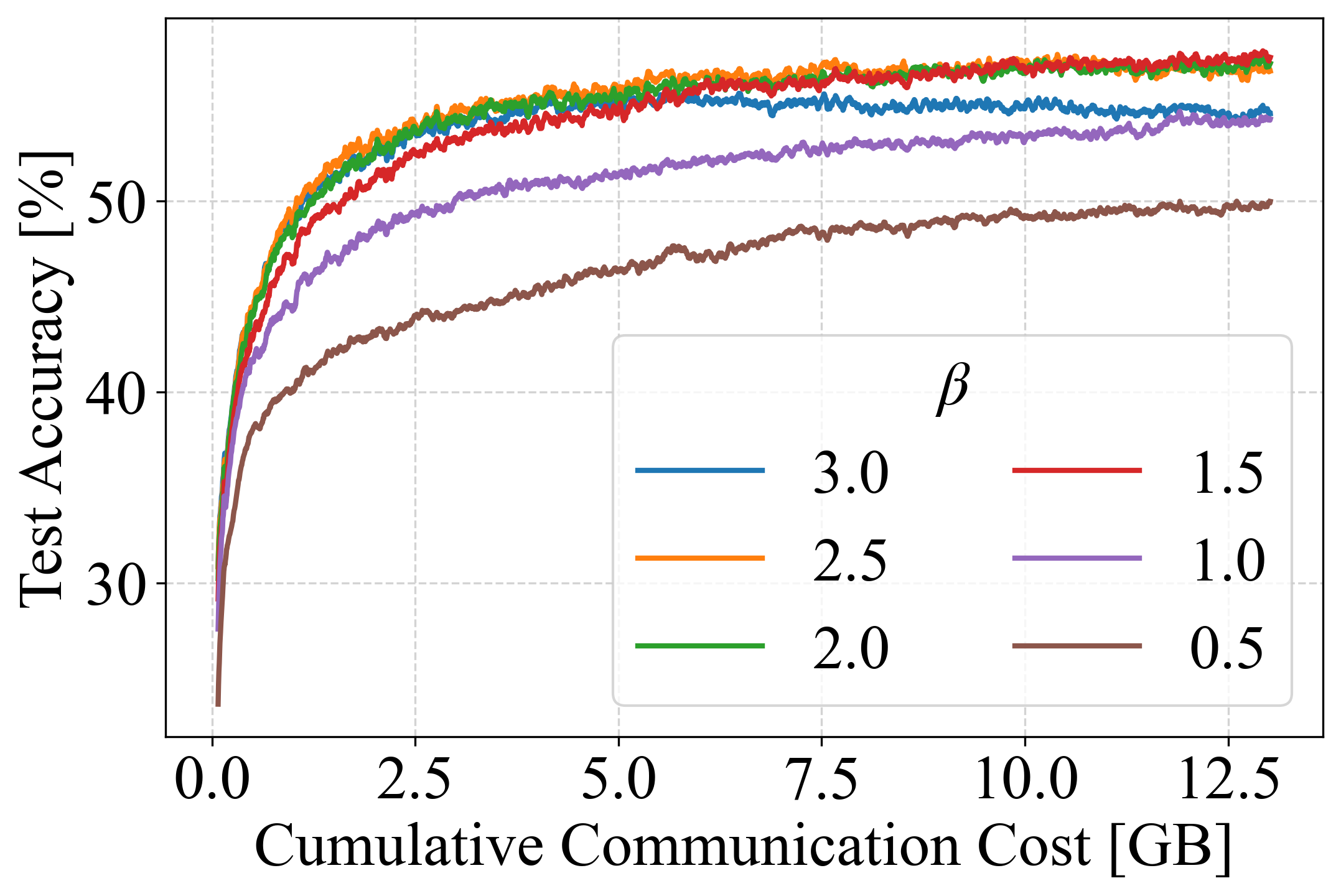}%
    \label{fig:ablation_eera:server_a005}%
  }\hfill
  \subfloat[Client ($\alpha = 0.05$)]{%
    \includegraphics[width=0.48\columnwidth]{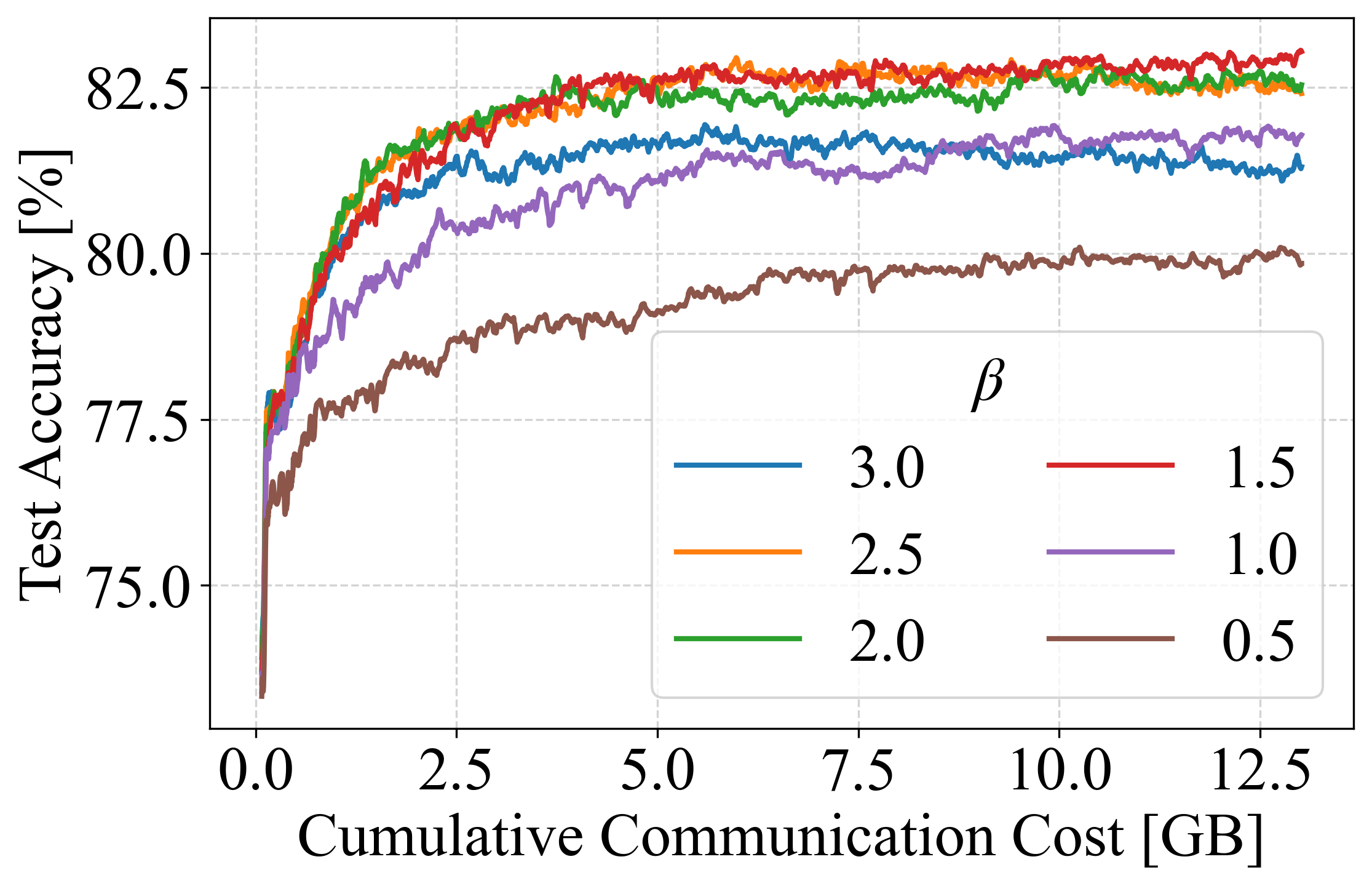}%
    \label{fig:ablation_eera:client_a005}%
  }

  \subfloat[Server ($\alpha = 0.3$)]{%
    \includegraphics[width=0.48\columnwidth]{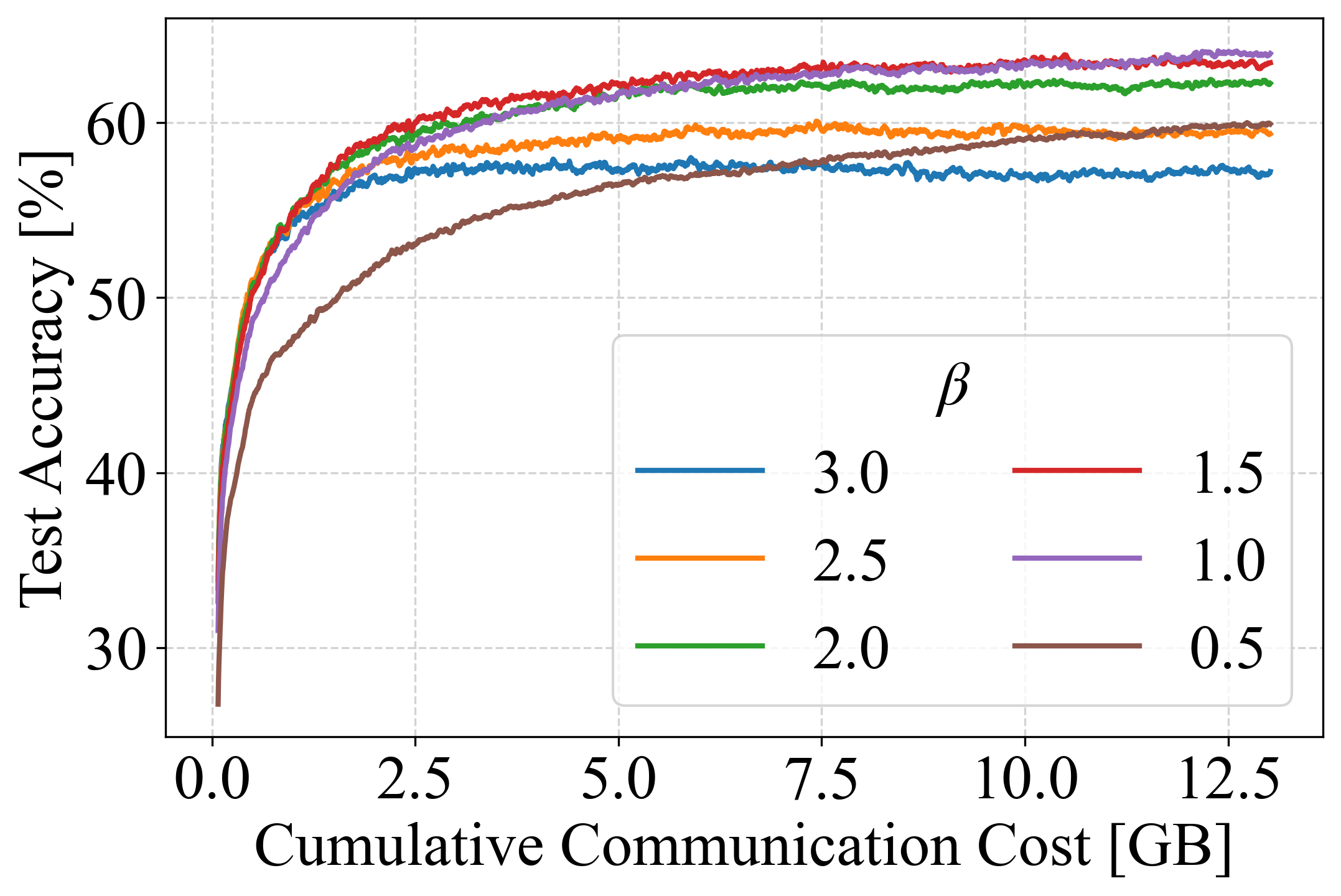}%
    \label{fig:ablation_eera:server_a03}%
  }\hfill
  \subfloat[Client ($\alpha = 0.3$)]{%
    \includegraphics[width=0.48\columnwidth]{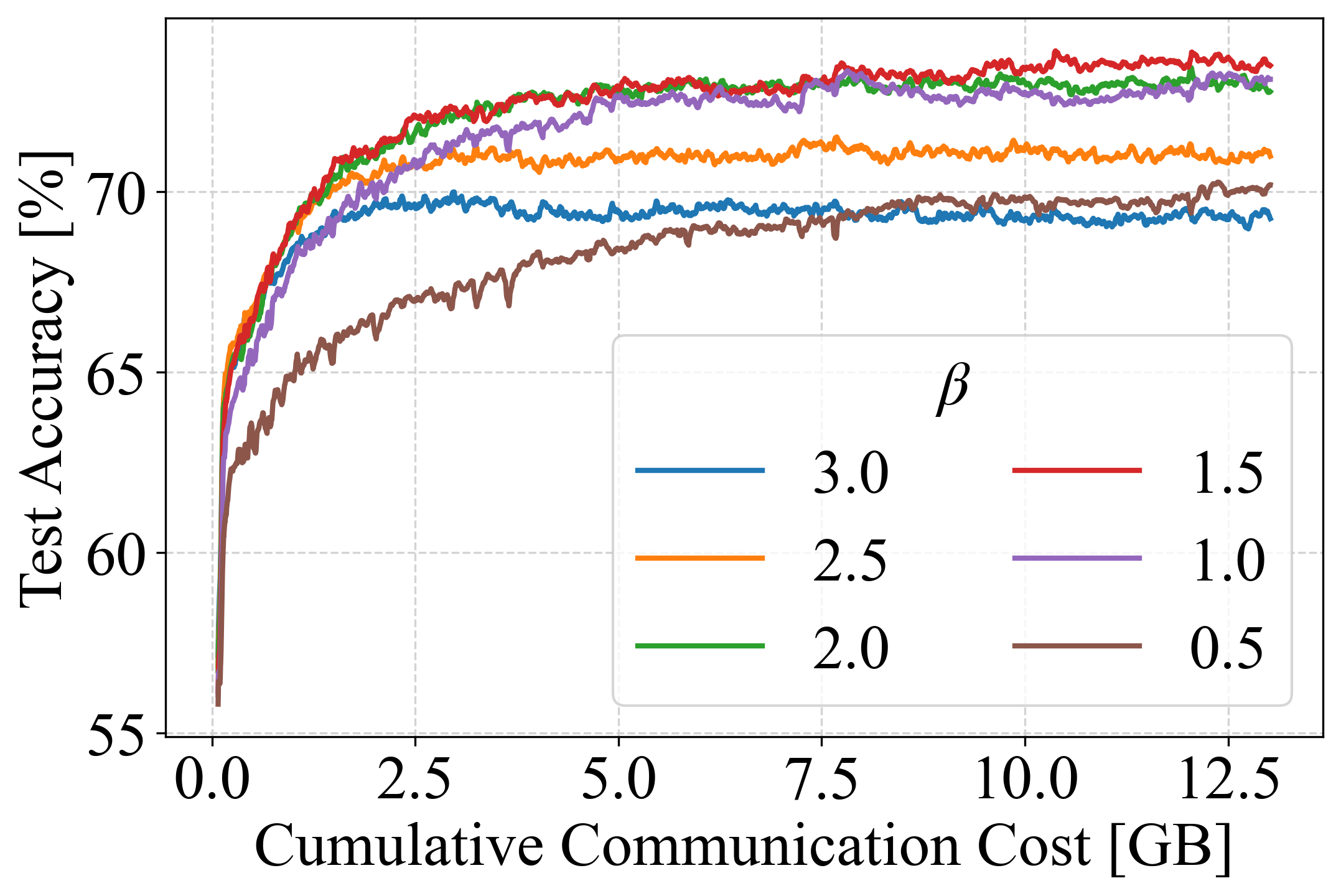}%
    \label{fig:ablation_eera:client_a03}%
  }

  \caption{Impact of the aggregation sharpness parameter $\beta$ in Enhanced ERA on test accuracy for SCARLET on CIFAR-10
  under Dirichlet distributions with $\alpha = 0.05$ (strong non-IID) and $\alpha = 0.3$ (moderate non-IID). Varying $\beta$ exhibits distinct effects: under strong non-IID conditions (\ref{fig:ablation_eera:server_a005}, \ref{fig:ablation_eera:client_a005}), increasing $\beta$ (e.g., $\beta \ge 1.5$) accelerates server-side convergence and improves accuracy. Under moderate non-IID conditions (\ref{fig:ablation_eera:server_a03}, \ref{fig:ablation_eera:client_a03}), $\beta=1.0$ (simple averaging) is most stable for the server, while $\beta=1.5$ achieves the highest client-side accuracy, highlighting a trade-off between generalization and personalization.}
  \label{fig:ablation_eera}
\end{figure}

For strong non-IID scenarios ($\alpha = 0.05$, Fig.~\ref{fig:ablation_eera:server_a005} and \ref{fig:ablation_eera:client_a005}), increasing $\beta$ values (e.g., $\beta \ge 1.5$) significantly improved server-side convergence (Fig.~\ref{fig:ablation_eera:server_a005}). While $\beta = 1.5$ achieved the highest final accuracy on both server and client, increasing $\beta$ further to 2.0 offered a slight improvement in convergence speed at the cost of a marginal (less than 1 percentage point) drop in final accuracy. This acceleration is attributed to the sharper aggregation of client predictions, which mitigates the smoothing effects of locally biased soft-labels. However, on the client side (Fig.~\ref{fig:ablation_eera:client_a005}), excessively high $\beta$ values caused accuracy degradation in later rounds. This degradation likely results from over-sharpened labels introducing biases during local training, thus limiting the generalization of client-specific models.

For moderate non-IID conditions ($\alpha = 0.3$, Fig.~\ref{fig:ablation_eera:server_a03} and \ref{fig:ablation_eera:client_a03}), $\beta = 1.0$ preserved steady server-side accuracy (Fig.~\ref{fig:ablation_eera:server_a03}) across all communication rounds, avoiding the risks associated with over-sharpening or over-smoothing. On the client side (Fig.~\ref{fig:ablation_eera:client_a03}), $\beta = 1.5$ yielded the highest accuracy, striking a balance between effective aggregation and model generalization. Larger $\beta$ values led to performance degradation, suggesting that less aggressive sharpening is preferable in scenarios with lower heterogeneity.

To more precisely characterize the sensitivity of Enhanced ERA to data heterogeneity, we extend the ablation beyond the two $\alpha$ settings and sweep $\alpha \in \{0.05, \dots, 0.30\}$ and $\beta \in \{0.5, \dots, 3.0\}$. For each ($\alpha$, $\beta$) pair, we report the mean test accuracy over the last 10 rounds at round 3000 and visualize the results as heatmaps for server- and client-side evaluations (Fig.~\ref{fig:ablation_eera_heatmap}). This grid complements the learning-curve analysis in Fig.~\ref{fig:ablation_eera} and clarifies how sharpening should be scheduled across non-IID strengths.

\begin{figure}[t]
  \centering

  \subfloat[Server]{%
    \includegraphics[width=0.48\columnwidth]{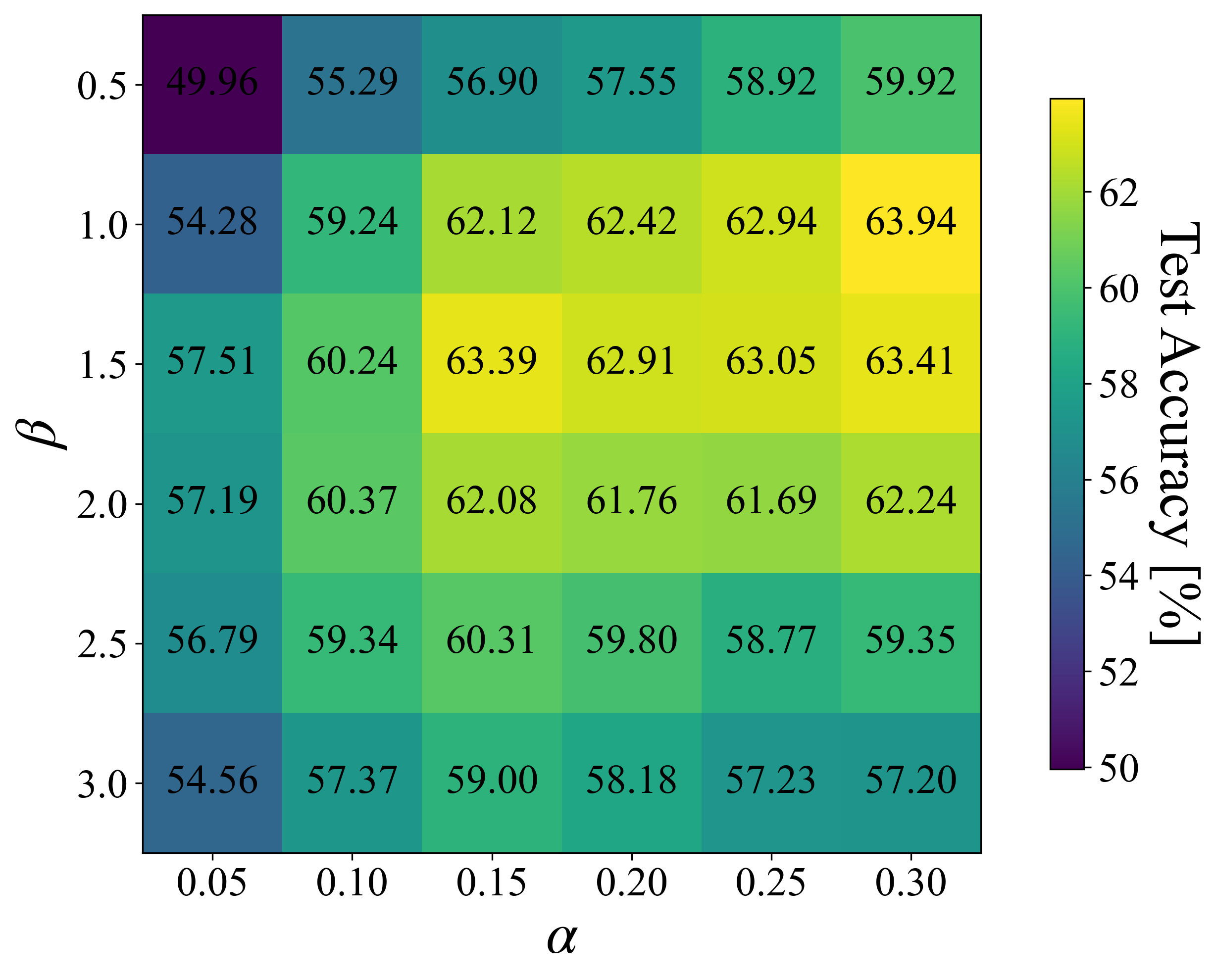}%
    \label{fig:ablation_eera_heatmap:server}%
  }\hfill
  \subfloat[Client]{%
    \includegraphics[width=0.48\columnwidth]{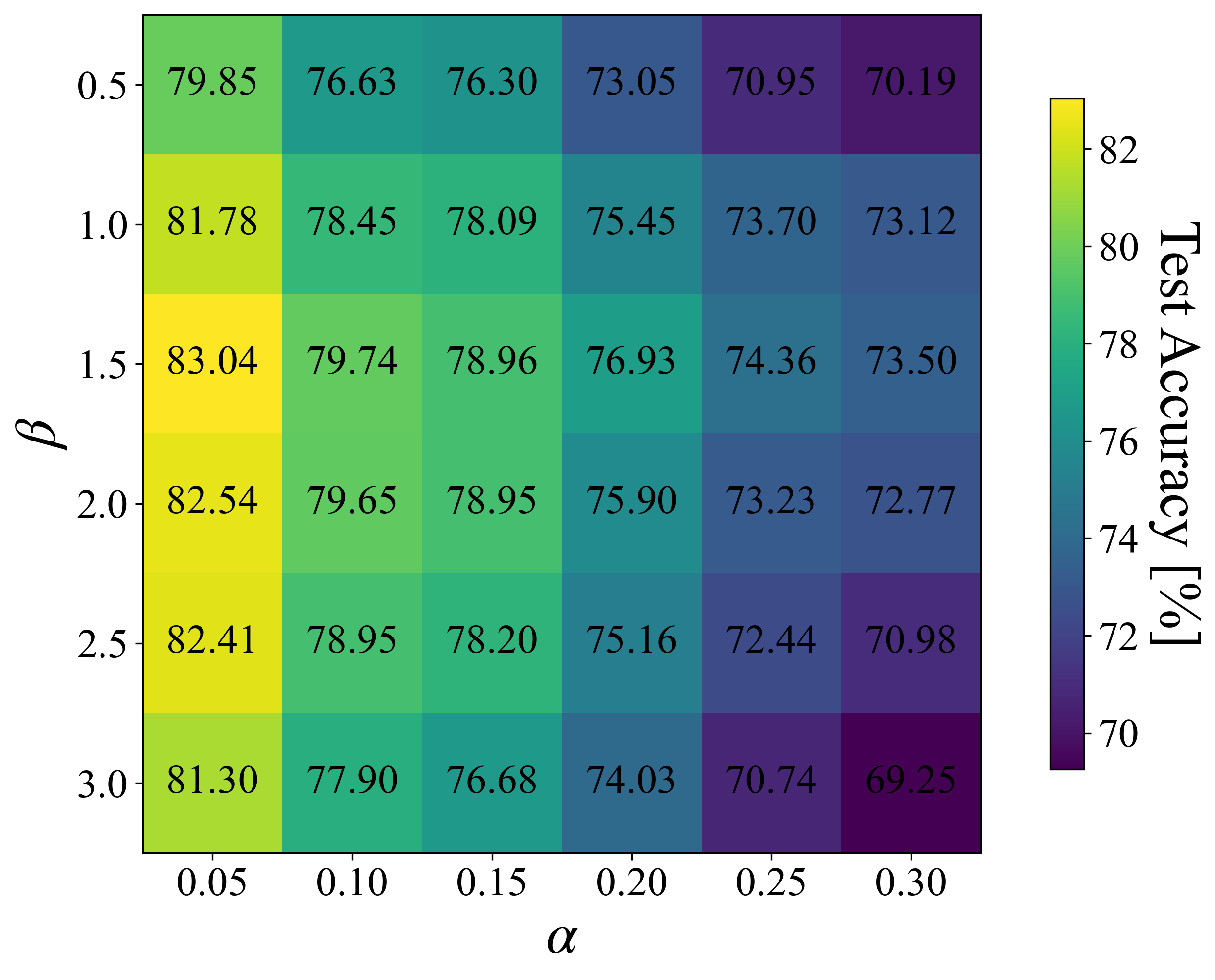}%
    \label{fig:ablation_eera_heatmap:client}%
  }

  \caption{Sensitivity of Enhanced ERA to data heterogeneity.
  Heatmaps show final test accuracy (\%) vs. aggregation sharpness $\beta$ and Dirichlet $\alpha$ for server-side (\ref{fig:ablation_eera_heatmap:server}) and client-side (\ref{fig:ablation_eera_heatmap:client}) evaluation. Each cell represents the mean accuracy over the last 10 rounds at round 3000. As $\alpha$ increases (approaching IID), the server-side optimum shifts toward $\beta = 1.0$, whereas $\beta = 1.5$ remains consistently strong across $\alpha$ and is typically optimal on the client side. Together with Fig.~\ref{fig:ablation_eera}, these results highlight a speed–accuracy trade-off: larger $\beta$ accelerates server-side convergence under strong non-IID conditions, while overly sharp aggregation can dampen late-stage client performance gains.}
  \label{fig:ablation_eera_heatmap}
\end{figure}

\begin{figure}[t]
  \centering

  \subfloat[Server]{%
    \includegraphics[width=0.48\columnwidth]{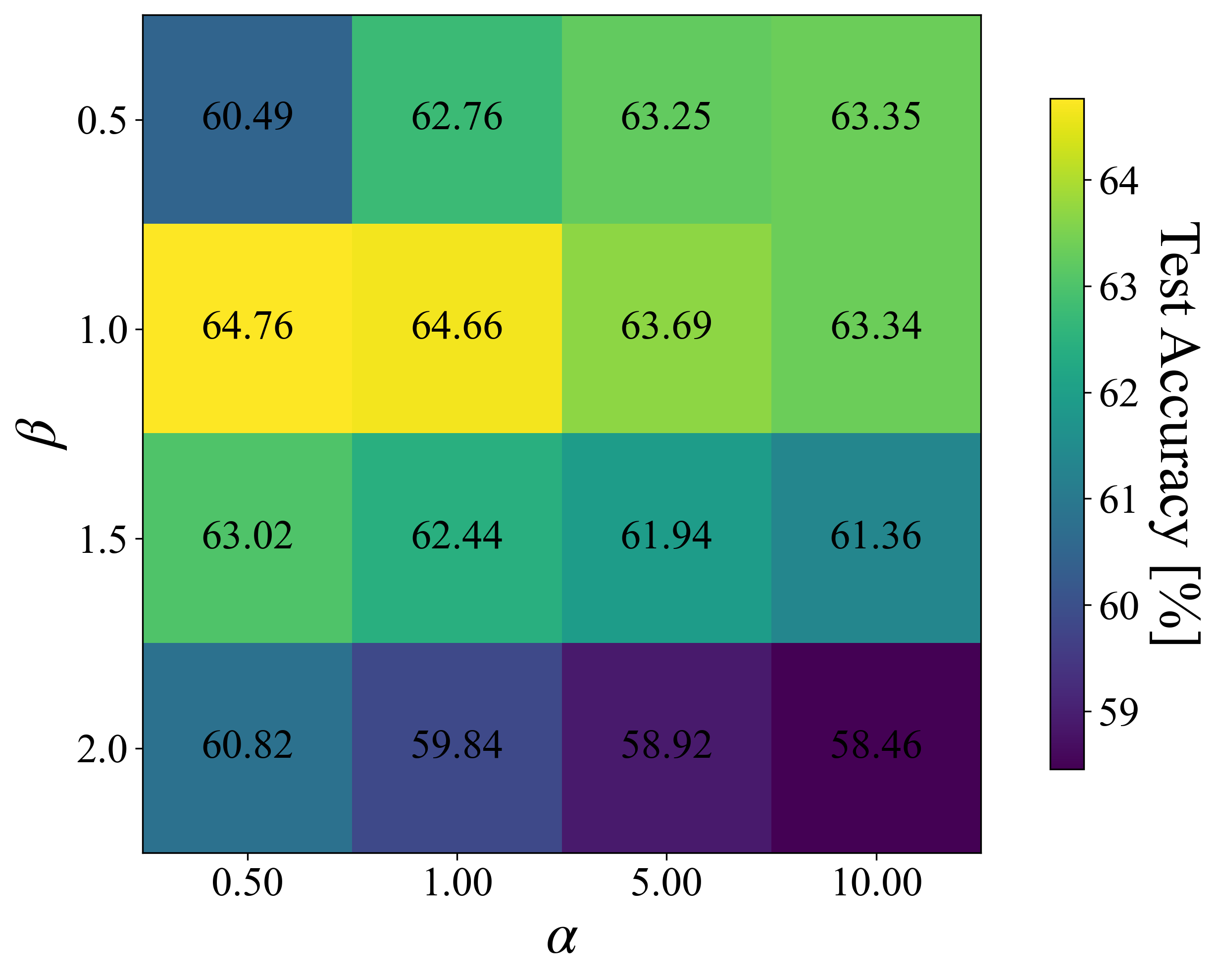}%
    \label{fig:near_iid_heatmap:server}%
  }\hfill
  \subfloat[Client]{%
    \includegraphics[width=0.48\columnwidth]{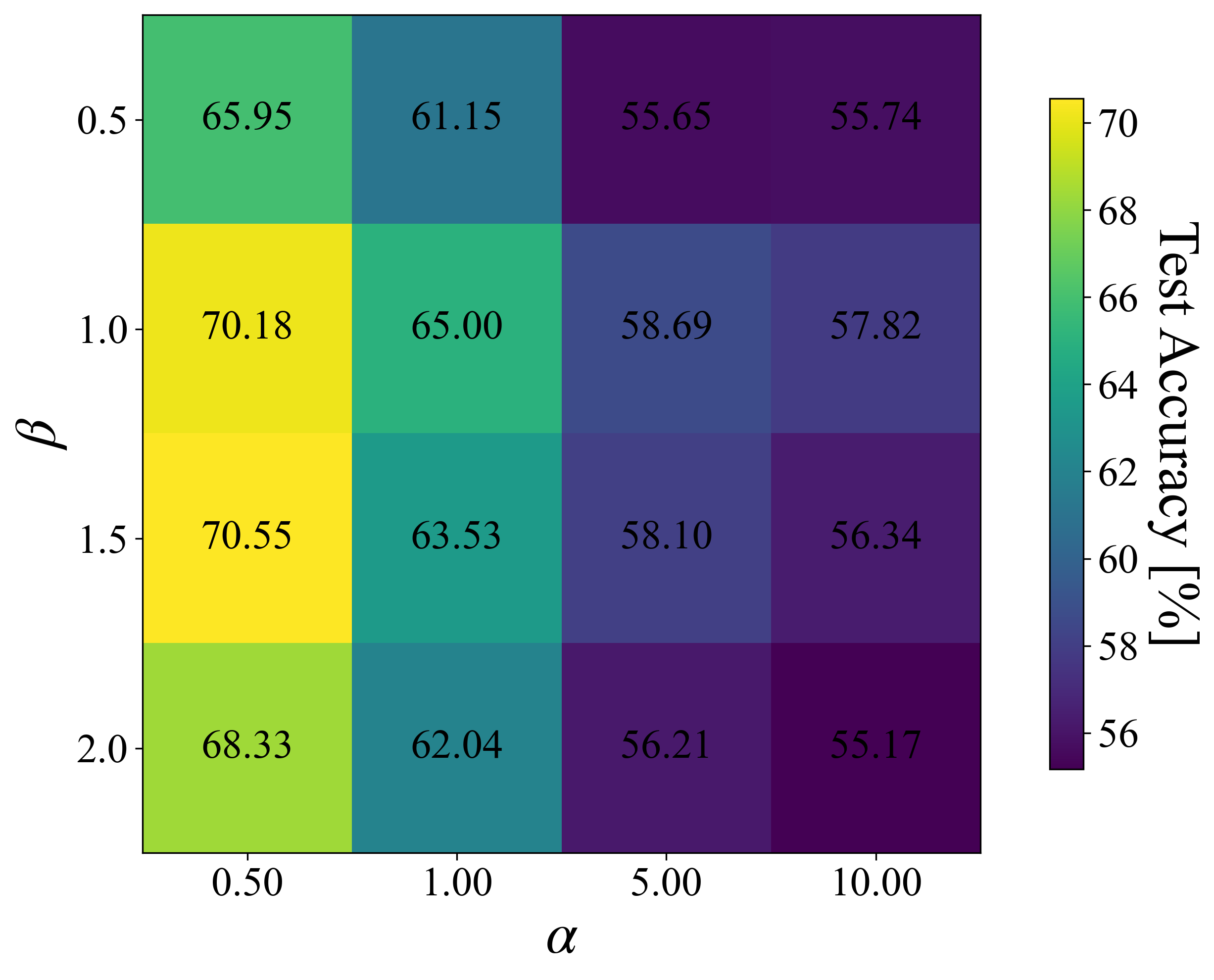}%
    \label{fig:near_iid_heatmap:client}%
  }

  \caption{Sensitivity of Enhanced ERA to near-IID data heterogeneity.
  Heatmaps show final test accuracy (\%) vs. aggregation sharpness $\beta$ and Dirichlet $\alpha$ for server-side (\ref{fig:near_iid_heatmap:server}) and client-side (\ref{fig:near_iid_heatmap:client}) evaluation. As $\alpha$ increases (approaching IID), the optimal performance for both server and client converges toward $\beta = 1.0$, confirming that Enhanced ERA behaves consistently with standard aggregation under homogeneous data conditions.}
  \label{fig:near_iid_heatmap}
\end{figure}

Three observations emerge consistently. On the server side, the optimal $\beta$ drifts toward 1.0 as the partition approaches IID: under strong non-IID ($\alpha = 0.05$), $\beta \approx 1.5 \text{-} 2.0$ yields the best accuracy, whereas under moderate non-IID ($\alpha = 0.3$) the optimum settles at $\beta = 1.0$. Notably, $\beta = 1.5$ is robust across the entire sweep, remaining within a small margin of the per-$\alpha$ optimum even when that optimum is $\beta=1.0$. On the client side, $\beta = 1.5$ consistently attains the highest final accuracy across $\alpha$, and the penalty from “over-sharpening” (increasing $\beta$ by one or two steps) is milder than on the server---particularly when $\alpha$ indicates stronger non-IID.

To provide a complete picture, we extended this evaluation to near-IID scenarios ($\alpha \ge 0.5$), as shown in Fig.~\ref{fig:near_iid_heatmap}. This evaluation follows the same methodology, sweeping $\alpha \in \{0.5, \dots, 10.0\}$ and $\beta \in \{0.5, \dots, 2.0\}$. The results confirm the trend: as the client data distribution approaches IID (i.e., as $\alpha$ increases), the optimal aggregation sharpness $\beta$ converges to 1.0 for both server-side and client-side evaluations.

For $\alpha \ge 1.0$, $\beta=1.0$ (which corresponds to simple soft-label averaging) consistently achieves optimal or near-optimal performance. This indicates that when data heterogeneity is low, the aggressive sharpening ($\beta > 1.0$) required for strong non-IID settings is unnecessary and, in fact, degrades performance.

Notably, on the server side (Fig. \ref{fig:near_iid_heatmap:server}), for near-IID settings ($\alpha \in \{5.0, 10.0\}$), a $\beta$ value below 1.0 (specifically $\beta=0.5$) achieves final accuracy comparable to that of the optimal $\beta=1.0$. This suggests that when all clients produce similar, high-confidence predictions, sharpening provides no additional benefit, and a slight ``softening" (or entropy-increasing) effect does not degrade—and may marginally improve—the global model's accuracy. This implies a potential for Enhanced ERA to be utilized ``in reverse" to manage overly confident and homogeneous client predictions, a direction warranting further investigation.

Taken together, these findings (Fig.~\ref{fig:ablation_eera}, \ref{fig:ablation_eera_heatmap}, and \ref{fig:near_iid_heatmap}) demonstrate that Enhanced ERA successfully translates the complex problem of non-IID aggregation into a practical speed-accuracy trade-off. This allows for simple, robust tuning guidelines for deployments where $\alpha$ is unknown: $\beta = 1.5$ serves as a strong default, achieving near-optimal for both server and clients across a broad range of heterogeneity. For specific goals, $\beta$ can be adjusted further: prioritizing faster server-side convergence under strong non-IID conditions suggests $\beta = 2.0$, while $\beta=1.0$ (i.e., simple averaging) is optimal as the data approaches IID.

\smallbreak
\subsubsection{Evaluation under Partial Client Participation}
\label{subsubsection:evaluation_under_partial_client_participation}
\smallbreak

We evaluated this partial participation scenario on CIFAR-10 with $K=100$ clients, varying the client participation ratio, denoted as $p\ (0 < p < 1)$, from 0.1 (10 clients per round) to 1.0 (all 100 clients). All other hyperparameters, including $D=50$ and $\beta$ values, remain the same as in the main experiment. The results are summarized in Fig.~\ref{fig:partial_participation}. These graphs plot the final accuracy (mean and standard deviation over the last 10 rounds, left y-axis) and the total cumulative communication cost (right y-axis) against the client participation ratio for both $\alpha=0.05$ (Fig.~\ref{fig:partial_participation:0-05}) and $\alpha=0.3$ (Fig.~\ref{fig:partial_participation:0-3}).

\begin{figure}[t]
  \centering

  \subfloat[Dirichlet $\alpha = 0.05$]{%
    \includegraphics[width=\columnwidth]{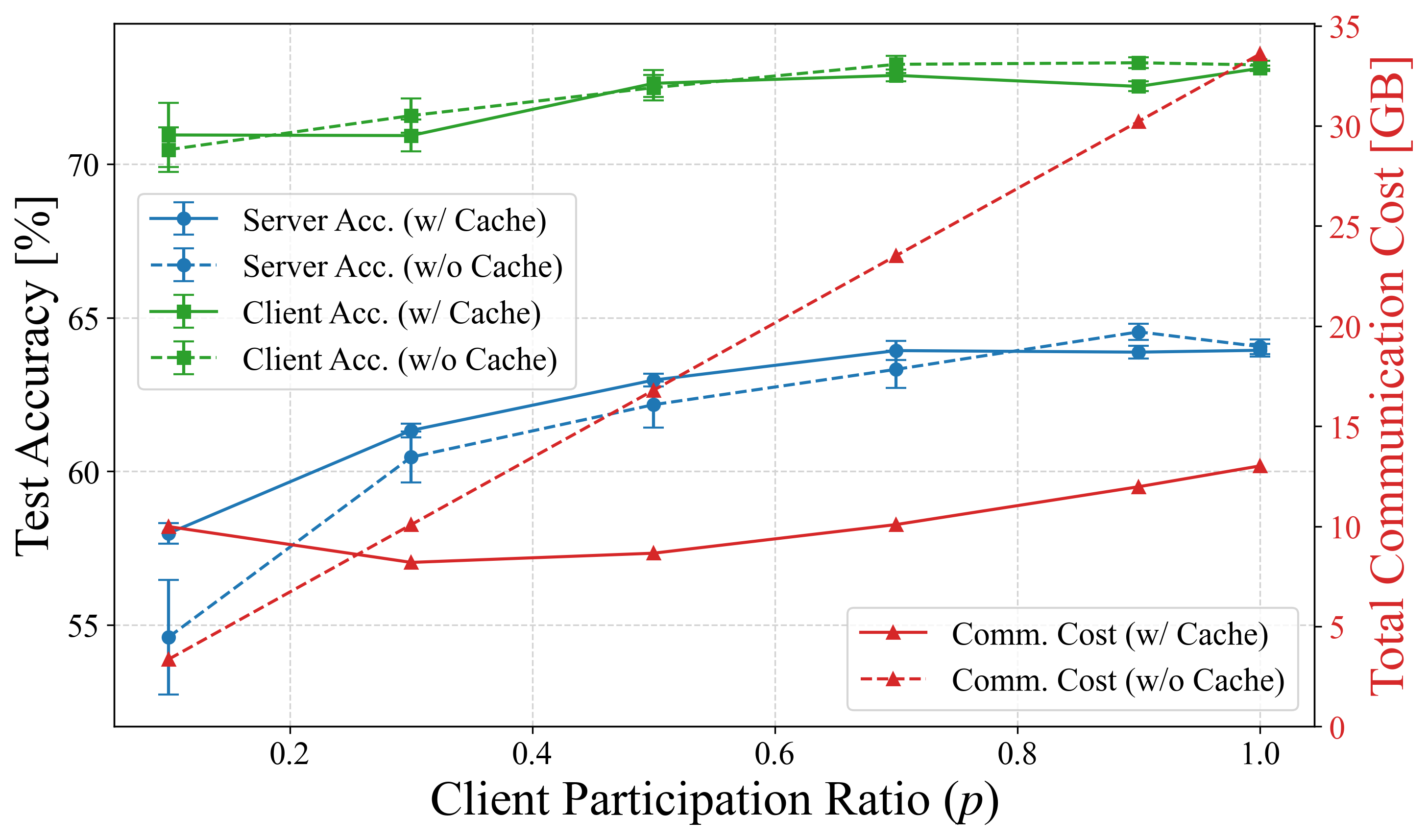}%
    \label{fig:partial_participation:0-05}%
  }

  \subfloat[Dirichlet $\alpha = 0.3$]{%
    \includegraphics[width=\columnwidth]{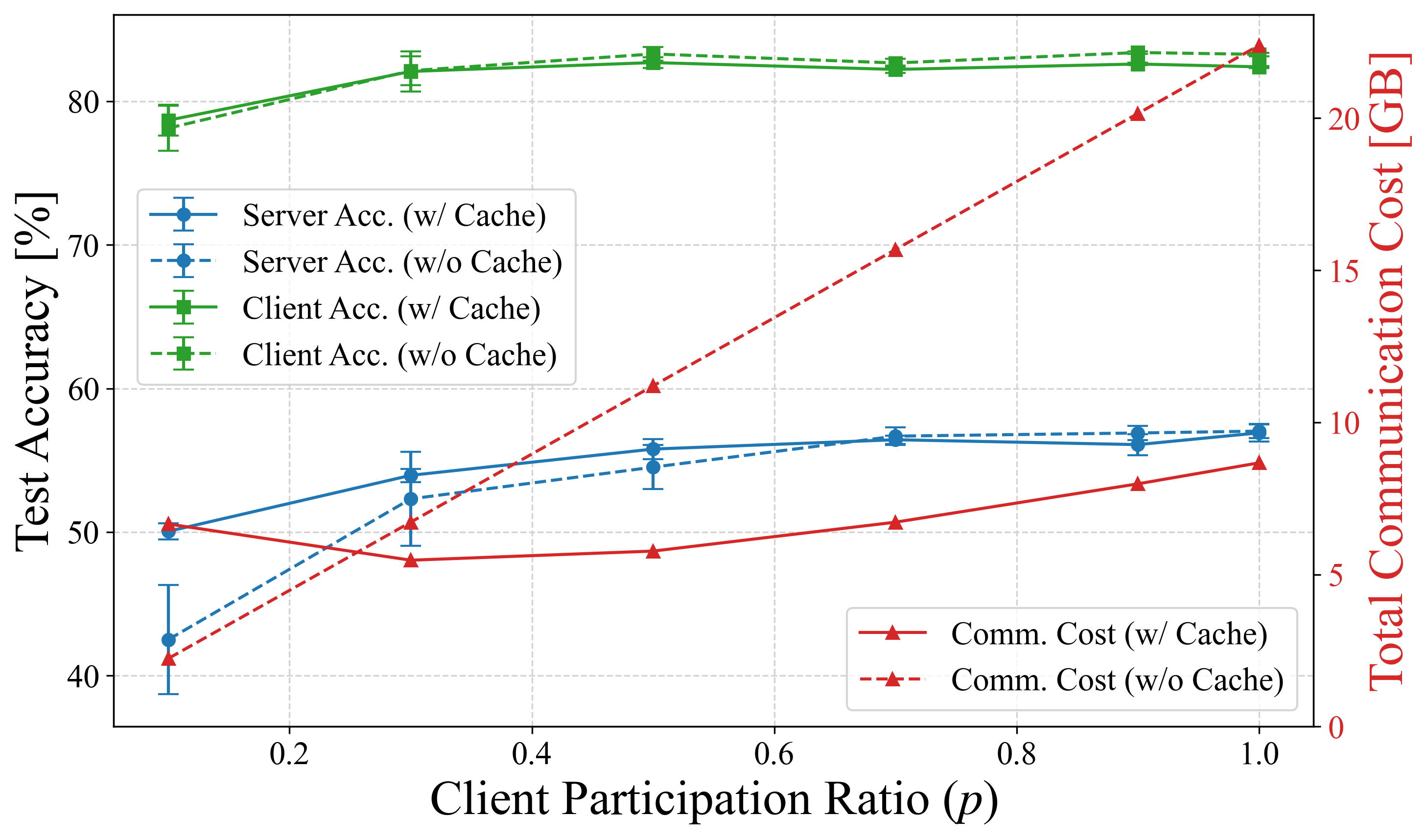}%
    \label{fig:partial_participation:0-3}%
  }

  \caption{Performance Comparison for SCARLET under varying client participation ratios ($p$) on CIFAR-10. Solid lines indicate the soft-label caching mechanism was enabled (w/ Cache), while dashed lines indicate it was disabled (w/o Cache). Accuracy values are the mean $\pm$ std. dev. over the final 10 rounds. The figure highlights that enabling caching results in significant communication cost reduction with only slight accuracy degradation in high participation settings ($p \ge 0.7$), whereas in low participation settings ($p \le 0.3$), it improves accuracy considerably despite increasing communication cost.}
  \label{fig:partial_participation}
\end{figure}

As shown in Fig.~\ref{fig:partial_participation}, SCARLET's caching mechanism (solid lines) demonstrates its primary benefit at high participation ratios ($p \ge 0.7$). It achieves substantial communication cost reductions (e.g., the red solid line is significantly below the red dashed line for $p \ge 0.7$ in both Fig.~\ref{fig:partial_participation:0-05} and Fig.~\ref{fig:partial_participation:0-3}) while maintaining comparable accuracy (blue and green solid lines are close to their dashed counterparts), with performance degradation typically below 1 percentage point in these high participation settings.

Conversely, in low participation scenarios (specifically $p = 0.1$), the cache provides a significant stabilization effect. In Fig.~\ref{fig:partial_participation:0-3}, the baseline model's server accuracy (blue dashed line at $p=0.1$) degrades significantly to 42.51\%, whereas the caching mechanism (blue solid line) maintains it at 50.05\%, a 7.5 percentage point improvement. We attribute this to temporal stabilization. In low participation settings, the aggregated teacher signal is volatile, as it is computed from a small, frequently changing subset of clients. The baseline model is sensitive to this instability. SCARLET's cache mitigates this by reusing soft-labels from previous rounds, which functions as a temporal smoothing mechanism, providing a more consistent training signal and dampening the volatility from low client sampling ratios. This stabilization effect introduces a trade-off, as seen in the $\alpha=0.3$, $p=0.1$ scenario (Fig.~\ref{fig:partial_participation:0-3}): the communication cost with the cache (red solid line at 9.99 GB) is significantly higher than without it (red dashed line at 3.36 GB). This increased communication cost is due to the frequent transmission of catch-up packages to stale clients.

These results demonstrate the robustness of SCARLET's design: while it significantly reduces costs in high participation settings, it also functions as a critical performance stabilizer in low participation, volatile environments.

\section{Limitations and Future Directions}
\label{limitations}

While SCARLET represents a significant advancement in communication-efficient FL, several limitations merit further investigation to broaden its practical applicability. This section discusses these limitations and suggests directions for future research.

First, SCARLET, like other distillation-based FL approaches, depends on the availability of a shared public dataset whose distribution closely resembles, but does not necessarily match exactly, that of the clients' data. This dependency restricts the applicability scenarios for distillation-based FL methods. Knowledge distillation trains a student model to replicate a teacher's outputs for given inputs, essentially enforcing functional equivalence. In theory, even random noise could serve as the shared dataset; however, prior work indicates that this yields poor results \cite{stanton2021does}. Recent techniques leveraging generative models have been proposed to produce suitable synthetic data in place of a real public dataset \cite{fedgen}, \cite{fedcg}, \cite{fediod}. Nevertheless, training generative models within an FL context still carries a risk of inadvertently exposing sensitive client patterns if not carefully managed. Thus, robust safeguards for privacy-preserving synthetic data generation remain critical. Integrating these privacy-conscious data synthesis techniques offers a promising direction to broaden the practical applicability of distillation-based FL frameworks like SCARLET.

Second, the current implementation of SCARLET assumes a homogeneous model architecture across all clients and the server for analytical clarity and simplicity. However, one significant advantage of distillation-based FL is its inherent capability to support heterogeneous models. Several studies have demonstrated successful implementations of model heterogeneity within FL frameworks, covering both client-server diversity \cite{fedgkt} and inter-client variability \cite{fedmd}. Nevertheless, comprehensive evaluations of accuracy degradation resulting from introducing model heterogeneity remain limited, with only a handful of studies, such as \cite{cmfd} addressing this critical aspect. Consequently, the experimental validation of SCARLET and similar distillation-based FL frameworks under heterogeneous model conditions remains an open research area. Investigating this unexplored dimension is essential not only from an academic standpoint but also for practical FL deployments, where diverse device capabilities necessitate tailored model architectures across client devices.

Third, our ablation study of the cache duration $D$ (Section~\ref{subsubsection:ablation_study_of_cache_duration}) clarifies its role and provides strong practical guidance for deployment. The results reveal a distinct performance cliff: conservative $D$ values offer substantial communication savings with minimal accuracy loss, establishing them as a robust and effective default. Conversely, the study confirmed that excessively long durations negatively impact server-side convergence and client-side stability. Critically, this risk is predictable. The lightweight simulation (Section~\ref{subsubsection:cache_module}) allows these ineffective, overly long $D$ values to be identified and avoided, before incurring the cost of a full FL experiment. While our current simple and fair refresh strategy proved highly effective with a conservative $D$, a promising future direction is to explore more granular caching policies. For instance, a probabilistic or selective per-sample expiration strategy might mitigate the instability caused by mass-refresh events observed with very long durations, potentially unlocking further efficiency gains.

Finally, our ablation study of Enhanced ERA (Section~\ref{subsubsection:ablation_study_of_enhanced_era}) clarifies its behavior and offers practical guidance for deployment. The mechanism is fundamentally robust, overcoming the class-count ($N$) dependency of conventional ERA and offering stable control where $\beta=1$ recovers simple averaging. Our empirical results also established a clear operational guideline: a static heuristic (e.g., $\beta=1.5$ in our CIFAR-10 tests) achieved near-optimal final accuracy across the entire spectrum of data heterogeneity, from strong non-IID to near-IID. While this establishes a highly effective default, our analysis also revealed opportunities for further optimization, such as prioritizing server-side convergence speed. Therefore, the primary remaining research direction shifts from identifying a static default to automating this tuning process. Developing a theoretical or heuristic method to set $\beta$ adaptively, perhaps using server-visible signals like aggregated soft-label entropy, remains a promising avenue for future work.

In summary, future research directions include investigating synthetic data generation methods as alternatives to public datasets, evaluating SCARLET’s efficacy in heterogeneous model settings, and developing more granular or probabilistic caching policies to build upon our analysis. Additionally, automating the tuning of the Enhanced ERA sharpness parameter $\beta$ promises further improvements by creating a fully dynamic mechanism for diverse non-IID data scenarios.

\section{Conclusion}
\label{conclusion}

In this paper, we proposed SCARLET, a novel distillation-based FL framework that addresses key challenges in communication efficiency and data heterogeneity. By introducing the soft-label caching mechanism and the Enhanced ERA mechanism, SCARLET significantly reduces communication costs while maintaining competitive model accuracy. Extensive experiments demonstrated that SCARLET outperforms state-of-the-art distillation-based FL approaches in both server- and client-side evaluations. The modular design of SCARLET ensures compatibility with other FL frameworks, offering a flexible and scalable solution for diverse deployment scenarios. SCARLET thus represents a substantial step forward in communication-efficient FL. As discussed in Section~\ref{limitations}, extending this work by investigating model heterogeneity, developing more granular or probabilistic caching policies, and automating the tuning of the Enhanced ERA sharpness parameter $\beta$ will address an even wider range of real-world challenges.

\section*{Acknowledgments}
This work was supported in part by JST, PRESTO under Grant JPMJPR2035, in part by JST, ALCA-Next under Grant JPMJAN24F1, Japan, and was carried out using the TSUBAME4.0 supercomputer at Institute of Science Tokyo.


 

\bibliographystyle{IEEEtran}
\bibliography{main}

@InProceedings{fedavg,
    title = {{Communication-Efficient Learning of Deep Networks from Decentralized Data}},
    author = {McMahan, Brendan and Moore, Eider and Ramage, Daniel and Hampson, Seth and Arcas, Blaise Aguera y},
    booktitle = {Proc. 20th Int. Conf. Artif. Intell. Stat.},
    pages = {1273--1282},
    year = {2017},
    editor = {Singh, Aarti and Zhu, Jerry},
    volume = {54},
    series = {Proc. Mach. Learn. Res.},
    publisher = {PMLR},
}

@ARTICLE{challenges,
    author = {Li, Tian and Sahu, Anit Kumar and Talwalkar, Ameet and Smith, Virginia},
    journal = {IEEE Signal Process. Mag.}, 
    title={{Federated Learning: Challenges, Methods, and Future Directions}}, 
    year = {2020},
    volume = {37},
    number = {3},
    pages = {50-60},
}

@InProceedings{fedgen,
  title = {{Data-Free Knowledge Distillation for Heterogeneous Federated Learning}},
  author = {Zhu, Zhuangdi and Hong, Junyuan and Zhou, Jiayu},
  booktitle = {Proc. 38th Int. Conf. Mach. Learn. (ICML)},
  pages = {12878--12889},
  year = {2021},
  editor = {Meila, Marina and Zhang, Tong},
  volume = {139},
  series = {Proc. Mach. Learn. Res.},
  publisher = {PMLR},
}

@ARTICLE{hbs,
  author = {Wang, Chengjia and Yang, Guang and Papanastasiou, Giorgos and Zhang, Heye and Rodrigues, Joel J. P. C. and de Albuquerque, Victor Hugo C.},
  journal = {IEEE Trans. Ind. Informat.}, 
  title = {{Industrial Cyber-Physical Systems-Based Cloud IoT Edge for Federated Heterogeneous Distillation}}, 
  year = {2021},
  volume = {17},
  number = {8},
  pages = {5511-5521},
}

@article{ktpfl,
    title = {{Parameterized Knowledge Transfer for Personalized Federated Learning}},
    author = {Zhang, Jie and Guo, Song and Ma, Xiaosong and Wang, Haozhao and Xu, Wenchao and Wu, Feijie},
    journal = {Adv. Neural Inf. Process. Syst. (NeurIPS)},
    volume = {34},
    pages = {10092--10104},
    year = {2021}
}

@INPROCEEDINGS{fedad,
    author = {Gong, Xuan and Sharma, Abhishek and Karanam, Srikrishna and Wu, Ziyan and Chen, Terrence and Doermann, David and Innanje, Arun},
    booktitle = {Proc. IEEE/CVF Int. Conf. Comput. Vis. (ICCV)}, 
    title = {{Ensemble Attention Distillation for Privacy-Preserving Federated Learning}}, 
    year = {2021},
    volume = {},
    number = {},
    pages = {15056-15066}
}

@inproceedings{fedkdprivacy,
    author = {Gong, Xuan and Sharma, Abhishek and Karanam, Srikrishna and Wu, Ziyan and Chen, Terrence and Doermann, David and Innanje, Arun},
    title = {{Preserving Privacy in Federated Learning with Ensemble Cross-Domain Knowledge Distillation}},
    booktitle = {Proc. AAAI Conf. Artif. Intell.},
    volume = {36},
    number = {11},
    pages = {11891--11899},
    year = {2022}
}

@inproceedings{pfedme,
    author = {T. Dinh, Canh and Tran, Nguyen and Nguyen, Josh},
    booktitle = {Proc. Adv. Neural Inf. Process. Syst.},
    editor = {H. Larochelle and M. Ranzato and R. Hadsell and M.F. Balcan and H. Lin},
    pages = {21394--21405},
    publisher = {Curran Associates, Inc.},
    title = {{Personalized Federated Learning with Moreau Envelopes}},
    volume = {33},
    year = {2020}
}

@inproceedings{fedprox,
    author = {Li, Tian and Sahu, Anit Kumar and Zaheer, Manzil and Sanjabi, Maziar and Talwalkar, Ameet and Smith, Virginia},
    booktitle = {Proc. Mach. Learn. Syst.},
    editor = {I. Dhillon and D. Papailiopoulos and V. Sze},
    pages = {429--450},
    title = {{Federated Optimization in Heterogeneous Networks}},
    volume = {2},
    year = {2020}
}

@InProceedings{scaffold,
    title = {{SCAFFOLD: Stochastic Controlled Averaging for Federated Learning}},
    author = {Karimireddy, Sai Praneeth and Kale, Satyen and Mohri, Mehryar and Reddi, Sashank and Stich, Sebastian and Suresh, Ananda Theertha},
    booktitle = {Proc. 37th Int. Conf. Mach. Learn.},
    pages = {5132--5143},
    year = {2020},
    editor = {III, Hal Daumé and Singh, Aarti},
    volume = {119},
    series = {Proc. Mach. Learn. Res.},
    publisher = {PMLR},
}

@inproceedings{fedntd,
    author = {Lee, Gihun and Jeong, Minchan and Shin, Yongjin and Bae, Sangmin and Yun, Se-Young},
    booktitle = {Proc. Adv. Neural Inf. Process. Syst.},
    editor = {S. Koyejo and S. Mohamed and A. Agarwal and D. Belgrave and K. Cho and A. Oh},
    pages = {38461--38474},
    publisher = {Curran Associates, Inc.},
    title = {{Preservation of the Global Knowledge by Not-True Distillation in Federated Learning}},
    volume = {35},
    year = {2022}
}

@InProceedings{fedpaq,
    title = {{FedPAQ: A Communication-Efficient Federated Learning Method with Periodic Averaging and Quantization}},
    author = {Reisizadeh, Amirhossein and Mokhtari, Aryan and Hassani, Hamed and Jadbabaie, Ali and Pedarsani, Ramtin},
    booktitle = {Proc. 23rd Int. Conf. Artif. Intell. Stat.},
    pages = {2021--2031},
    year = {2020},
    editor = {Chiappa, Silvia and Calandra, Roberto},
    volume = {108},
    series = {Proc. Mach. Learn. Res.},
    publisher = {PMLR},
}

@InProceedings{fedcomgate,
    title = {{Federated Learning with Compression: Unified Analysis and Sharp Guarantees}},
    author = {Haddadpour, Farzin and Kamani, Mohammad Mahdi and Mokhtari, Aryan and Mahdavi, Mehrdad},
    booktitle = {Proc. 24th Int. Conf. Artif. Intell. Stat.},
    pages = {2350--2358},
    year = {2021},
    editor = {Banerjee, Arindam and Fukumizu, Kenji},
    volume = {130},
    series = {Proc. Mach. Learn. Res.},
    month = {13--15 Apr},
    publisher = {PMLR},
}

@INPROCEEDINGS{ags,
    author={Han, Pengchao and Wang, Shiqiang and Leung, Kin K.},
    booktitle={Proc. 40th IEEE Int. Conf. Distrib. Comput. Syst. (ICDCS)}, 
    title={{Adaptive Gradient Sparsification for Efficient Federated Learning: An Online Learning Approach}}, 
    year={2020},
    pages={300-310},
}

@inproceedings{feddf,
    author = {Lin, Tao and Kong, Lingjing and Stich, Sebastian U and Jaggi, Martin},
    booktitle = {Proc. Adv. Neural Inf. Process. Syst.},
    editor = {H. Larochelle and M. Ranzato and R. Hadsell and M.F. Balcan and H. Lin},
    pages = {2351--2363},
    publisher = {Curran Associates, Inc.},
    title = {{Ensemble Distillation for Robust Model Fusion in Federated Learning}},
    volume = {33},
    year = {2020}
}

@article{fedkd,
    title = {{Communication-Efficient Federated Learning via Knowledge Distillation}},
    author = {Wu, Chuhan and Wu, Fangzhao and Lyu, Lingjuan and Huang, Yongfeng and Xie, Xing},
    journal = {Nature Commun.},
    volume = {13},
    number = {1},
    pages = {2032},
    year = {2022},
    publisher = {Nature Publishing Group UK London}
}

@inproceedings{kd,
    title = {{Distilling the Knowledge in a Neural Network}},
    author = {Geoffrey Hinton and Oriol Vinyals and Jeffrey Dean},
    year = {2015},
    booktitle = {NIPS Deep Learn. Represent. Learn. Workshop}
}

@article{fedmd,
    title = {{FedMD: Heterogenous Federated Learning via Model Distillation}},
    author = {Li, Daliang and Wang, Junpu},
    journal = {arXiv:1910.03581},
    year = {2019}
}

@inproceedings{fd,
    author = {Jeong, Eunsung and Oh, Songhyun and Kim, Hyesung and Park, Jaeho and Bennis, Mehdi and Kim, Seong-Lyun},
    title = {{Communication-Efficient On-Device Machine Learning: Federated Distillation and Augmentation under Non-IID Private Data}},
    booktitle = {NeurIPS Workshop on Mach. Learn. Phone Consum. Devices (MLPCD)},
    year = {2018}
}

@ARTICLE{dsfl,
    author = {Itahara, Sohei and Nishio, Takayuki and Koda, Yusuke and Morikura, Masahiro and Yamamoto, Koji},
    journal = {IEEE Trans. Mobile Comput.}, 
    title = {{Distillation-Based Semi-Supervised Federated Learning for Communication-Efficient Collaborative Training With Non-IID Private Data}}, 
    year = {2023},
    volume = {22},
    number = {1},
    pages = {191-205},
}

@ARTICLE{cfd,
    author = {Sattler, Felix and Marban, Arturo and Rischke, Roman and Samek, Wojciech},
    journal = {IEEE Trans. Netw. Sci. Eng.}, 
    title = {{CFD: Communication-Efficient Federated Distillation via Soft-Label Quantization and Delta Coding}}, 
    year = {2022},
    volume = {9},
    number = {4},
    pages = {2025-2038},
}

@ARTICLE{comet,
  author={Cho, Yae Jee and Wang, Jianyu and Chirvolu, Tarun and Joshi, Gauri},
  journal={IEEE J. Sel. Top. Signal Process.}, 
  title={{Communication-Efficient and Model-Heterogeneous Personalized Federated Learning via Clustered Knowledge Transfer}}, 
  year={2023},
  volume={17},
  number={1},
  pages={234-247},
}

@Article{selectivefd,
    author={Shao, Jiawei and Wu, Fangzhao and Zhang, Jun},
    title={{Selective knowledge sharing for privacy-preserving federated distillation without a good teacher}},
    journal={Nature Commun.},
    year={2024},
    volume={15},
    number={1},
    pages={349},
}

@ARTICLE{fedcache,
    author = {Wu, Zhiyuan and Sun, Sheng and Wang, Yuwei and Liu, Min and Xu, Ke and Wang, Wen and Jiang, Xuefeng and Gao, Bo and Lu, Jinda},
    journal = {IEEE Trans. Mobile Comput.}, 
    title = {{FedCache: A Knowledge Cache-Driven Federated Learning Architecture for Personalized Edge Intelligence}}, 
    year = {2024},
    volume = {23},
    number = {10},
    pages = {9368-9382},
}

@article{fedcache2,
    title = {{FedCache 2.0: Exploiting the Potential of Distilled Data in Knowledge Cache-driven Federated Learning}},
    author = {Pan, Quyang and Sun, Sheng and Wu, Zhiyuan and Wang, Yuwei and Liu, Min and Gao, Bo},
    journal = {arXiv:2405.13378},
    year = {2024}
}

@inproceedings{cachefl,
    title = {{Cache-Enabled Federated Learning Systems}},
    author = {Liu, Yuezhou and Su, Lili and Joe-Wong, Carlee and Ioannidis, Stratis and Yeh, Edmund and Siew, Marie},
    booktitle = {Proc. 24th Int. Symp. Theory, Algorithmic Foundations, and Protocol Design for Mobile Netw. and Mobile Comput.},
    pages = {1--11},
    year = {2023}
}

@INPROCEEDINGS{resnet,
    author = {He, Kaiming and Zhang, Xiangyu and Ren, Shaoqing and Sun, Jian},
    booktitle = {Proc. IEEE Conf. Comput. Vis. Pattern Recognit. (CVPR)}, 
    title = {{Deep Residual Learning for Image Recognition}}, 
    year = {2016},
    pages = {770-778},
}

@article{cifar,
    title = {{Learning Multiple Layers of Features from Tiny Images}},
    author = {Krizhevsky, Alex and Hinton, Geoffrey and others},
    year = {2009},
    publisher = {Toronto, ON, Canada}
}

@article{tinyimagenet,
  title = {{Tiny ImageNet Visual Recognition Challenge}},
  author = {Le, Ya and Yang, Xuan},
  journal = {CS 231N},
  volume = {7},
  number = {7},
  pages = {3},
  year = {2015}
}

@techreport{caltech,
  title = {{Caltech-256 Object Category Dataset}},
  author = {Griffin, Gregory and Holub, Alex and Perona, Pietro and others},
  year = {2007},
  institution  = {California Inst. Technol. Pasadena},
  type         = {Tech. Rep.},
  number       = {7694},
}

@article{dirichlet,
  title = {{Measuring the Effects of Non-Identical Data Distribution for Federated Visual Classification}},
  author = {Hsu, Tzu-Ming Harry and Qi, Hang and Brown, Matthew},
  journal = {arXiv:1909.06335},
  year = {2019}
}

@article{stanton2021does,
  title={{Does Knowledge Distillation Really Work?}},
  author={Stanton, Samuel and Izmailov, Pavel and Kirichenko, Polina and Alemi, Alexander A and Wilson, Andrew G},
  journal={Adv. Neural Inf. Process. Syst. (NeurIPS)},
  volume={34},
  pages={6906--6919},
  year={2021}
}

@inproceedings{fedcg,
  title     = {FedCG: Leverage Conditional GAN for Protecting Privacy and Maintaining Competitive Performance in Federated Learning},
  author    = {Wu, Yuezhou and Kang, Yan and Luo, Jiahuan and He, Yuanqin and Fan, Lixin and Pan, Rong and Yang, Qiang},
  booktitle = {Proc. 31st Int. Joint Conf. Artif. Intell. (IJCAI)},
  publisher = {Int. Joint Conf. Artif. Intell. Org.},
  editor    = {Lud De Raedt},
  pages     = {2334--2340},
  year      = {2022},
}

@inproceedings{fediod,
  title={Federated Learning via Input-Output Collaborative Distillation},
  author={Gong, Xuan and Li, Shanglin and Bao, Yuxiang and Yao, Barry and Huang, Yawen and Wu, Ziyan and Zhang, Baochang and Zheng, Yefeng and Doermann, David},
  booktitle={Proc. AAAI Conf. Artif. Intell.},
  volume={38},
  number={20},
  pages={22058--22066},
  year={2024}
}

@article{fedgkt,
  title={{Group Knowledge Transfer: Federated Learning of Large CNNs at the Edge}},
  author={He, Chaoyang and Annavaram, Murali and Avestimehr, Salman},
  journal={Adv. Neural Inf. Process. Syst. (NeurIPS)},
  volume={33},
  pages={14068--14080},
  year={2020}
}

@ARTICLE{cmfd,
  author={Taya, Akihito and Nishio, Takayuki and Morikura, Masahiro and Yamamoto, Koji},
  journal={IEEE Trans. Signal Inf. Process. Netw.}, 
  title={Decentralized and Model-Free Federated Learning: Consensus-Based Distillation in Function Space}, 
  year={2022},
  volume={8},
  number={},
  pages={799-814},
}
\vfill
 
\vspace{11pt}

\clearpage
\onecolumn
\twocolumn
\appendices
\renewcommand{\thesubsectiondis}{\Roman{subsection}.}
\section{Algorithm for Lightweight Cache Hit Rate Simulation}
\label{appendix:cache_simulation}

This appendix provides the detailed algorithm for the lightweight, standalone simulation discussed in Section~\ref{subsubsection:cache_module}. This simulation is used to estimate the cache hit rate (as shown in Fig.~\ref{fig:cache_simulation}) by modeling only the random sampling of the public dataset and the cache expiration logic. As the algorithm demonstrates, it is computationally trivial and runs independently of the FL training, serving as a practical tool for tuning the cache duration $D$.

\begin{algorithm}[t]
\caption{Lightweight Cache Hit Rate Simulation}
\label{algorithm:cache_simulation}
\begin{algorithmic}[1]

\State \textbf{Input:}
\State \quad $|\mathcal{P}|$: Total size of public dataset (e.g., 10000)
\State \quad $|\mathcal{P}^t|$: Number of public samples per round (e.g., 1000)
\State \quad $D$: Cache duration to simulate (e.g., 50)
\State \quad $T$: Total simulation rounds (e.g., 2000)
\State \textbf{Output:}
\State \quad $R_\text{cached}$: An array of cache hit ratios for $t \in [1, T]$
\vspace*{1em}
\State $\mathcal{C}_\text{sim} \leftarrow$ array of size $|\mathcal{P}|$, initialized to \textbf{null} \newline \hspace*{1em} \AlgComment{Stores the cache timestamp for each sample}
\State $R_\text{cached} \leftarrow$ empty array
\vspace*{1em}
\If {$D = 0$}
    \State $R_\text{cached} \leftarrow$ array of size $T$, initialized to 0.0
    \State \textbf{return} $R_\text{cached}$
\EndIf
\vspace*{1em}
\For {$t=1,\cdots,T$ \textbf{simulation rounds}}\textbf{:}
    \State $\mathcal{I}^t \leftarrow$ Randomly select $|\mathcal{P}^t|$ indices from $[1, |\mathcal{P}|]$
    \State $HitCount \leftarrow 0$
    \For{each index $i \in \mathcal{I}^t$}
        \If{$\mathcal{C}_\text{sim}[i] = \textbf{null}$}
            \State $\mathcal{C}_\text{sim}[i] \leftarrow t$ \newline \hspace*{5em} \AlgComment{Cache miss (empty): Add to cache}
        \ElsIf{$t - \mathcal{C}_\text{sim}[i] > D$}
            \State $\mathcal{C}_\text{sim}[i] \leftarrow t$ \newline \hspace*{5em} \AlgComment{Cache miss (expired): Refresh cache}
        \Else
            \State $HitCount \leftarrow HitCount + 1$ \newline \hspace*{5em} \AlgComment{Cache hit (valid)}
        \EndIf
    \EndFor
    \State Append ($HitCount / |\mathcal{P}^t|$) to $R_\text{cached}$
\EndFor
\State \textbf{return} $R_\text{cached}$

\end{algorithmic}
\end{algorithm}

\section{Proof of Majorization for Power Exponents}
\label{appendix:majorization_proof}

In this appendix, we present the mathematical proof of the majorization property introduced in Section \ref{subsection:enhanced_era}. Specifically, we prove that for any two power exponents \( \beta_1 \) and \( \beta_2 \) satisfying \( \beta_2 > \beta_1 > 0 \), the distribution produced by Enhanced ERA with exponent \( \beta_2 \) is majorized by the distribution produced with exponent \( \beta_1 \). This result justifies the monotonic reduction of entropy as \(\beta\) increases, leveraging the Schur-concavity of Shannon entropy.

\begin{theorem}[Majorization via Power Exponents]
\label{thm:majorization}
Let \( \{x_i\}_{i=1}^N \) be nonnegative real numbers in non-\emph{decreasing} order, 
\[
0 \;\le\; x_1 \;\le\; x_2 \;\le\; \cdots \;\le\; x_N \;\le\; 1,
\]
and let 
\[
\sum_{i=1}^N x_i \;=\; 1.
\]
For any \( \beta_2 > \beta_1 > 0 \), define the unnormalized vectors
\begin{align*}
    X^{(\beta_1)} \;=\; \bigl(x_1^{\beta_1},\, x_2^{\beta_1},\, \ldots,\, x_N^{\beta_1}\bigr), \\
X^{(\beta_2)} \;=\; \bigl(x_1^{\beta_2},\, x_2^{\beta_2},\, \ldots,\, x_N^{\beta_2}\bigr),
\end{align*}
and their corresponding \emph{normalized} distributions
\[
\hat{X}^{(\beta_1)} \;=\; \frac{X^{(\beta_1)}}{\sum_{j=1}^N x_j^{\beta_1}}, 
\qquad
\hat{X}^{(\beta_2)} \;=\; \frac{X^{(\beta_2)}}{\sum_{j=1}^N x_j^{\beta_2}}.
\]
Then, for each \(1 \le k \le N\), the following holds:
\begin{equation}
    \label{eq:ratio_inequality}
    \frac{\sum_{i=1}^k x_i^{\beta_2}}{\sum_{j=1}^N x_j^{\beta_2}}
    \;\;\le\;\;
    \frac{\sum_{i=1}^k x_i^{\beta_1}}{\sum_{j=1}^N x_j^{\beta_1}}.
\end{equation}
Equivalently, when viewed as probability distributions, 
\(\hat{X}^{(\beta_2)}\) is \emph{majorized} by \(\hat{X}^{(\beta_1)}\).
\end{theorem}

\begin{proof}
\textbf{Condition 1: Equality of Total Sums}\\
Both normalized vectors sum to 1, satisfying the first requirement for majorization:
\[
\sum_{i=1}^N \hat{X}^{(\beta_1)}_i \;=\; 1
\quad\text{and}\quad
\sum_{i=1}^N \hat{X}^{(\beta_2)}_i \;=\; 1.
\]

\noindent
\textbf{Condition 2: Prefix Sum Comparison}\\
To prove the prefix sum dominance, we must show that for each \(1 \le k \le N\),
\[
\frac{\sum_{i=1}^k x_i^{\beta_2}}{\sum_{j=1}^N x_j^{\beta_2}}
\;\;\le\;\;
\frac{\sum_{i=1}^k x_i^{\beta_1}}{\sum_{j=1}^N x_j^{\beta_1}}.
\]

\paragraph{Step 1: Power Function Properties}
Since \(0 \le x_i \le 1\) for all \(i\) and \(\beta_2>\beta_1 > 0\), raising \(x_i\) to the higher exponent \(\beta_2\) decreases its value:
\[
x_i^{\beta_2} \;=\; (x_i^{\beta_1})^{\tfrac{\beta_2}{\beta_1}}
\;\;\le\;\;
x_i^{\beta_1},
\]
This inequality holds pointwise and extends to sums over subsets of indices:
\[
\sum_{i=1}^k x_i^{\beta_2} 
\;\le\;
\sum_{i=1}^k x_i^{\beta_1},
\qquad
\sum_{j=1}^N x_j^{\beta_2} 
\;\le\;
\sum_{j=1}^N x_j^{\beta_1}.
\]

\paragraph{Step 2: Ratio Comparison via Cross-Multiplication}
To establish inequality~\ref{eq:ratio_inequality}, consider the cross-product inequality:
\[
\Bigl(\sum_{i=1}^k x_i^{\beta_2}\Bigr)
\Bigl(\sum_{j=1}^N x_j^{\beta_1}\Bigr)
\;\;\le\;\;
\Bigl(\sum_{i=1}^k x_i^{\beta_1}\Bigr)
\Bigl(\sum_{j=1}^N x_j^{\beta_2}\Bigr).
\]
Rearranging this inequality confirms that the normalized prefix sums satisfy
\[
\frac{\sum_{i=1}^k x_i^{\beta_2}}
     {\sum_{j=1}^N x_j^{\beta_2}}
\;\;\le\;\;
\frac{\sum_{i=1}^k x_i^{\beta_1}}
     {\sum_{j=1}^N x_j^{\beta_1}},
\]
for all \(1\le k \le N\).

\paragraph{Conclusion}
Both the normalization and prefix sum comparison conditions are satisfied, so \(\hat{X}^{(\beta_2)}\) is majorized by \(\hat{X}^{(\beta_1)}\). By the Schur-concavity of Shannon entropy, this implies that increasing the power exponent \(\beta\) reduces entropy, aligning with the statement in the main text.

This completes the proof.
\end{proof}

\section{Analysis of Input Sensitivity and Control Stability in Aggregation}
\label{appendix:analysis_sensitivity_stability}

This appendix provides a formal mathematical analysis to substantiate the claim from Section~\ref{subsection:enhanced_era}: that the conventional ERA (Eq.~\ref{equation:era}) is highly sensitive to the input soft-label's entropy, whereas our proposed Enhanced ERA (Eq.~\ref{equation:enhanced_era}) provides robust and predictable control.

Our analysis focuses on the \textbf{log-probability ratio} between two arbitrary classes, $i$ and $j$. This metric clearly exposes the functional relationship between the control parameter ($T$ or $\beta$) and the resulting sharpness, allowing us to formally analyze sensitivity and stability.

\subsection{Limitations and Instability of Conventional ERA}

Conventional ERA sharpens the averaged soft-label $\overline{z}$ using the temperature-scaled softmax function, as defined in (Eq.~\ref{equation:era}):
\begin{align*}
    {\hat{z}_{i}}^\text{ERA} = \frac{e^{\overline{z}_{i}/T}}{\sum_{k}e^{\overline{z}_{k}/T}}
\end{align*}
The ratio of the output probabilities between two classes, $i$ and $j$, is:
\begin{align*}
    \mathit{Ratio}_\text{ERA} = \frac{{\hat{z}_{i}}^\text{ERA}}{{\hat{z}_{j}}^\text{ERA}} = \frac{e^{\overline{z}_{i}/T} / \sum_{k}e^{\overline{z}_{k}/T}}{e^{\overline{z}_{j}/T} / \sum_{k}e^{\overline{z}_{k}/T}} = e^{(\overline{z}_{i} - \overline{z}_{j})/T}
\end{align*}
Taking the natural logarithm of both sides reveals the core mechanism:
\begin{align}
    \ln(\mathit{Ratio}_\text{ERA}) = \frac{\overline{z}_{i} - \overline{z}_{j}}{T} = \frac{\Delta z}{T}
    \label{eq:era_log_ratio}
\end{align}
where $\Delta z = \overline{z}_{i} - \overline{z}_{j}$ is the \textbf{absolute difference} between the input probabilities.

This relationship (Eq.~\ref{eq:era_log_ratio}) reveals two critical points of instability:

\subsubsection{Conventional ERA: Sensitivity to Input Entropy (Scale-Dependent)}
The output log-ratio is proportional to $\Delta z$. This transformation is mathematically \textbf{scale-dependent}, meaning its output is sensitive to the absolute scale of the input probabilities, even when the relative knowledge (the ratio $R = \overline{z}_i / \overline{z}_j$) is identical.

This scale-dependency directly implies entropy-sensitivity because the entropy of a distribution strongly constrains the absolute scale of its constituent probabilities. A high-entropy (flat) distribution necessarily consists of probabilities with a low absolute scale (e.g., all values are near $1/N$). Conversely, a low-entropy (sharp) distribution is defined by having at least one probability with a high absolute scale (approaching 1). Therefore, any transformation that is scale-dependent will, by extension, be entropy-sensitive.

This dependency is evident by mathematically decomposing $\Delta z$:
\begin{align*}
    \Delta z = \overline{z}_i - \overline{z}_j = (R \cdot \overline{z}_j) - \overline{z}_j = \overline{z}_j (R - 1)
\end{align*}
The output log-ratio is therefore $\ln(\mathit{Ratio}_\text{ERA}) = (\overline{z}_j (R - 1)) / T$. This shows that the transformation is not a pure function of the relative knowledge ($R$), but is instead conflated with the input's absolute scale ($\overline{z}_j$), which is dictated by the entropy state.

Consider two inputs with the exact same relative knowledge ($R=1.5$):
\begin{itemize}
    \item \textbf{High-Entropy (Low-Scale) Input:} $\overline{z}_i=0.15, \overline{z}_j=0.10$. 
    Here, $R=1.5$ and the scale $\overline{z}_j=0.10$. The resulting signal is $\Delta z = \mathbf{0.05}$.
    
    \item \textbf{Lower-Entropy (Higher-Scale) Input:} $\overline{z}_i=0.30, \overline{z}_j=0.20$.
    Here, $R=1.5$, but the scale $\overline{z}_j=0.20$. The resulting signal is $\Delta z = \mathbf{0.10}$.
\end{itemize}
Despite the relative knowledge being identical ($R=1.5$), ERA treats these as two completely different signals, applying twice the sharpening effect to the second input for any given $T$. This demonstrates its fundamental sensitivity to input entropy (scale).

\subsubsection{Conventional ERA: Unstable Control Parameter}
The log-ratio is proportional to $1/T$, an \textbf{inverse relationship}. This function $f(T) = 1/T$ is inherently unstable in the sharpening region ($T \to 0$). To quantify this, we examine the sensitivity (the derivative) of the log-ratio with respect to the control parameter $T$:
\begin{align}
    \frac{d}{dT} \left( \frac{\Delta z}{T} \right) = - \frac{\Delta z}{T^2}
    \label{eq:era_sensitivity}
\end{align}
This sensitivity (Eq.~\ref{eq:era_sensitivity}) \textbf{explodes quadratically} as $T \to 0$. A change in $T$ from $0.2 \to 0.1$ has a much smaller effect than the same $0.1$ change from $T=0.15 \to 0.05$. This makes $T$ a non-linear, unpredictable, and ``run-away'' parameter, rendering robust hyperparameter tuning extremely difficult.

\subsection{Robustness and Stability of Enhanced ERA}

Enhanced ERA sharpens using a power function, as defined in Eq.~\ref{equation:enhanced_era}:
\begin{align*}
    {\hat{z}_{i}}^\text{E-ERA} = \frac{(\overline{z}_{i})^{\beta}}{\sum_{k}(\overline{z}_{k})^{\beta}}
\end{align*}
We repeat the same analysis. The output probability ratio is:
\begin{align*}
    \mathit{Ratio}_\text{E-ERA} = \frac{{\hat{z}_{i}}^\text{E-ERA}}{{\hat{z}_{j}}^\text{E-ERA}} = \frac{(\overline{z}_{i})^{\beta} / \sum_{k}(\overline{z}_{k})^{\beta}}{(\overline{z}_{j})^{\beta} / \sum_{k}(\overline{z}_{k})^{\beta}} = \left(\frac{\overline{z}_{i}}{\overline{z}_{j}}\right)^{\beta}
\end{align*}
Taking the natural logarithm reveals its fundamentally different mechanism:
\begin{equation}
    \ln(\mathit{Ratio}_\text{E-ERA}) = \beta \ln \left(\frac{\overline{z}_{i}}{\overline{z}_{j}}\right) = \beta \ln R
    \label{eq:enhanced_era_log_ratio}
\end{equation}
where $R = \overline{z}_{i} / \overline{z}_{j}$ is the \textbf{ratio} of the input probabilities.

This relationship (Eq. \ref{eq:enhanced_era_log_ratio}) directly solves both instability problems:

\subsubsection{Enhanced ERA: Robustness to Input Entropy (Scale-Invariant)}
This transformation is mathematically \textbf{scale-invariant}. The function $f(\overline{\mathbf{z}}) = \ln(\overline{z}_i / \overline{z}_j) = \ln R$ mathematically isolates the relative knowledge ($R$) from the absolute scale of the input probabilities. Because this transformation is invariant to the absolute scale, it is also, by extension, robust to the input's entropy state.

We can see this by re-examining the two inputs from the previous section, both of which contain the identical relative knowledge ($R=1.5$):
\begin{itemize}
    \item \textbf{High-Entropy (Low-Scale) Input:} $\overline{z}_i=0.15, \overline{z}_j=0.10$.
    The input signal is $\ln(1.5)$.
    
    \item \textbf{Lower-Entropy (Higher-Scale) Input:} $\overline{z}_i=0.30, \overline{z}_j=0.20$.
    The input signal is also $\ln(1.5)$.
\end{itemize}
Enhanced ERA maps both inputs to the \textbf{exact same value} before the control parameter $\beta$ is applied. This guarantees that the sharpening effect is consistent and robust to changes in input entropy (scale), as long as the relative knowledge (the ratio) remains constant.

\subsubsection{Enhanced ERA: Stable Control Parameter}
The log-ratio is \textbf{linearly proportional} to the control parameter $\beta$. This $g(\beta) = \beta$ relationship is the most stable form of control possible. The sensitivity (derivative) of the log-ratio with respect to $\beta$ is:
\begin{align}
    \frac{d}{d\beta} (\beta \ln R) = \ln R
    \label{eq:enhanced_era_sensitivity}
\end{align}
This sensitivity (Eq.~\ref{eq:enhanced_era_sensitivity}) is a \textbf{constant} (with respect to $\beta$) that depends only on the input ratio $R$. A change from $\beta=1.0 \to 1.5$ has the exact same additive effect on the log-ratio as a change from $\beta=2.0 \to 2.5$. This makes $\beta$ a predictable, linear, and exceptionally stable control knob.

\subsection{Conclusion}

In the limit (as $T \to 0$ or $\beta \to \infty$), both methods converge to the same one-hot vector concentrated on the maximal element(s) of $\overline{z}$.

However, the critical difference is not the destination, but the \textbf{stability of the path} to that destination.
\begin{itemize}
    \item \textbf{ERA} provides an unstable, non-linear control knob ($T$) and uses a \textbf{scale-dependent} transformation ($\Delta z$) that conflates knowledge with entropy.

    \item \textbf{Enhanced ERA} provides a stable, linear control knob ($\beta$) and uses a \textbf{scale-invariant} transformation ($\ln R$) that cleanly isolates knowledge from entropy.
\end{itemize}
Crucially, $\beta=1$ acts as a stable ``identity'' baseline, recovering the original averaged soft-label. This predictable, linear control and stable, scale-invariant transformation make $\beta$ a far more practical and robust hyperparameter for aggregation across diverse non-IID scenarios.

\section{Validation of Convergence Criteria in Practical Scenarios}
\label{appendix:validation}

A critical challenge in real-world FL deployments is defining a practical convergence criterion. 
Here, we examine how the common practice in deep learning---using a validation dataset and its loss for termination---can be applied to distillation-based FL and evaluate its effectiveness. 
In distillation-based FL, clients possess local private datasets in addition to a shared public dataset, but, in principle, the roles of training and validation can be separated in a manner similar to standard deep learning pipelines. 
It is important to note, however, that a private validation set reflects the convergence of each client's personalized accuracy, whereas a public validation set assesses the convergence of the global model, and the two do not necessarily converge simultaneously. The details of this evaluation are described below.

\subsection{Methodology for Validation}

\begin{figure}[t]
    \includegraphics[width=0.95\columnwidth]{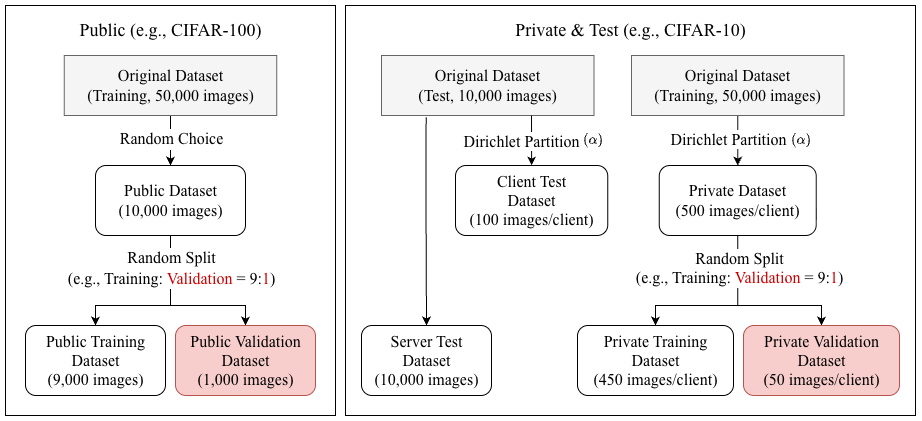}
    \centering
    \caption{Data partitioning strategy for validating practical convergence criteria. To simulate a real-world scenario without a test set, the Public Dataset is split into train/validation sets (left), and each client's Private Dataset is split into local train/validation sets (right). The held-out test sets (center) are used only for ground-truth evaluation of the proxy metrics.}
    \label{fig:validation_split}
\end{figure}

In a practical deployment, the server and clients can only monitor metrics derived from the data they possess (i.e., public data and local private data, respectively). We simulated this constraint by creating validation sets partitioned from the available public and private data, as illustrated in the data-split diagram (Fig.~\ref{fig:validation_split}).

\begin{itemize}
    \item \textbf{Private Validation Dataset:} For each client, we partitioned their local private data (e.g., 500 images from CIFAR-10) into a \textit{Private Training Dataset} (e.g., 450 images, 90\%) and a \textit{Private Validation Dataset} (e.g., 50 images, 10\%). Clients can compute metrics (e.g., loss and accuracy) on this static, held-out validation set to assess their local model's generalization performance on their own data distribution.
    
    \item \textbf{Public Validation Dataset:} We partitioned the server's public dataset (e.g., 10,000 images from CIFAR-100) into a \textit{Public Training Dataset} (e.g., 9,000 images, 90\%) and a \textit{Public Validation Dataset} (e.g., 1,000 images, 10\%). The server can compute metrics on this validation set (e.g., the global model's distillation loss) to monitor the stability and convergence of the global model.
\end{itemize}

We then compared the learning curves of these validation metrics against the ``ground truth'' performance on the held-out test datasets (server-side test accuracy and client-side test accuracy). The objective was to identify validation metrics that converge or plateau concurrently with the test accuracy, thereby serving as reliable indicators for stopping the FL process.

\subsection{Results and Discussion}

We analyzed the correlation between the ``ground truth'' test accuracies (which are unavailable in practice) and the available validation proxy metrics. Fig.~\ref{fig:convergence_validation} visualizes the strong correlations identified: comparing server-side test accuracy to the server's public validation loss, and client-side test accuracy to the client's private validation loss.

\begin{figure}[t]
    \centering
    \subfloat[Server test accuracy]{%
    \includegraphics[width=0.48\columnwidth]{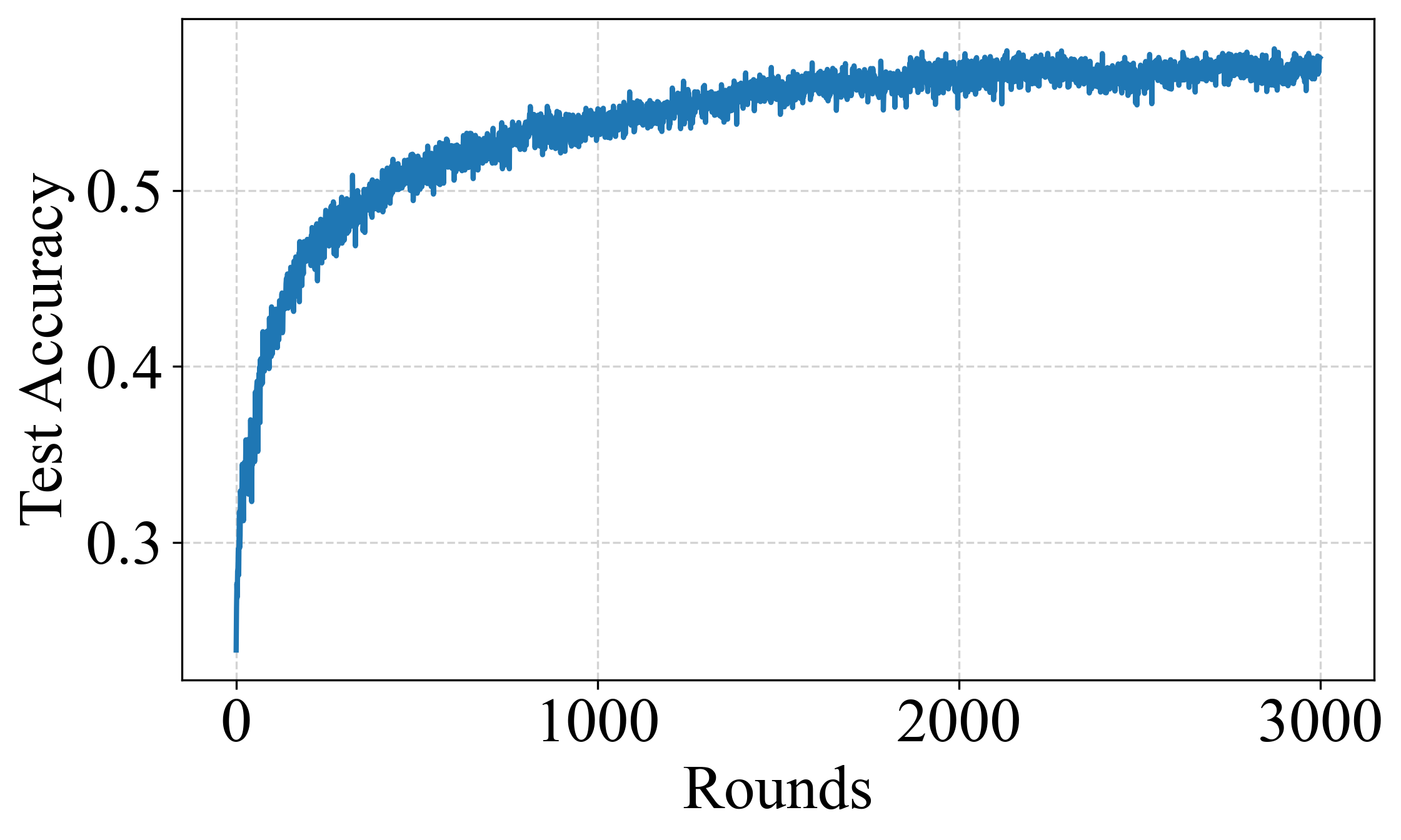}%
    \label{fig:convergence_validation:server_test_acc}%
    }\hfill
    \subfloat[Server public validation loss]{%
    \includegraphics[width=0.48\columnwidth]{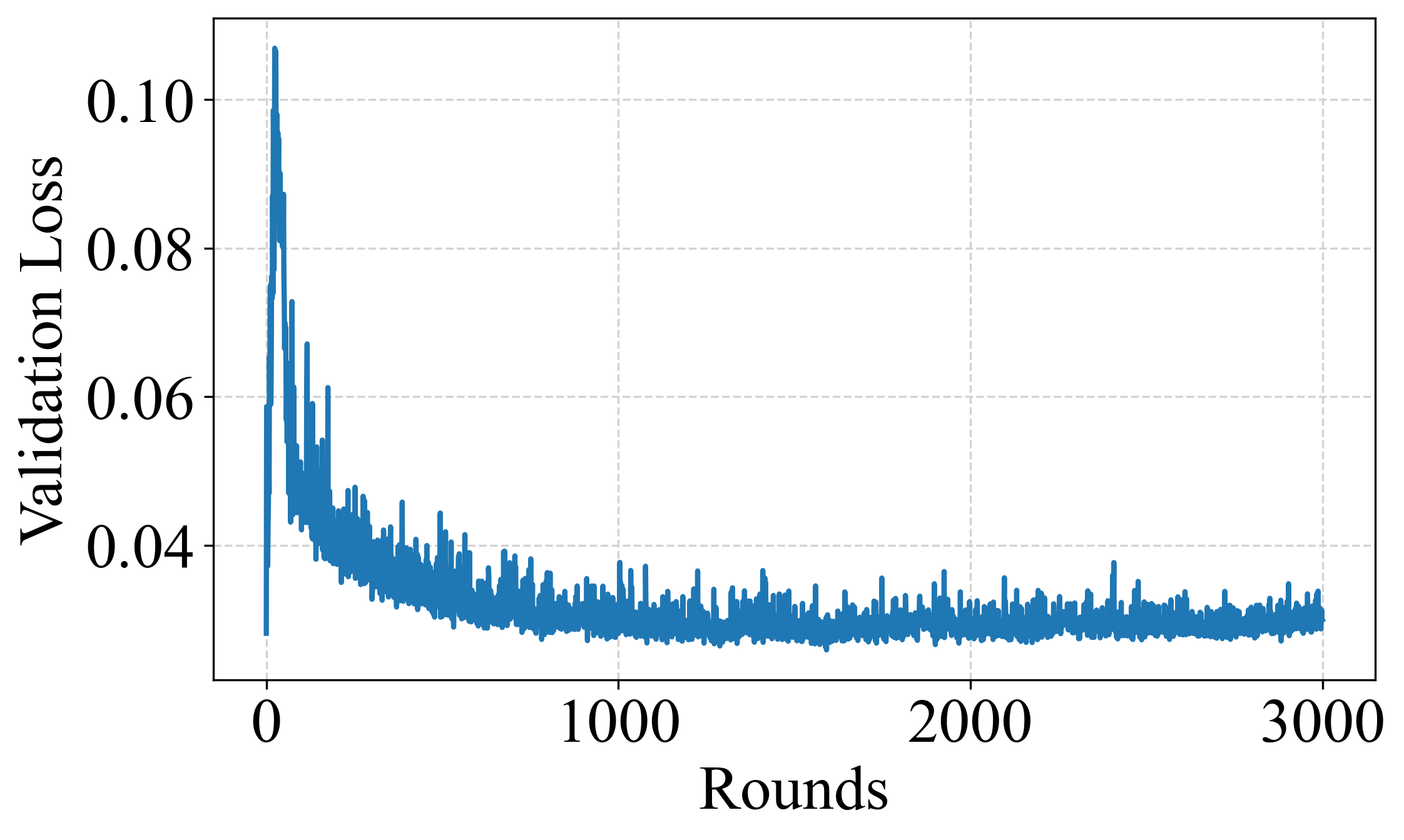}%
    \label{fig:convergence_validation:server_public_val_loss}%
    }
    
    \subfloat[Client test accuracy]{%
    \includegraphics[width=0.48\columnwidth]{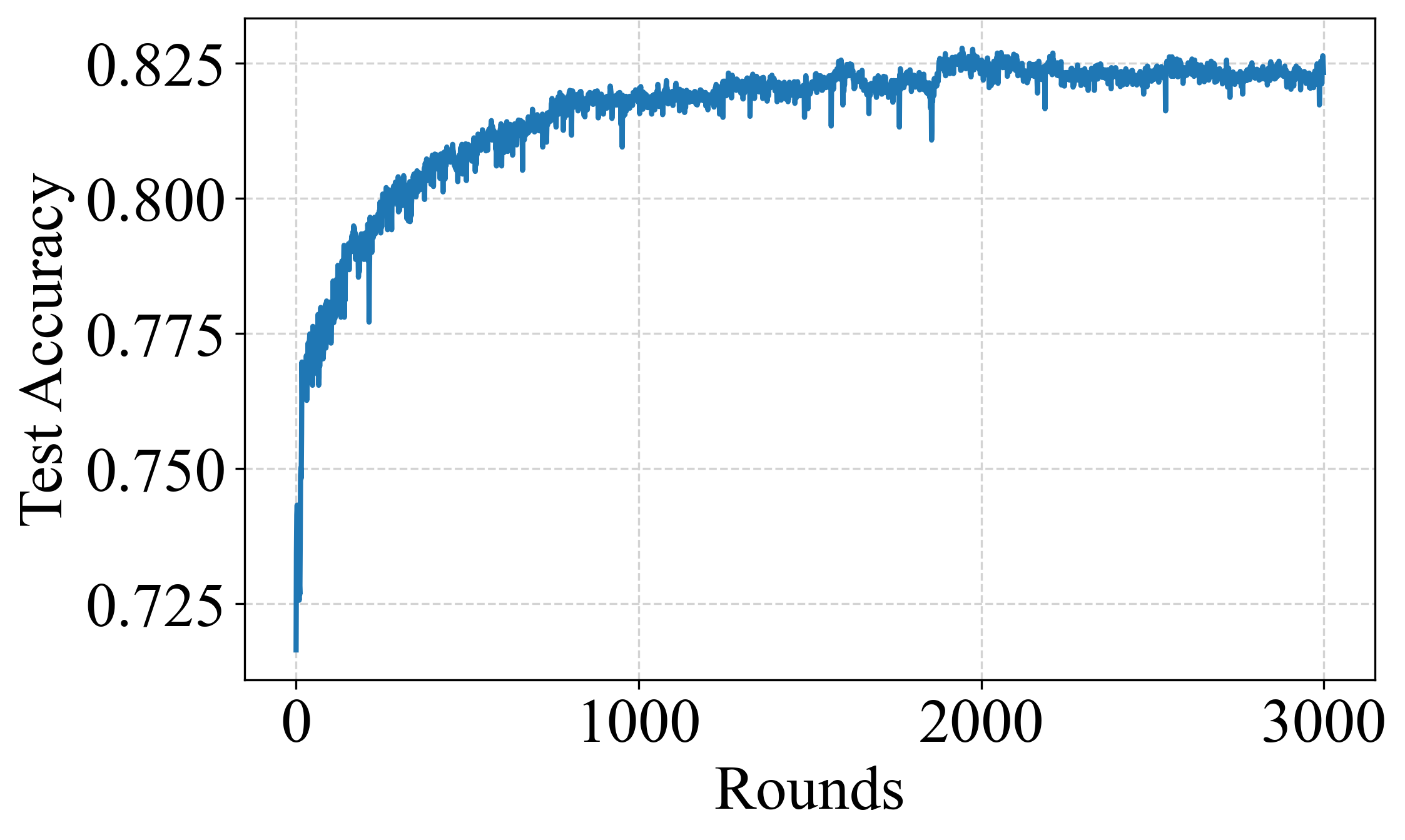}%
    \label{fig:convergence_validation:client_test_acc}%
    }\hfill
    \subfloat[Client private validation loss]{%
    \includegraphics[width=0.48\columnwidth]{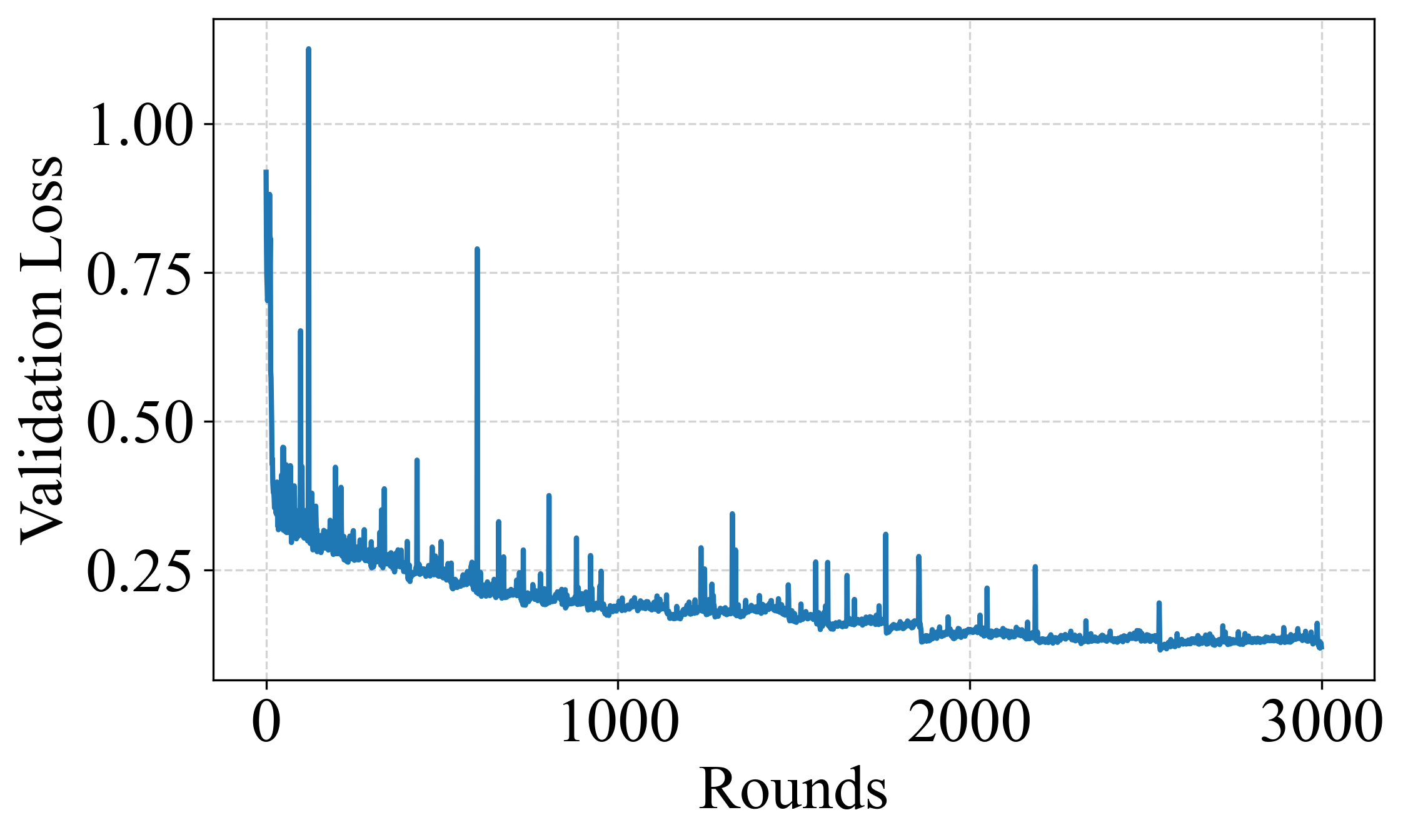}%
    \label{fig:convergence_validation:client_private_val_loss}%
    }
    \caption{Convergence analysis comparing unavailable ``ground-truth'' metrics with accessible proxy metrics. (\ref{fig:convergence_validation:server_test_acc}) Server test accuracy (ground truth) and (\ref{fig:convergence_validation:server_public_val_loss}) Server public validation loss (proxy). (\ref{fig:convergence_validation:client_test_acc}) Mean client test accuracy (ground truth) and (\ref{fig:convergence_validation:client_private_val_loss}) Mean client private validation loss (proxy). The proxy metrics (\ref{fig:convergence_validation:server_public_val_loss}, \ref{fig:convergence_validation:client_private_val_loss}) demonstrate strong concurrent convergence with their respective ground-truth accuracies (\ref{fig:convergence_validation:server_test_acc}, \ref{fig:convergence_validation:client_test_acc}), stabilizing as the test accuracy plateaus. This validates their use as effective and practical indicators for stopping the FL process.}
    \label{fig:convergence_validation}
\end{figure}

\smallbreak
\noindent
\textbf{Server-Side Convergence:}
We observed a strong correlation between the server's ground-truth performance and its accessible proxy metric. The server-side test accuracy (Fig.~\ref{fig:convergence_validation:server_test_acc}) reaches a stable plateau (approx. 57-58\% around rounds 2700-3000). This stabilization occurs concurrently with the server's loss on the public validation set (Fig.~\ref{fig:convergence_validation:server_public_val_loss}), which also settles into a stable low-value range (approx. 0.027-0.030) in the same period. This indicates that the server's public validation loss is a reliable proxy for the convergence of the global model.

\smallbreak
\noindent
\textbf{Client-Side Convergence:}
A similarly strong correlation was observed for the clients. The mean client test accuracy (Fig.~\ref{fig:convergence_validation:client_test_acc}) approaches its plateau (approx. 82-83\% around rounds 2000-3000). Concurrently, the mean loss on the client's local private validation set (Fig.~\ref{fig:convergence_validation:client_private_val_loss}) also stabilizes in a low-value region (approx. 0.11-0.12). Other metrics, such as the client's distillation loss on the public data (not shown), were less indicative, as they reflect the dynamic, round-to-round distillation process rather than the static generalization performance of the client's trained model on its own task. This confirms that monitoring the local private validation loss is a highly effective and practical method for a client to determine when their personalized model has converged.

\subsection{Conclusion}
This analysis demonstrates that the standard validation–based convergence criterion generally works well for distillation-based FL. For clients, the static loss on a held-out local validation set (private validation loss) serves as an excellent proxy for local model convergence. For the server, the loss on the public validation set (server public validation loss) provides a robust indicator of the global model’s convergence.

This validation confirms that the performance plateau observed in our main experiments is achievable using practical proxy metrics. It should be noted, however, that this local validation strategy implies a trade-off, particularly for clients in data-scarce (e.g., cross-device) settings. Partitioning an already small private dataset reduces the data available for local training. Determining the optimal train/validation split ratio, or whether such clients should perform local validation at all (perhaps relying solely on the server's convergence signal), remains an open research question dependent on data scarcity and the specific application. A full exploration of validation policies for highly resource-constrained clients is a potential area for future work.

\section{Hyperparameter Selection and Tuning for Comparative Analysis}
\label{appendix:hyperparameter}
To ensure the reproducibility of our experiments, we document the hyperparameter tuning process for each method and dataset in detail. The final hyperparameter values, summarized in Table~\ref{table:hyperparams_specific_with_values}, were selected based on comprehensive testing across candidate configurations.

\begin{table*}[t]
\centering
\caption{Method-Specific Hyperparameters for Each Private Dataset}
\label{table:hyperparams_specific_with_values}
\begin{tabular}{llccc}
\toprule
\multirow{2}{*}{Method} & \multirow{2}{*}{Hyperparameter} & \multicolumn{3}{c}{Private Dataset} \\
& & CIFAR-10 & CIFAR-100 & Tiny ImageNet \\ \midrule
DS-FL & ERA temperature ($T$) & 0.1 & 0.01 & 0.05 \\ \midrule
\multirow{2}{*}{CFD} & Upstream quantization bits ($b_{\text{up}}$) & 1 & 1 & 1  \\
& Downstream quantization bits ($b_{\text{down}}$) & 32 & 32 & 32  \\ \midrule
\multirow{2}{*}{COMET} & Number of clusters ($c$) & 2 & 2 & 3 \\
& Regularization weight ($\lambda$) & 1.0 & 1.0 & 1.0 \\ \midrule
\multirow{2}{*}{Selective-FD} & Client-side selector threshold ($\tau_{\text{client}}$) & 0.0625 & 0.0625 & 0.03125 \\
& Server-side selector threshold ($\tau_{\text{server}}$) & 2.0 & 2.0 & 2.0 \\ \bottomrule
\end{tabular}
\end{table*}

For DS-FL~\cite{dsfl}, the ERA temperature was tested over the range $T \in \{0.001, 0.005, 0.01, 0.05, 0.1, 0.5\}$. For CFD~\cite{cfd}, based on insights from the original paper, we directly adopted the settings $b_{\text{up}} = 1$ and $b_{\text{down}} = 32$ without further tuning. For COMET~\cite{comet}, the number of clusters was tested within the range $c \in \{2, 3, 4\}$. Additionally, the regularization weight $\lambda$ was tuned across $\lambda \in \{0.5, 1.0, 2.0\}$. For Selective-FD~\cite{selectivefd}, the client-side selector threshold was adjusted with $\tau_{\text{client}} \in \{0.03125, 0.0625, 0.125, 0.25\}$, and the server-side selector threshold was varied as $\tau_{\text{server}} \in \{1.0, 1.5, 2.0\}$.

The selection of hyperparameters was guided by performance metrics tailored to each method. Specifically, for DS-FL and CFD, we prioritized server-side accuracy, choosing the configurations that achieved the highest accuracy on the server. In contrast, for COMET and Selective-FD, client-side performance was the primary criterion, and the configurations yielding the highest client-side accuracy were selected. When multiple configurations exhibited similar accuracy, we opted for the one with faster convergence.

Certain methods required additional considerations. For Selective-FD, the official implementation does not include a server-side selector. We tested various values for the server-side selector threshold and found that $\tau_{\text{server}} = 2.0$, which effectively disables the selector, achieved the best accuracy. Consequently, our implementation for Selective-FD assumes no server-side selector. For CFD, delta coding was not included in our experiments (as no official implementation is available), as the original paper demonstrates that it primarily affects communication cost while having no impact on accuracy. With the upstream quantization already set to 1-bit in our setup, the absence of delta coding does not alter the accuracy results at all and only minimally affects the overall communication efficiency. For COMET, while the original paper focuses on soft-label aggregation without server-side model training, we extended the evaluation to include server-side training for consistency with other methods. This adjustment did not affect client-side evaluation. Furthermore, because COMET assumes partial client participation, we ensured fairness by determining client cluster assignments server-side and transmitting this information directly to each client, thereby avoiding additional communication costs from clustering.

These considerations ensure that each method is fairly evaluated under consistent conditions, taking into account their unique characteristics and theoretical assumptions. The final configurations adopted for each private dataset setting are detailed in Table~\ref{table:hyperparams_specific_with_values}.

\end{document}